\documentclass{article} %
\usepackage[dvipsnames,table]{xcolor} %
\usepackage{fullpage}
\usepackage{times}
\usepackage{natbib}

\renewcommand{\cite}{\citep}

\newcounter{arxiv}
\usepackage{tikz}
\usetikzlibrary{shapes}
\usetikzlibrary{positioning}
\usetikzlibrary{plotmarks}
\usetikzlibrary{patterns}
\usetikzlibrary{fit,calc}
\usepackage{pgfplots}
\usetikzlibrary{arrows.meta}
\usepackage[customcolors,shade]{hf-tikz}

\newcommand{\tabemph}[1]{\cellcolor{gray!10}{#1}}%
\usepackage{microtype}
\usepackage{graphicx}
\usepackage{enumitem,caption,subcaption}
\usepackage{booktabs} %
\usepackage{adjustbox}
\usepackage{xspace}
\usepackage{multirow}

\usepackage{minitoc} %

\definecolor{mydarkblue}{rgb}{0,0.08,0.45}
\usepackage[colorlinks = true,
            linkcolor = mydarkblue,
            urlcolor  = mydarkblue,
            citecolor = mydarkblue,
            anchorcolor = mydarkblue]{hyperref}

\newcommand{\rev}[1]{#1}

\newlength{\fullwidth}
\usepackage{algorithm}
\usepackage{algorithmic}

\newcommand{\pflam}{FedAlt\xspace}
\newcommand{\pflsu}{FedSim\xspace}
\newcommand{\lr}{\eta}
\newcommand{\hlsim}[1]{#1}
\newcommand{\hlalt}[1]{#1}

\newcommand{\myparagraph}[1]{\paragraph{#1.}\hspace{-0.8em}}  %
\renewcommand{\epsilon}{\varepsilon}

\usepackage{bm,mathtools,amsthm,amssymb}
\usepackage{bbm}
\usepackage{lscape}

\usepackage{color}
\definecolor{puorange}{rgb}{0.80,0.20,0}
\definecolor{bluegray}{rgb}{0.04,0,0.7}
\definecolor{greengray}{rgb}{0.05,0.50,0.15}
\definecolor{darkbrown}{rgb}{0.40,0.2,0.05}
\definecolor{darkcyan}{rgb}{0,0.4,1}
\definecolor{black}{rgb}{0,0,0}
\definecolor{grey}{rgb}{0.93,0.93,0.93}

\newcommand{\kp}[1]{{\color{greengray} {\textbf{Krishna}:} #1}}
\newcommand{\linx}[1]{{\color{red} {\textbf{Lin}:} #1}}
\newcommand{\maz}[1]{{\color{purple} {\textbf{Maziar}:} #1}}

\newcommand{\km}[1]{{\color{orange} {\textbf{Kshitiz}:} #1}}

\DeclareMathOperator*{\minimize}{minimize}

\DeclareMathOperator{\diag}{diag}
\newcommand{\E}{\mathbf{E}}
\newcommand{\R}{\mathbb{R}}
\newcommand{\cD}{\mathcal{D}}

\newcommand{\tF}{\widetilde{F}}

\newcommand{\tu}{\tilde{u}}
\newcommand{\tv}{\tilde{v}}
\newcommand{\tV}{\widetilde{V}}
\newcommand{\ut}{u^{(t)}}
\newcommand{\vt}{v^{(t)}}
\newcommand{\utp}{u^{(t+1)}}
\newcommand{\vit}{v_i^{(t)}}
\newcommand{\Vt}{V^{(t)}}

\newcommand{\zt}{z^{(t)}}
\newcommand{\St}{S^{(t)}}
\newcommand{\llangle}{\left\langle}
\newcommand{\rrangle}{\right\rangle}

\newcommand \Zcal {\mathcal Z}
\newcommand \Tcal {\mathcal T}

\newcommand \Fcal {\mathcal F}

\newcommand \Dcal {\mathcal D}

\newcommand \Scal {\mathcal S}

\newcommand \reals {\mathbb{R}}

\newcommand \prob {\mathbb{P}}

\newcommand \expect {\mathbf{E}}
\newcommand \Var {\mathbf{Var}}

\newcommand \pow [1]{^{(#1)}}
\DeclarePairedDelimiterX{\inp}[2]{\langle}{\rangle}{#1, #2} %
\DeclarePairedDelimiterX{\norm}[1]{\Vert}{\Vert}{#1} %
\DeclarePairedDelimiterX{\normsq}[1]{\Vert}{\Vert^2}{#1} %

\newcommand \eps \epsilon

\newcommand \op {\operatorname*{op}} %
\newcommand \grad {\nabla}

\newtheorem{theorem}{Theorem}
\newtheorem{lemma}[theorem]{Lemma}
 \newtheorem{claim}[theorem]{Claim}

\newtheorem{corollary}[theorem]{Corollary}
\newtheorem{remark}[theorem]{Remark}
\newtheorem{assumption}{Assumption}

\usepackage{thmtools, thm-restate} %

\newenvironment{assumptionP}[1]
  {\renewcommand{\theassumption}{\ref{#1}$'$}%
   \addtocounter{assumption}{-1}%
   \begin{assumption}}
  {\end{assumption}}

\newcommand\blfootnote[1]{%
  \begingroup
  \renewcommand\thefootnote{}\footnote{#1}%
  \addtocounter{footnote}{-1}%
  \endgroup
}

\title{Federated Learning with Partial Model Personalization}
\author{
Krishna Pillutla$^{1}$\quad 
Kshitiz Malik$^{2}$ \quad 
Abdelrahman Mohamed$^{2}$ \\ 
Michael Rabbat$^{2}$ \quad
Maziar Sanjabi$^{2}$ \quad 
Lin Xiao$^{2}$
\vspace{0.3cm}\\
  $^{1}$Paul G.\ Allen School of Computer Science \& Engineering, University of Washington \\
  $^{2}$Meta AI
}

\date{\vspace*{-2em}}

\begin{document}
\maketitle  %
\doparttoc %
\faketableofcontents %

\blfootnote{Published at the \textit{39\textsuperscript{th}International Conference on Machine Learning}, 
Baltimore, Maryland, USA, PMLR 162, 2022.}

\begin{abstract}
We consider two federated learning algorithms for training partially personalized models, where the shared and personal parameters are updated either simultaneously or alternately on the devices. Both algorithms have been proposed in the literature, but their convergence properties are not fully understood, especially for the alternating variant. We provide convergence analyses of both algorithms in the general nonconvex setting with partial participation and delineate the regime where one dominates the other. Our experiments on real-world image, text, and speech datasets demonstrate that (a) partial personalization can obtain most of the benefits of full model personalization with a small fraction of personal parameters, and, (b) the alternating update algorithm outperforms the simultaneous update algorithm \rev{by a small but consistent margin}. 
 \end{abstract}

\section{Introduction} \label{sec:intro}
Federated Learning \citep{mcmahan2017communication}
has emerged as a powerful paradigm for distributed and privacy-preserving machine learning 
\citep[see][and references therein]{kairouz2019flsurvey}.
We consider a typical setting of Federated Learning (FL) with~$n$ devices (also called clients), where each device~$i$ has a training dataset of $N_i$ samples 
$z_{i, 1}, \cdots, z_{i, N_i}$.
Let $w\in\R^d$ represent the parameters of a machine learning model and $f_i(w,z_{i,j})$ be the loss of the model on the training example $z_{i,j}$.
Then the loss function associated with device~$i$ is 
$F_i(w) = ({1}/{N_i})\sum_{j=1}^{N_i} f_i(w,z_{i,j})$.
A common objective of FL is to find model parameters that minimize the weighted average loss across all devices %
\begin{equation}\label{eqn:one-fits-all}
\minimize_{w} \quad \sum_{i=1}^n \alpha_i F_i(w),
\end{equation}
where the weights $\alpha_i>0$ satisfy $\sum_{i=1}^n \alpha_i=1$.
A common practice is to choose $\alpha_i=N_i/N$ where $N=\sum_{i=1}^n N_i$, which corresponds to minimizing the average loss across all samples:
$({1}/{N})\sum_{i=1}^n\sum_{j=1}^{N_i} f_i(w,z_{i,j})$.

The main motivation for minimizing the average loss over all devices is to leverage their collective statistical power for better generalization, because the amount of data on each device can be very limited. This is especially important for training modern deep learning models with large number of parameters.
However, this argument assumes that the datasets from different devices are sampled from the same, or at least very similar, distributions. 
Given the diverse characteristics of the users and increasing trend of personalized on-device services, such an i.i.d.\ assumption may not hold in practice. Thus, the one-model-fits-all formulation in~\eqref{eqn:one-fits-all} can be ineffective and undesirable.

Several approaches have been proposed for personalized FL, 
including ones based on
multi-task learning \citep{smith2017federated}, 
meta learning \citep{fallah2020personalized}, 
and proximal methods \citep{dinh2020moreau,Li2021ditto}. 
A simple formulation that captures their main idea is
\begin{equation}\label{eqn:pfl-full-model}
\minimize_{w_0, \{w_i\}_{i=1}^n} \quad \sum_{i=1}^n \alpha_i \Bigl(F_i(w_i)+\frac{\lambda_i}{2}\|w_i-w_0\|^2\Bigr),
\end{equation}
where $w_i$ for $i=1,\ldots,n$ are personalized model parameters at the devices, $w_0$ is a reference model,
and the $\lambda_i$'s are regularization weights that control the extent of personalization.
A major disadvantage of the formulation~\eqref{eqn:pfl-full-model}, which we call \emph{full model personalization}, is that it requires twice the memory footprint of the full model, $w_i$ and $w_0$ at each device, which severely limits the size of trainable models.

On the other hand, full model personalization may be unnecessary for 
modern deep learning models, which are composed of many simple functional units, typically organized into layers or a more general interconnected architecture.
Personalizing the ``right'' components, selected with domain knowledge, may lead to substantial benefits with only a small increase in memory footprint. 
In addition, partial model personalization can be less susceptible to ``catastrophic forgetting'' \citep{mccloskey1989catastrophic}, where a large model finetuned on a small local dataset forgets the original (non-personalized) task, leading to degraded test performance.

\begin{figure*}[t]
\centering
\begin{subfigure}[b]{0.33\textwidth}
\centering
\includegraphics[height=0.24\fullwidth]{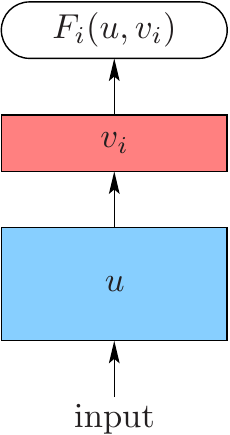}
\caption{\small{Personalized output layer(s).}}
\label{fig:model-output}
\end{subfigure}%
\begin{subfigure}[b]{0.33\textwidth}
\centering
\includegraphics[height=0.24\fullwidth]{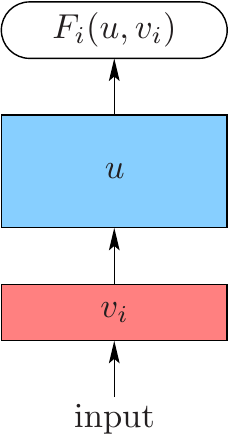}
\caption{\small{Personalized input layer(s).}}
\label{fig:model-input}
\end{subfigure}%
\begin{subfigure}[b]{0.33\textwidth}
\centering
\includegraphics[height=0.24\fullwidth]{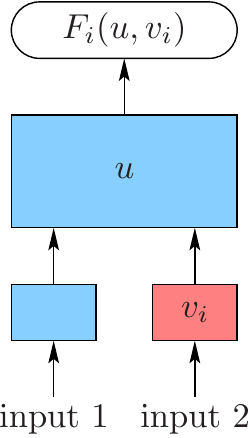}
\caption{\small{Personalized split input layer(s).}}
\label{fig:model-input2}
\end{subfigure}
\caption{\small{
Three simple examples of partitioning deep learning models.
}}
\label{fig:simple-models}
\end{figure*}

We consider a general setting of FL with \emph{partial model personalization}.
Specifically, we partition the model parameters into two groups: 
the \emph{shared} parameters $u\in\R^{d_0}$ and the \emph{personal}
parameters $v_i\in\R^{d_i}$ for $i=1,\ldots,n$. 
The full model on device~$i$ is 
denoted as $w_i=(u,v_i)$, and the local loss function is 
$F_i(u,v_i)=({1}/{N_i})\sum_{j=1}^{N_i} f_i\bigl((u,v_i),z_{i,j}\bigr)$.
Our goal is to solve the optimization problem
\begin{equation}\label{eqn:pfl-partial}
\minimize_{u,\, \{v_i\}_{i=1}^n} \quad \sum_{i=1}^n \alpha_i F_i(u, v_i).
\end{equation}
Notice that the dimensions of $v_i$ can be different across the devices, allowing the personalized components to have different number of parameters or even different architecture.
We investigate two FL algorithms for solving problem~\eqref{eqn:pfl-partial}: \emph{FedSim}, a simultaneous update algorithm and \emph{FedAlt}, an alternating update algorithm.
Both algorithms follow the standard FL protocol.
During each round, the server randomly selects a subset of the devices for update and broadcasts the current global version of the shared parameters to devices in the subset. 
Each selected device then performs one or more steps of (stochastic) gradient descent to update both the shared parameters and the personal parameters, and sends only the updated shared parameters to the server for aggregation. The updated personal parameters are kept locally at the device to serve as the initialization when the device is selected for another update.
In FedSim, the shared and personal parameters are updated simultaneously during each local iteration. 
In FedAlt, the devices
first update the personal parameters with the received shared parameters fixed and then update the shared parameters with the new personal parameters fixed.
We provide convergence analysis and empirical evaluation of both methods. 

\myparagraph{Contributions} 
Our main contributions are as follows.

\begin{itemize}[topsep=0pt, itemsep=1pt,leftmargin=\widthof{(a)}]%
\item We provide \emph{convergence guarantees} for the FedAlt and FedSim methods in the general (smooth) \emph{nonconvex setting with partial participation}.
While both methods have appeared in the literature previously, they are either used without convergence analysis or with results on limited settings (assuming convexity or full participation).
Our analysis focuses on the general nonconvex setting with partial participation, 
providing theoretical support for training modern deep learning models in practice.
The analysis of FedAlt with partial participation is especially challenging.
\rev{We decouple dependent random variables in FedAlt by introducing the technique of \emph{virtual full participation}}.
\item We conduct \textit{extensive experiments} on realistic image, text, and speech tasks, exploring different model personalization strategies for each task, and comparing with strong baselines. 
Our results demonstrate that partial model personalization can obtain most of the benefit of full model personalization with only a small fraction of personalized parameters, and that FedAlt outperforms FedSim \rev{by a small but consistent margin}.
\item 
Our experiments also reveal that personalization (full or partial) may lead to \textit{worse performance for some devices}, despite improving the average.
Typical forms of regularization such as weight decay and dropout do not mitigate this issue.
This phenomenon has been overlooked in previous work and calls for future research to improve both performance and fairness. 
\end{itemize}

It is our hope that the generality of our theory together with strong empirical study can provide valuable guidelines for training partially personalized models in practice.

\begin{figure*}[t]
\centering
\begin{subfigure}[b]{0.6\textwidth}
\centering
\includegraphics[height=0.32\fullwidth]{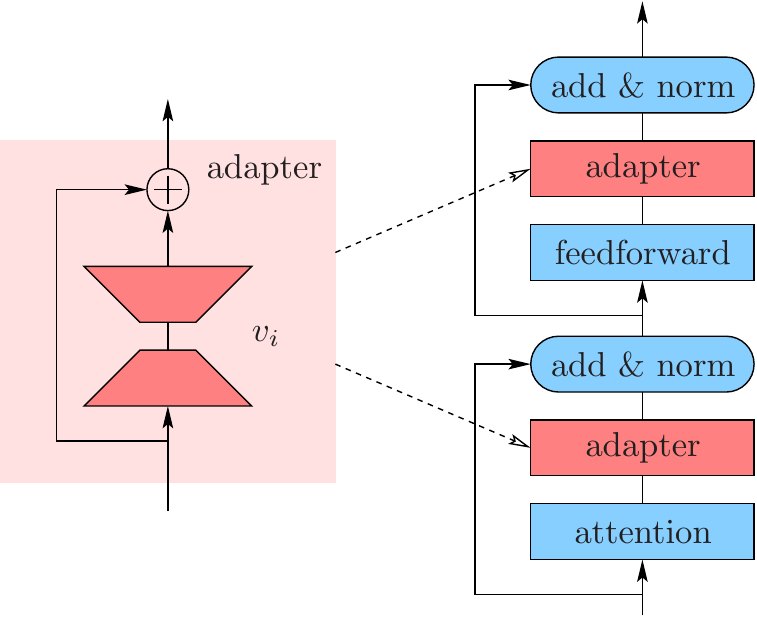}
\caption{Transformer layer with two adapters.}
\label{fig:model-adapter}
\end{subfigure}%
\begin{subfigure}[b]{0.4\textwidth}
\centering\qquad
\includegraphics[height=0.32\fullwidth]{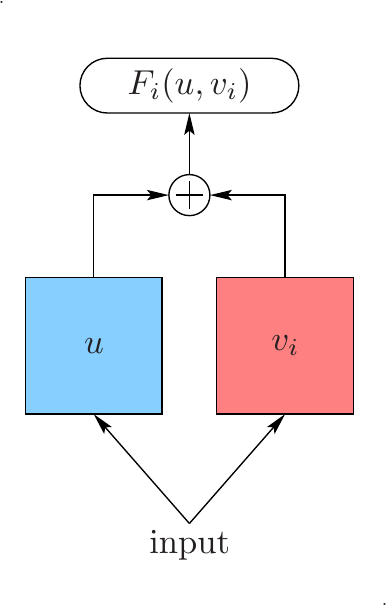}
\caption{Generalized additive model.}
\label{fig:model-additive}
\end{subfigure}%
\caption{\small{
More structured partial model personalization. 
(a) The adapter has a skip connection, thus it collapses to the identity mapping if $v_i=0$; in addition, it has a bottleneck in the middle 
\citep{houlsby2019parameter}.
(b) The generalized additive model can be further augmented with a shared input layer for representation learning.
}}
\label{fig:other-models}
\end{figure*}

\myparagraph{Related work}
\rev{
The ideas behind partial model personalization in federated learning can be traced back to seminal 
works on multi-task learning~\cite{caruana1997multitask,baxter2000model,collobert2008unified}.
These works advocate for learning a shared representation across various tasks. 
}
\rev{
These ideas were applied to the setting of federated learning by considering each client as a separate task by \citet{ghuhan2019federated} and \citet{collins2021expoiting}; see Figure~\ref{fig:model-output}. 
}
\citet{liang2020think} instead propose to personalize the input layers to learn a personalized representation (Figure~\ref{fig:model-input}).

\rev{Both optimization algorithms --- \pflsu and \pflam  ---} have appeared in the literature previously, but the scope of their convergence analyses is limited. 
Specifically,
\citet{liang2020think}, \citet{ghuhan2019federated} and \citet{hanzely2021unified} use \pflsu, while \citet{collins2021expoiting} and \citet{singhal2021federated} 
proposed variants of \pflam. 
Notably, \citet{hanzely2021unified} establish convergence of \pflsu with participation of all devices in each round in the convex and non-convex cases, while \citet{collins2021expoiting}
prove the linear convergence of \pflam for a two-layer linear network where $F_i(\cdot, v_i)$ and $F_i(u, \cdot)$ are both convex for fixed $v_i$ and $u$ respectively. 
We analyze both FedAlt and FedSim in the general nonconvex case with partial device participation where only a sample of devices participate in each round, hence addressing a more practical setting. 

While we primarily consider problem~\eqref{eqn:pfl-partial} in the context of partial model personalization, it can serve as a general formulation that covers many other problems. \citet{hanzely2021unified} demonstrate that various full model personalization formulations based on regularization~\citep{dinh2020moreau,Li2021ditto}, including~\eqref{eqn:pfl-full-model}, interpolation~\citep{deng2020adaptive,mansour2020three},
and meta-learning~\citep{fallah2020personalized,acar2021debiasing}
are special cases of this problem. 
The rates of convergence we prove in \S\ref{sec:algorithms} are competitive with or better than those in previous works for full model personalization methods in the non-convex case.

\section{Partially Personalized Models} \label{sec:models}

Modern deep learning models all have a multi-layer architecture. 
While a complete understanding of why they work so well is still out of reach, a general insight is that the lower layers (close to the input) are responsible for feature extraction and the upper layers (close to the output) focus on complex pattern recognition.
Depending on the application domain and scenarios, we may personalize either the input layer(s) or the output layer(s) of the model;
see Figure~\ref{fig:simple-models}.

In Figure~\ref{fig:model-input2}, the input layers are split horizontally into two parts, one shared and the other personal. They process different chunks of the input vector and their outputs are concatenated before feeding to the upper layers of the model.
As demonstrated by \citet{furl2019}, this partitioning can help protect user-specific private features (input~2 in Figure~\ref{fig:model-input2}) as the corresponding feature embedding (through~$v_i$) are personalized and kept local at the device.
Similar architectures have also been proposed in context-dependent language models \citep[e.g.,][]{mikolov2012context}. 

A more structured partitioning is illustrated in Figure~\ref{fig:model-adapter}, where a typical transformer layer \citep{vaswani2017attention}
is augmented with two adapters. 
This architecture is proposed by \citet{houlsby2019parameter} for finetuning large language models. 
Similar residual adapter modules are proposed by~\citet{rebuffi2017learning} for image classification models in the context of multi-task learning. In the context of FL, we treat the adapter parameters as personal and the rest of the model parameters as shared.

Figure~\ref{fig:model-additive} shows a generalized additive model, where the outputs of two separate models, one shared and the other personalized, are fused to generate a prediction.
Suppose the shared model is $h(u,\cdot)$ and the personal model is $h_i(v_i,\cdot)$. 
For regression tasks with samples $z_{i,j}=(x_{i,j},y_{i,j})$, where $x_{i,j}$ is the input and $y_{i,j}$ is the output, we let
$F_i(u,v_i) = ({1}/{N_i})\sum_{j=1}^{N_i} f_i\bigl((u,v_i),z_{i,j}\bigr)$ with
\[
f_i\bigl((u,v_i),z_{i,j}\bigr) =
\left\|y_{i,j} - h(u,x_{i,j}) - h_i(v_i,x_{i,j})\right\|^2.
\]
In this special case, the personal model fits the residual of the shared model and vice-versa~\citep{evgeniou2004regularized,agarwal2020federated}. 
For classification tasks, $h(u,\cdot)$ and $h_i(v_i,\cdot)$ produce probability distributions over multiple classes.
We can use the cross-entropy loss between $y_{i,j}$ and a convex combination of the two model outputs: $\theta h(u,x_{i,j})+(1-\theta) h_i(v_i,x_{i,j})$, where $\theta\in(0,1)$ is a learnable parameter.
Finally, we can cast full model personalization in~\eqref{eqn:pfl-full-model} as a special case of~\eqref{eqn:pfl-partial} by letting $u\gets w_0$, $v_i\gets w_i$ and 
\begin{align*}
F_i(u,v_i)\gets F_i(v_i) + (\lambda_i/2)\|v_i-u\|^2.
\end{align*}
Many other formulations of full model personalization can be reduced to~\eqref{eqn:pfl-partial} as well; see \citet{hanzely2021unified}.

\section{Algorithms and Convergence Analysis} \label{sec:algorithms}

\begin{algorithm}[t]
\caption{\hlsim{FedAlt}\,/\,\hlalt{FedSim} }
\label{alg:fedsim-fedalt}
\begin{algorithmic}[1]
		\STATE \textbf{Input:} Initial states $u^{(0)}, \{v_i^{(0)}\}_{i=1}^n$,
		    number of communication rounds $T$, 
		    number of devices per round $m$
	    \FOR{$t=0, 1, \cdots, T-1$}
        \STATE Server samples $m$ devices  $S^{(t)}\subset\{1,\ldots,n\}$
        \STATE Server broadcasts $u^{(t)}$ to each device in $S^{(t)}$
        \FOR{each device $i\in S^{(t)}$ in parallel}
        \STATE $u_i^{(t+1)},v_i^{(t+1)} = 
        \mbox{\hlsim{LocalAlt}\,/\,\hlsim{LocalSim}}
        \bigl(u^{(t)},v_i^{(t)}\bigr)$
        \STATE Device sends $u_i^{(t+1)}$ back to server
        \ENDFOR
        \STATE Server updates 
        $u^{(t+1)}=({1}/{m})\sum_{i\in S^{(t)}} u_i^{(t+1)}$
        \ENDFOR
\end{algorithmic}
\end{algorithm}

In this section, we present and analyze the FedAlt and FedSim algorithms for solving problem~\eqref{eqn:pfl-partial}.
To simplify presentation, we denote $V=(v_1,\ldots,v_n)\in\R^{d_1+\ldots+d_n}$ and focus on the case of $\alpha_i=1/n$, i.e.,
\begin{equation}\label{eqn:pfl-FwV}
\textstyle
\minimize_{u,\,V} \quad F(u,V):= \frac{1}{n}\sum_{i=1}^n F_i(u, v_i).
\end{equation}
This is equivalent to~\eqref{eqn:pfl-partial} if we scale $F_i$ by $n\alpha_i$, thus does not lose generality.
Moreover, we consider the more general setting with local functions $F_i(u,v_i)=\E_{z\sim\cD_i}[f_i((u,v_i),z)]$, where $\cD_i$ is 
the local data distribution.

The {FedAlt} and {FedSim} algorithms share a common outer-loop description given in Algorithm~\ref{alg:fedsim-fedalt}.
They differ only in the local update procedures {LocalAlt} and {LocalSim}, which are given in Algorithms~\ref{alg:localalt} and~\ref{alg:localsim} respectively.
We use $\widetilde{\nabla}_u$ and $\widetilde{\nabla}_v$ to represent stochastic gradients with respect to~$w$ and $v_i$ respectively.
In {LocalAlt} (Algorithm~\ref{alg:localalt}), the personal parameters are updated first with the received shared parameters fixed, then the shared parameters are updated with the new personal parameters fixed.
In {LocalSim} (Algorithm~\ref{alg:localsim}), the personal variables $v_i$ and local version of the shared parameters $u_i$ are updated simultaneously, with their partial gradients evaluated at the same point.
They are analogous respectively to the Gauss-Seidel and Jacobi update in numerical linear algebra \citep[e.g.,][\S6.5]{demmel1997book}.

The rest of the section is devoted to the convergence analysis.
We start with the assumptions in~\S\ref{sec:algos:asmp}. In \S\ref{sec:algos:vfp}, we outline the key technical difficulty of dependent random variables in the analysis of FedAlt and describe how we overcome it with virtual full participation. Finally, we compare the convergence rates of FedAlt and FedSim in \S\ref{sec:algos:comp}.

\subsection{Assumptions} \label{sec:algos:asmp}
We make some assumptions for the convergence analysis.

\begin{assumption}[Smoothness]\label{assmp:smoothness}
For each $i=1,\ldots,n$, the function $F_i$ is continuously differentiable.  
There exist constants $L_u, L_v, L_{uv}, L_{vu}$ such that
for each $i=1,\ldots,n$:
\begin{itemize}[topsep=0pt,itemsep=0pt,leftmargin=\widthof{(a)}]
\item $\nabla_u F_i(u,v_i)$ is $L_u$--Lipschitz with respect to~$u$ and $L_{uv}$--Lipschitz with respect to~$v_i$, and
\item $\nabla_v F_i(u,v_i)$ is $L_v$--Lipschitz with respect to~$v_i$ and $L_{vu}$--Lipschitz with respect to~$u$.
\end{itemize}
\end{assumption}
We summarize the relative cross-sensitivity of $\nabla_u F_i$ with respect to~$v_i$ and $\nabla_v F_i$ with respect to~$u$ with 
the scalar
\begin{equation}\label{eqn:chi-def}
\chi := \max\{L_{uv},\,L_{vu}\}\big/\sqrt{L_u L_v}.
\end{equation}

\begin{assumption}[Bounded Variance]
\label{assmp:stoc-grad-var}
The stochastic gradients in Algorithm~\ref{alg:localsim} and Algorithm~\ref{alg:localalt} are unbiased and have bounded variance. That is, for all $u$ and $v_i$, 
\begin{align*}
    \E\bigl[ \widetilde\nabla_u F_i(u, v_i)\bigr] &= \nabla_u F_i(u, v_i),
    \\
    \E\bigl[ \widetilde\nabla_v F_i (u, v_i)\bigr] &= \nabla_v F_i(u, v_i)\,.
\end{align*}
Furthermore, there exist constants $\sigma_u$ and $\sigma_v$ such that
\begin{align*}
    \E\bigl[\bigl\|\widetilde\nabla_u F_i(u, v_i) - \nabla_u F_i(u, v_i)\bigr\|^2\bigr] &\le \sigma_u^2\,, 
    \\ %
    \E\bigl[\bigl\|\widetilde\nabla_v F_i(u, v_i) - \nabla_v F_i(u, v_i)\bigr\|^2\bigr] &\le \sigma_v^2 \,.
\end{align*}
\end{assumption}

\begin{figure}[t]
\begin{minipage}[t]{0.48\textwidth}
\begin{algorithm}[H]
	\caption{\hlalt{LocalAlt}$\bigl(u, v_i\bigr)$
	}
	\label{alg:localalt}
\begin{algorithmic}[1]
		\STATE \textbf{Input:} Number of steps $\tau_v,\tau_u$, and step sizes $\gamma_v,\gamma_u$
        \STATE Initialize $v_{i,0}=v_i$ %
	    \FOR{$k=0, 1, \cdots, \tau_v\!-\!1$}
        \STATE $v_{i,k+1} = v_{i,k} - \gamma_v 
        \widetilde{\nabla}_v F_i\bigl(u,v_{i,k}\bigr)$
        \ENDFOR
        \STATE Update $v_i^+ =v_{i,\tau_v}$ and initialize $u_{i,0}=u$
	    \FOR{$k=0, 1, \cdots, \tau_u\!-\!1$}
        \STATE $u_{i,k+1} = u_{i,k} - \gamma_u 
        \widetilde{\nabla}_u F_i\bigl(u_{i,k},v_i^+\bigr)$
        \ENDFOR
        \STATE Update $u_i^+ = u_{i,\tau_u}$ 
        \STATE \textbf{Return} $\bigl(u_i^+,v_i^+\bigr)$
\end{algorithmic}
\end{algorithm}
\end{minipage}\hfill
\begin{minipage}[t]{0.48\textwidth}
\begin{algorithm}[H]
	\caption{\hlsim{LocalSim}$\bigl(u, v_i\bigr)$ 
	}
	\label{alg:localsim}
\begin{algorithmic}[1]
		\STATE \textbf{Input:} Number of steps $\tau$, and step sizes $\gamma_v, \gamma_u$
        \STATE Initialize $v_{i,0}=v_i$
        \STATE Initialize $u_{i,0}=u$
	    \FOR{$k=0, 1, \cdots, \tau-1$}
        \STATE $v_{i,k+1} = v_{i,k} - \gamma_v 
        \widetilde{\nabla}_v F_i\bigl(u_{i,k},v_{i,k}\bigr)$
        \STATE $u_{i,k+1} = u_{i,k} - \gamma_u 
        \widetilde{\nabla}_u F_i\bigl(u_{i,k},v_{i,k}\bigr)$
        \ENDFOR
        \STATE Update $v_i^+ =v_{i,\tau}$
        \STATE Update $u_i^+ = u_{i,\tau}$ 
        \STATE \textbf{Return} $\bigl(u_i^+,v_i^+\bigr)$
\end{algorithmic}
\end{algorithm}
\end{minipage}
\end{figure}

This is a standard bounded variance assumption on the per-device stochastic gradients~\cite{bottou2018optimization}. 
We have another source of stochasticity in our setting due to partial device participation. 
We can view $\nabla_u F_i(u,v_i)$, when~$i$ is randomly sampled from $\{1,\ldots,n\}$, as a stochastic partial gradient of $F(u,V)$.
The next assumption imposes a constant variance bound.
\begin{assumption}[Partial Gradient Diversity]
\label{assmp:grad-diversity}
There exist a constant $\delta\ge 0$ such that for all $u$ and~$V$,
\[
\textstyle
    \frac{1}{n}\sum_{i=1}^n \bigl\|\nabla_u F_i(u, v_i) - \nabla_u F(u, V)\bigr\|^2
    \le \delta^2  \,.
\]
\end{assumption}

Throughout this paper, we assume $F$ is bounded below by $F^\star$ and denote $\Delta F_0 = F\left(u\pow{0}, V\pow{0}\right)-F^\star$.
Further, we use the shorthands $V^{(t)}=(v_1^{(t)},\ldots,v_n^{(t)})$,
\begin{align*}
\textstyle
        \Delta_u\pow{t} &= \normsq*{\grad_u F\bigl(u\pow{t}, V\pow{t} \bigr)}\,,  \quad \text{and} \\
        \Delta_v\pow{t} &= \frac{1}{n}\sum_{i=1}^n \bigl\|\grad_v F_i\bigl(u\pow{t}, v_i\pow{t}\bigr)\bigr\|^2 \,.
\end{align*}
For smooth and nonconvex loss functions~$F_i$, we obtain convergence in expectation to a stationary point of~$F$ if the expected values of these two sequences converge to zero.

\subsection{Challenges of FedAlt and Virtual Full Participation} \label{sec:algos:vfp}

To convey the salient ideas, we assume full gradients on each device ($\sigma_u^2 = 0 = \sigma_v^2$) and a single local update per device ($\tau_u = 1 = \tau_v$). The only stochasticity in the algorithm comes from partial participation, i.e., sampling $m$ devices in each round.

\myparagraph{Dependent Random Variables}
Consider the iterates $(u\pow{t}, V\pow{t})$ generated by FedAlt (Algorithm~\ref{alg:fedsim-fedalt} with local updates from Algorithm~\ref{alg:localalt}).
In order to analyze the effect of the $u$-update, 
we invoke the smoothness of $F(\cdot \, , V\pow{t+1})$ as
\begin{align}
    \label{eq:pfl-am:main:smoothness}
    &F\bigl(u\pow{t+1}, V\pow{t+1}\bigr)
    - F\bigl(u\pow{t}, V\pow{t+1}\bigr)
    \le
     \bigl\langle \grad_u F\bigl(u\pow{t}, V\pow{t+1}\bigr),\, u\pow{t+1}\!\! - \! u\pow{t}\bigr\rangle
    + \frac{L_u}{2}\bigl\|u\pow{t+1} \!\!-\! u\pow{t}\bigr\|^2 \,.
\end{align}
Standard convergence proofs of stochastic gradient methods rely on the fact that when we  take expectation w.r.t.\ the sampling $S\pow{t}$ over the first order term (within the inner product), we obtain simplifications because the gradient is usually independent of $S\pow{t}$. 
\rev{This is true for FedSim and the $v$-step of FedAlt.
However, this is not the case for the $u$-step of FedAlt since 
\begin{align*}
\expect_t
\left[\bigl\langle \grad_u F\bigl(u\pow{t}, V\pow{t+1}\bigr),\, u\pow{t+1}\!\! - \! u\pow{t}\bigr\rangle \right]
\ne 
\bigl\langle \expect_t [\grad_u F\bigl(u\pow{t}, V\pow{t+1}\bigr)],\, \expect_t[u\pow{t+1}\!\! - \! u\pow{t}]\bigr\rangle
\end{align*}
in general,
where $\expect_t = \expect[\, \cdot \, | u\pow{t}, V\pow{t}]$ denotes the expectation w.r.t. $S\pow{t}$. 
Indeed, 
$V\pow{t+1}$ is already updated based on $S\pow{t}$, so both $V\pow{t+1}$ and $u\pow{t+1}$ are dependent random variables, due to their mutual dependence on the sampling $S\pow{t}$; see Figure~\ref{fig:a:fedalt:dependent_rv} (left)}. Therefore, directly taking expectation w.r.t.\ $S\pow{t}$ in~\eqref{eq:pfl-am:main:smoothness} does not lead to a useful result.

\myparagraph{Virtual Full Participation}
We decouple the dependent random variables with virtual full participation.
Define $\tV\pow{t+1}$ as the result of local $v$-updates 
as if \emph{every} device had participated.
This iterate is \emph{virtual}, meaning that it is a tool of the analysis but is not required by the algorithm. 
We introduce $\tV\pow{t+1}$ on the right hand side of \eqref{eq:pfl-am:main:smoothness}
to get 
\begin{align*}
    F\bigl(u\pow{t+1}, V\pow{t+1}\bigr)
    - F\bigl(u\pow{t}, V\pow{t+1}\bigr)
    \le  E\pow{t} +
    \bigl\langle\grad_u F(u\pow{t}, \tV\pow{t+1}),\, u\pow{t+1}\!\! -\! u\pow{t}\bigr\rangle
    + \frac{L_u}{2}\bigl\|u\pow{t+1} \!\! - \! u\pow{t}\bigr\|^2 \,,
\end{align*}
where $E\pow{t}$ is the error term from replacing $V\pow{t+1}$ with $\tV\pow{t+1}$.
Since $\tV\pow{t+1}$ is deterministic when conditioned on $(u\pow{t}, V\pow{t})$, 
we can now take an expectation w.r.t. the sampling $S\pow{t}$ over $u\pow{t+1}$ only, cf. Figure~\ref{fig:a:fedalt:dependent_rv} (right). 
This allows us to simplify the first order term as
\rev{
\begin{align*}
\expect_t
\left[\bigl\langle \grad_u F\bigl(u\pow{t}, \tV\pow{t+1}\bigr),\, u\pow{t+1}\!\! - \! u\pow{t}\bigr\rangle \right]
 &= 
\bigl\langle \grad_u F\bigl(u\pow{t}, \tV\pow{t+1}\bigr),\, \expect_t[u\pow{t+1}\!\! - \! u\pow{t}]\bigr\rangle 
\\&= 
-\frac{\gamma_u}{n} \sum_{i=1}^n  \expect_t\|\grad_u F(u\pow{t}, \tv\pow{t+1})\|^2 \,.
\end{align*}
}

Finally, we bound the error term $\expect_t[E\pow{t}] \le O(L_u\gamma_u^2 + \chi^2 L_v\gamma_v^2)$, which can be made small by choosing appropriately small learning rates. 

The technique of virtual full participation is distinct from shadow iterates $\bar u_k\pow{t} = (1/n)\sum_{i=1}^n u_{i,k}\pow{t}$ typically used in decentralized~\citep{yuan2016convergence}
and federated optimization~\citep{wang2021field}, \rev{and could be of independent interest}. 
We refer to Appendix~\ref{sec:a:convrg:vfp} for additional details. 

\begin{figure*}[t]
\centering
\begin{adjustbox}{max width=0.3\linewidth}
\begin{tikzpicture}

\node[draw, color=black, fill=black!1, very thick, circle, minimum size=1.7cm] 
  (S) at (0, 2.5) {\fontfamily{cmss}\selectfont \Large $S^{(t)}$};

\node[draw, color=black, fill=red!10, very thick, circle, minimum size=1.5cm] (V) at (-2,0) 
{\fontfamily{cmss}\selectfont \Large $V^{(t+1)}$}; 
    
\node[draw, color=black, fill=blue!10, very thick, circle, minimum size=1.5cm] (u) at (2,0) 
{\fontfamily{cmss}\selectfont \Large $u^{(t+1)}$};  

\node[draw, color=black, very thick, fill=black!1,
    rounded rectangle,
    minimum width=2.5cm,
    minimum height=1.5cm] (inp) at (0, -3) 
    {\fontfamily{cmss}\selectfont \Large 
    $\Big\langle
    \nabla_u F\big(u^{(t)}, \, 
    \tikzmarkin[set fill color=red!10,
set border color=black!0]{vv}(0.1,-0.25)(-0.1,0.6) V^{(t+1)}\tikzmarkend{vv}  \, \big),  \, \,
    \tikzmarkin[set fill color=blue!10,
set border color=black!0]{uu}(0.1,-0.25)(-0.1,0.6)u^{(t+1)}\tikzmarkend{uu} 
    \, - u^{(t)}
    \Big\rangle$
    };

\draw[-{Latex[width=4mm]}, very thick] (S) edge (V) ;
\draw[-{Latex[width=4mm]}, very thick] (S) edge (u) ;

\draw[-{Latex[width=4mm]}, very thick] (V) edge (inp) ;
\draw[-{Latex[width=4mm]}, very thick] (u) edge (inp) ;

\end{tikzpicture}
 \end{adjustbox}
\hspace{10em}
\begin{adjustbox}{max width=0.3\linewidth}
\begin{tikzpicture}

\node[draw, color=black, fill=black!1, very thick, circle, minimum size=1.7cm] 
  (S) at (0, 2.5) {\fontfamily{cmss}\selectfont \Large $S^{(t)}$};

\node[draw, color=black, fill=red!10, very thick, circle, minimum size=1.5cm] (V) at (-2,0) 
{\fontfamily{cmss}\selectfont \Large $V^{(t+1)}$}; 
    
\node[draw, color=black, fill=blue!10, very thick, circle, minimum size=1.5cm] (u) at (2,0) 
{\fontfamily{cmss}\selectfont \Large $u^{(t+1)}$};  

\node[draw, color=black, very thick, fill=black!1,
    rounded rectangle,
    minimum width=2.5cm,
    minimum height=1.5cm] (inp) at (0, -3) 
    {\fontfamily{cmss}\selectfont \Large 
    $\Big\langle
    \nabla_u F\big(u^{(t)}, \, 
    \tikzmarkin[set fill color=yellow!45,
set border color=black!0]{vvv}(0.1,-0.25)(-0.1,0.6) \widetilde V^{(t+1)}\tikzmarkend{vvv}  \, \big),  \, \,
    \tikzmarkin[set fill color=blue!10,
set border color=black!0]{uuu}(0.1,-0.25)(-0.1,0.6)u^{(t+1)}\tikzmarkend{uuu} 
    \, - u^{(t)}
    \Big\rangle$
    };

\draw[-{Latex[width=4mm]}, very thick] (S) edge (V) ;
\draw[-{Latex[width=4mm]}, very thick] (S) edge (u) ;

\draw[-{Latex[width=4mm]}, very thick] (u) edge (inp) ;

\end{tikzpicture}
 \end{adjustbox}
\caption{\textbf{Left}: Graphical model depicting the problem of dependent random variables in the analysis of FedAlt. We cannot take an expectation of the bottom-most inner product term w.r.t. the device sampling $S\pow{t}$  because both $V\pow{t+1}$ and $u\pow{t+1}$ depend on it.
\textbf{Right}: Virtual full participation overcomes this problem, since the virtual iterates $\tV\pow{t+1}$ are statistically independent of the sampling $S\pow{t}$.
The expectation can now pass through the inner product, as required by standard stochastic gradient analyses.  
}
\label{fig:a:fedalt:dependent_rv}
\end{figure*}
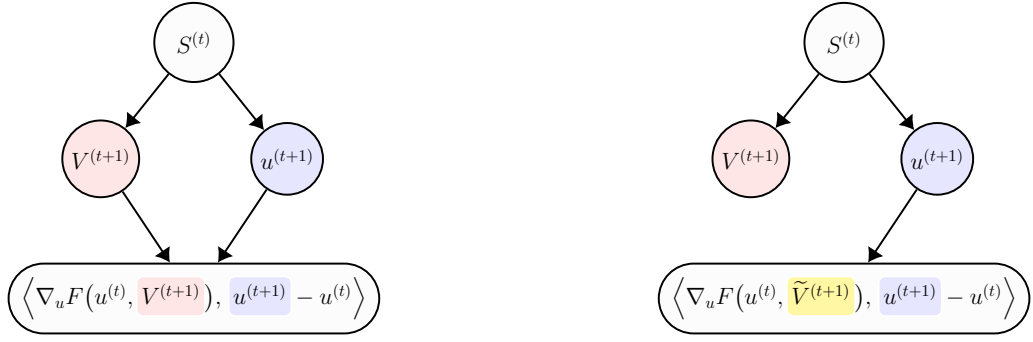

\subsection{Comparing FedAlt and FedSim} \label{sec:algos:comp}

We first present our main result for FedAlt
(Algorithm~\ref{alg:fedsim-fedalt} with LocalAlt). The proof relies on the technique of virtual full participation and is proved in Appendix~\ref{sec:a:proof:am}.

\begin{theorem}[\textbf{Convergence of \hlsim{FedAlt}}]
\label{thm:pfl-am:main}
Suppose Assumptions~\ref{assmp:smoothness}, \ref{assmp:stoc-grad-var} and~\ref{assmp:grad-diversity} hold and the learning rates in FedAlt are chosen as $\gamma_u = \lr/(L_u\tau_u)$ and $\gamma_v=\lr/(L_v\tau_v)$.
For a choice of $\lr$ depending on 
the problem parameters $L_u, L_v, \chi^2, \sigma_u^2, \sigma_v^2, \delta^2, m, n$, and the number of rounds $T$,
we have (ignoring absolute constants),
    \begin{align} \label{eqn:fedalt-bound}
    \begin{aligned}
        \frac{1}{T}\sum_{t=0}^{T-1} \left(\frac{1}{ L_u} \expect\bigl[\Delta_u\pow{t}\bigr]
        + \frac{m}{nL_v} \expect\bigl[\Delta_v\pow{t}\bigr] \right)
        \le 
        &\frac{\left(\Delta F_0 \, \sigma_{\mathrm{alt}, 1}^2 \right)^{1/2}}{\sqrt{T}} 
        +  \frac{\left(\Delta F_0^2 \, \sigma_{\mathrm{alt}, 2}^2 \right)^{1/3}}{T^{2/3}}  
        + 
        O\left( \frac{1}{T} \right)
        \,,
    \end{aligned}
    \end{align}
    where we define effective variance terms 
    \begin{align*}
        \sigma_{\mathrm{alt}, 1}^2 &= \frac{\delta^2}{L_u}\left(1-\frac{m}{n}\right) + \frac{\sigma_u^2}{L_u} + \frac{\sigma_v^2 (m + \chi^2(n-m))}{L_v n}\,, \\
        \sigma_{\mathrm{alt}, 2}^2 &= \frac{\sigma_u^2 + \delta^2}{L_u}(1-\tau_u^{-1}) + \frac{\sigma_v^2 m}{L_v n}(1-\tau_v^{-1}) + \frac{\chi^2 \sigma_v^2}{L_v} \,,
    \end{align*}
    and $O(\cdot)$ hides problem constants independent of $T$.
\end{theorem}

The left-hand side of~\eqref{eqn:fedalt-bound} is the average over time of a weighted sum of
$\expect\bigl[\Delta_u\pow{t}\bigr]$ and $\expect\bigl[\Delta_v\pow{t} \bigr]$. 
Convergence is measured in the rate at which this quantity decays to zero
and depends on effective noise variances $\sigma_{\mathrm{alt}, 1}^2, \sigma_{\mathrm{alt}, 2}^2$; these are weighed sums of the variances $\delta^2$, $\sigma_u^2$, and $\sigma_v^2$ contributed by the three sources of stochasticity.
The right side contains a standard $T^{-1/2}$ term with effective noise variance $\sigma_{\mathrm{alt}, 1}^2$ and a lower order $T^{-2/3}$ term with variance $\sigma_{\mathrm{alt}, 2}^2$.

Next, we present our main result for FedSim
(Algorithm~\ref{alg:fedsim-fedalt} with LocalSim), proved in
Appendix~\ref{sec:a:proof:su}. 

\begin{theorem}[\textbf{Convergence of \hlsim{FedSim}}]
\label{thm:pfl-su:main}
Suppose Assumptions~\ref{assmp:smoothness}, \ref{assmp:stoc-grad-var} and~\ref{assmp:grad-diversity} hold and the learning rates in FedSim are chosen as $\gamma_u = \lr/(L_u\tau)$ and $\gamma_v=\lr/(L_v\tau)$.
Then, for a $\lr$ depending on the problem parameters and the number of rounds $T$, the bound \eqref{eqn:fedalt-bound}
holds where the effective variance terms 
$\sigma_{\mathrm{alt},1}^2, \sigma_{\mathrm{alt},2}^2$
are respectively replaced by
\begin{align*}
    \sigma_{\mathrm{sim}, 1}^2 &= (1+\chi^2)\left(\frac{\delta^2}{L_u}\left(1-\frac{m}{n}\right) + \frac{\sigma_u^2}{L_u} + \frac{\sigma_v^2m}{L_v n}\right)\,, \\
    \sigma_{\mathrm{sim}, 2}^2 &= (1+\chi^2)\left(\frac{\delta^2}{L_u} + \frac{\sigma_u^2}{L_u} + \frac{\sigma_v^2}{L_v}\right) (1-\tau^{-1}) \,.
\end{align*}
\end{theorem}

The bound of FedSim is analogous to that of FedAlt, with the only difference in the noise terms $\sigma_{\mathrm{sim}, 1}^2$ and $\sigma_{\mathrm{sim}, 2}^2$.

\myparagraph{FedAlt vs. FedSim: Two Regimes}
Comparing the variances $\sigma_{\mathrm{alt}, 1}^2$ and $\sigma_{\mathrm{sim}, 1}^2$ in the leading $1/\sqrt{T}$
term, we identify two regimes in terms of problem parameters. The regime where FedAlt dominates FedSim is characterized by
the condition
\[
    \frac{\sigma_v^2}{L_v}\left(1 - \frac{2m}{n}\right) 
    < \frac{\sigma_u^2 + \delta^2(1-m/n)}{m L_u} \,.
\]
A practically relevant scenario where this is true is $\sigma_v^2 \approx 0$ and $\sigma_u^2\approx 0$ from using a large or full batch on a small number of samples per device. In this case, the rate of FedAlt is better than FedSim by a factor of $(1 + \chi^2)^{1/2}$, indicating that the rate of FedAlt is less affected by the coupling $\chi^2$ between the personal and shared parameters. Our experiments in \S\ref{sec:experiments} corroborate the practical relevance of this regime.

\begin{table*}[t]
    \caption{\small{Summary of datasets and models. A histogram of data per device is given in Figure~\ref{fig:expt:ds:hist} (Appendix~\ref{sec:a:expt-setup}).
    }}
\label{table:expt:dataset-summary}
\begin{center}
\begin{adjustbox}{max width=0.9\linewidth}
\small
\begin{tabular}{lccccccc}
\toprule
Task & Dataset & \#Classes & Model & \# Model Params & \#Devices & \multicolumn{2}{c}{\#Data per device} \\
& & &  &&& Mean  & Max  \\
\midrule
Next-word prediction & StackOverflow & $10000$ & $4$-layer transformer & $6M$ & $1000$ & $4964$ & $15520$ \\
Landmark recognition & GLDv2 & $2028$ & ResNet-18 & $12M$ & $823$ & $88$ & $1000$ \\
Character recognition & EMNIST & $63$ & ResNet-18 & $11M$ & $1114$ & $298$ & $418$ \\
Speech recognition & LibriSpeech & N/A & $6$-layer transformer & $15M$ & $902$ & $8.3$ min & $15$ min \\
\bottomrule
\end{tabular}
\end{adjustbox}
\end{center}
\end{table*}

\myparagraph{Extensions and Discussion}
Theorems~\ref{thm:pfl-am:main} and~\ref{thm:pfl-su:main} are also interesting because of the broad generality of the optimization model \eqref{eqn:pfl-partial}, as we discussed in \S\ref{sec:models}
and as pointed out by \citet{hanzely2021unified}. 
In particular, Theorems~\ref{thm:pfl-am:main} and~\ref{thm:pfl-su:main} also give rates for
full personalization schemes without convergence guarantees in the nonconvex case such as FedRes~\citep{agarwal2020federated}, Mapper~\citep{mansour2020three}, and Ditto~\citep{Li2021ditto}.
Furthermore, our rates are better than those of \citep{dinh2020moreau} for their pFedMe objective.

We give fully non-asymptotic versions of these theorems under more general assumptions in Appendix~\ref{sec:a:convergence}. The $O(1/T)$ term is lower order and can be ignored for $T \ge \Omega((n/m)^2)$ for FedAlt and $T\ge \Omega(n/m)$ for FedSim.

\section{Experiments} \label{sec:experiments}
We experimentally compare different model personalization schemes using \pflam and \pflsu. 
Further details about the experiments and hyperparameters as well as additional experimental results are provided in the appendices.
The code to reproduce the experimental results is publicly available.\footnote{
\url{https://github.com/krishnap25/FL_partial_personalization}
}

\myparagraph{Datasets, Tasks and Models}
We consider four learning tasks, summarized in Table~\ref{table:expt:dataset-summary}.
\begin{enumerate}[label=(\alph*),itemsep=0em, topsep=0em, leftmargin=\widthof{(a) }]
\item{\textit{Next-Word Prediction}:}
We use the StackOverflow dataset, where each device corresponds to the questions and answers of one user on \url{stackoverflow.com}.
This is representative of mobile keyboard predictions. We use a 4-layer transformer model~\citep{vaswani2017attention} trained with the cross entropy loss and evaluated with top-1 accuracy of next word prediction.

\item{\textit{Landmark Recognition}:}
We use GLDv2~\citep{weyand2020google}, 
a large-scale image dataset of global landmarks. Each device corresponds to a Wikipedia contributor and is representative of
smartphone users capturing images while traveling. 
We use ResNet-18~\citep{he2016deep}.
with group norm instead of batch norm~\citep{hsieh2020noniid} and images are reshaped to $224\times 224$. 
It is trained with the cross entropy loss and evaluated with the classification accuracy.

\item{\textit{Character Recognition}:}
We use the EMNIST dataset~\citep{cohen2017emnist}, 
where the input is a $28\times 28$ grayscale image of a handwritten character and the output is its label (0-9, a-z, A-Z). Each device corresponds to a writer of the character. We use a ResNet-18 model
with input and output layers modified to accommodate the smaller image size and number of classes.

\item{\textit{Speech Recognition (ASR)}:}
We construct a federated version of the LibriSpeech dataset~\citep{panayotov2015librispeech}, partitioned by the speaker of the audio. The input is an audio clip of English speech represented by log-mel filterbank coefficients and the output is its text transcription. We use a $6$-layer transformer model trained with the connectionist temporal classification (CTC) criterion~\citep{graves2006connectionist}
and report the word error rate for evaluation. 
\end{enumerate}

\begin{table*}[t]
    \caption{\small{Comparison of partial model personalization with full model personalization in terms of the \textit{average} test accuracy \% across devices.
    The subscript denotes the standard deviation over 5 random runs. The boldfaced/highlighted numbers denote entries within one standard deviation of the maximum in each row. For partial personalization, we show the accuracy of \pflam; see Table~\ref{table:expt:opt-algo:stateful} for \pflsu.
    }}
\label{table:expt:main}
\begin{center}
\begin{adjustbox}{max width=\linewidth}
\small

\renewcommand{\arraystretch}{1.2}
\begin{tabular}{lrrrrrrr}
\toprule
{} & Non-pers. & \multicolumn{3}{c}{Full Model Personalization} & \multicolumn{3}{c}{Partial Model Personalization} \\
\cmidrule(lr){3-5}
\cmidrule(lr){6-8}
    {} & FedAvg &        Finetune &           Ditto &                   pFedMe &     Input Layer &    Output Layer &                  Adapter \\
\midrule
StackOverflow &   $23.82$ &  \tabemph{} $\mathbf{25.20}_{0.01}$ &  \tabemph{} $\mathbf{25.20}_{0.01}$ &   \tabemph{} $\mathbf{25.21}_{0.01}$ &  $24.44_{0.01}$ &  $25.05_{0.01}$ &           $24.82_{0.01}$ \\
GLDv2         &   $51.43$ &  $62.85_{0.02}$ &  $62.85_{0.01}$ &           $62.92_{0.02}$ &  $53.94_{0.07}$ &  $56.64_{0.05}$ &   \tabemph{} $\mathbf{66.41}_{0.06}$ \\
EMNIST        &   $93.18$ &   \tabemph{} $\mathbf{94.13}_{0.01}$ &   \tabemph{} $\mathbf{94.13}_{0.01}$ &   \tabemph{} $\mathbf{94.13}_{0.01}$ &  $93.62_{0.04}$ &  $93.57_{0.05}$ &            \tabemph{} $\mathbf{94.13}_{0.03}$ \\
\bottomrule
\end{tabular}
 \end{adjustbox}
\end{center}
\end{table*}

\begin{figure*}[t]
\includegraphics[width=0.99\textwidth]{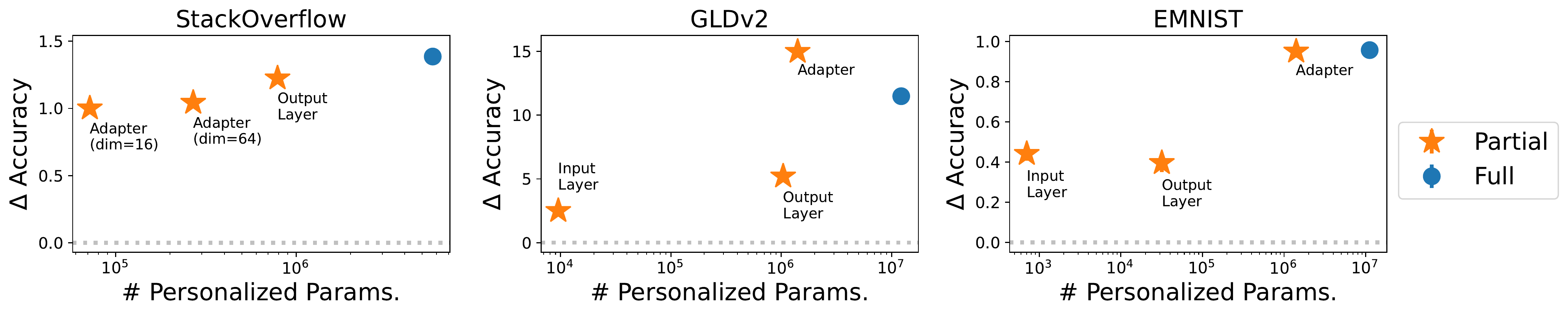}
\caption{\small
    Absolute change in accuracy (percentage points) due to personalization plotted against number of personal parameters (i.e., dimensionality of $v_i$). Note that the $x$-axis is in log scale. 
}
\label{fig:expt:main:partial-v-full}
\end{figure*}

\myparagraph{Model Partitioning for Partial Personalization}
We consider three partitioning schemes.
\begin{enumerate}[itemsep=0cm,leftmargin=1.6em,topsep=0cm,label=(\alph*)]
\item \textit{Input layer personalization}: 
    This architecture personalizes the input layer to learn personal representations, while the rest of the model is shared
    (Figure~\ref{fig:model-input}).
    For next-word prediction, we personalize the first transformer layer instead of the embedding layer.
\item \textit{Output layer personalization}:
    This architecture learns a shared representation but personalizes the prediction layer (Figure~\ref{fig:model-output}). 
    We personalize the last transformer layer for a transformer model instead of the output layer.
\item \textit{Adapter personalization}: 
    Each device adds personal adapter modules to a shared model (Figure~\ref{fig:model-adapter}). We use the transformer adapters of \citet{houlsby2019parameter}
    and the residual adapters of \citet{rebuffi2017learning}.
\end{enumerate}

\myparagraph{Algorithms and Experimental Pipeline}
We consider three full personalization baselines: 
(i) \textit{Finetune}, where each device finetunes its personal full model starting from a learned common model, 
(ii)~\textit{Ditto}~\citep{Li2021ditto}, which is finetuning with $\ell_2$ regularization, and, 
(iii) \textit{pFedMe}~\citep{dinh2020moreau} which minimizes the objective~\eqref{eqn:pfl-full-model}. 
All methods, including FedAlt, FedSim and the baselines are initialized with a global model trained with FedAvg.

\subsection{Experimental Results}

\myparagraph{Partial personalization nearly matches full personalization and can sometimes outperform it}
Table~\ref{table:expt:main} shows the \emph{average} test accuracy across all devices of different FL algorithms.
We see that on the StackOverflow dataset, 
output layer personalization ($25.05\%$) makes up nearly $90\%$ of the gap between the non-personalized baseline ($23.82\%$) and full personalization $(25.21\%)$.
On EMNIST, adapter personalization exactly matches full personalization. 
Most surprisingly, on GLDv2, adapter personalization outperforms full personalization by $3.5$pp (percentage points).

This success of adapter personalization can be explained partly by the nature of GLDv2. On average, the training data on each device contains $25$ classes out of a possible $2028$ while the testing data contains $10$ classes not seen in its own training data. These unseen classes account for nearly $23\%$ of all testing data.
Personalizing the full model is susceptible to ``forgetting'' the original task~\citep{kirkpatrick2017overcoming}, making it harder to get these unseen classes right. Such \emph{catastrophic forgetting} is worse when finetuning on a very small local dataset, as we often have in FL.
On the other hand, personalizing the adapters does not suffer as much from this issue~\citep{rebuffi2017learning}.  

\myparagraph{Partial personalization only requires a fraction of the parameters to be personalized}
Figure~\ref{fig:expt:main:partial-v-full} shows that the number of personal parameters required to compete with full personalization is rather small. 
On StackOverflow, personalizing $1.2\%$ of the parameters with adapters captures $72\%$ of the accuracy boost from personalizing all $5.7M$ parameters; this can be improved to nearly $90\%$ by personalizing $14\%$ of the parameters (output layer). 
Likewise, we match full personalization on EMNIST and exceed it on GLDv2 with adapters, personalizing $11.5$-$12.5\%$ of parameters.

\begin{table}[b]
    \caption{\small{
    Comparison of finetuning and partial personalization for ASR on Librispeech. We report the word error rate (WER, \%) on the test data, averaged across devices. Smaller values are better.
    }}
\label{table:expt:speech:main}
\begin{center}
\begin{adjustbox}{max width=\linewidth}
\small
\begin{tabular}{rrrr}
\toprule
Finetune & Input Layer &  Output Layer &  Adapter \\
\midrule
$15.55$ &   \tabemph{} $\mathbf{15.13}$ &   $15.53$ &    $15.50$ \\
\bottomrule
\end{tabular}

 \end{adjustbox}
\end{center}
\end{table}

\begin{table*}[t]
    \caption{\small{
    \pflam vs. \pflsu for partial personalization.
    ``FT (part.)'' means finetuning the personal parameters $v_i$ while fixing the shared parameters $u$ from FedAvg. The numbers are averaged over 5 random runs and the subscript denotes the standard deviation.
    }}
\label{table:expt:opt-algo:stateful}
\begin{center}
\begin{adjustbox}{max width=\linewidth}
\small
\renewcommand{\arraystretch}{1.2}
\begin{tabular}{llllllllll}
\toprule
{} & \multicolumn{3}{c}{StackOverflow} & \multicolumn{3}{c}{GLDv2} & \multicolumn{3}{c}{EMNIST} \\
\cmidrule(lr){2-4}
\cmidrule(lr){5-7}
\cmidrule(lr){8-10}
{} &                 FT (part.) &                   FedAlt &          FedSim &        FT (part.) &                   FedAlt &          FedSim &        FT (part.) &                   FedAlt &          FedSim \\
\midrule
Input Layer  & \tabemph{} $\mathbf{24.96}_{0.01}$ &           $24.44_{0.01}$ &  $24.81_{0.01}$ &  $51.97_{0.02}$ &  \tabemph{}$\mathbf{53.94}_{0.06}$ &  $53.64_{0.08}$ &  $93.29_{0.00}$ &  \tabemph{}$\mathbf{93.62}_{0.03}$ &  $93.55_{0.05}$ \\
Output Layer &           $24.93_{0.01}$ &  \tabemph{}$\mathbf{25.05}_{0.01}$ &  $25.02_{0.01}$ &  $53.21_{0.01}$ &  \tabemph{}$\mathbf{56.64}_{0.05}$ &  $56.24_{0.04}$ &  $93.37_{0.01}$ &  \tabemph{}$\mathbf{93.57}_{0.04}$ &  \tabemph{}$\mathbf{93.55}_{0.05}$ \\
Adapter      &           $24.71_{0.00}$ &  \tabemph{}$\mathbf{24.82}_{0.01}$ &  $24.74_{0.01}$ &  $63.86_{0.06}$ &  \tabemph{}$\mathbf{66.41}_{0.05}$ &  $66.35_{0.03}$ &  $93.66_{0.00}$ &  \tabemph{}$\mathbf{94.13}_{0.03}$ &  $94.07_{0.03}$ \\
\bottomrule
\end{tabular}
 \end{adjustbox}
\end{center}
\end{table*}

\myparagraph{The best personalized architecture is model and task dependent}
Table~\ref{table:expt:main} shows that personalizing the final transformer layer (denoted as ``Output Layer'') achieves the best performance for StackOverflow, while the residual adapter achieves the best performance for GLDv2 and EMNIST. 
In contrast, input layer personalization achieves the best performance for speech recognition, cf. Table~\ref{table:expt:speech:main}.

This variation is explained via the primary source of data heterogeneity across devices for each task.
The choice of the next word after a context can vary between users, so the output layer is the right component to personalize for this task. 
Likewise, there is greater heterogeneity in the audio of LibriSpeech (accent, tone, and voice of the speaker) than the text (standard literary English), so input layer personalization works best in this case.
This shows that the approach of personalizing a fixed model part, as in past works, is suboptimal. Our framework allows for the use of domain knowledge to determine customized personalization.

\myparagraph{Finetuning is competitive with other full personalization methods}
Full finetuning matches the performance of pFedMe and Ditto on StackOverflow and EMNIST. On GLDv2, however, pFedMe outperforms finetuning by $0.07$pp, but is still $3.5$pp worse than adapter personalization.

\myparagraph{\pflam outperforms \pflsu{} \rev{by a small but consistent margin}}
Table~\ref{table:expt:opt-algo:stateful} shows that \pflam almost always outperforms \pflsu by a small margin, e.g., $0.08$pp for StackOverflow/Adapter and $0.3$pp for GLDv2/Input Layer. \pflsu in turn yields a higher accuracy than simply finetuning the personal part of the model by a margin of $0.12$pp for StackOverflow/Output Layer and $2.55$pp for GLDv2/Adapter. 
\rev{
Furthermore, we observe that the difference between FedAlt and FedSim is much larger than the standard
deviation across runs. For instance, under output layer personalization for GLDv2, this difference is $0.4$pp ($= 8\times$ std).
}

\rev{
As a practical recommendation, 
\textit{we recommend using \pflam as a default, but it does not hurt much to use \pflsu}.
}

\begin{figure*}[t]
\includegraphics[width=0.98\textwidth]{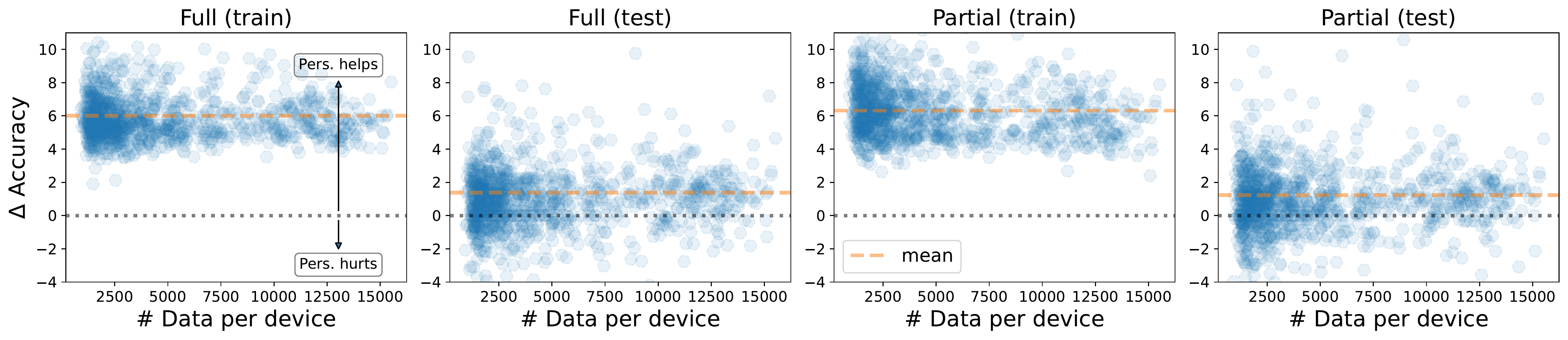}
\caption{\small{StackOverflow task:
    Scatter plot of change in training and test accuracy (pp) per-device versus the number of training samples on the device for (a) \textbf{Left}: full personalization with finetuning, and, (b) \textbf{Right}: partial personalization with the output layer.
}}
\label{fig:expt:scatter:main-so}
\end{figure*}

\subsection{Effects of Personalization on Generalization}

\myparagraph{Personalization hurts the test accuracy on some devices}
Figure~\ref{fig:expt:scatter:main-so} shows the change in training and test accuracy of each device, over a non-personalized model baseline.
We see that personalization leads to an improvement in training accuracy across all devices, but \emph{a reduction in test accuracy} on some of the devices. 
Devices whose testing performance is hurt by personalization are mostly on the left side of the plot, meaning that they have relatively small number of training samples. 
On the other hand, many devices with the most improved test accuracy also appear on the left side, signaling the benefit of personalization.
Therefore, there is a large variation of results for devices with few samples.

Additional results in Appendix~\ref{sec:a:expt-results} show that
using $\ell_2$ regularization as in~\eqref{eqn:pfl-full-model}, or weight decay does not mitigate this issue.
Increasing regularization strength (less personalization) can reduce the spread of per-device accuracy, but degrades the average accuracy.
Dropout does not fix this issue either. 

An ideal personalized method would boost performance on most of the devices without causing a reduction in (test) accuracy on any device. 
Realizing this goal calls for a sound statistical analysis for personalized FL and may require sophisticated methods for local performance diagnosis and structured regularization.

\section{Discussion} \label{sec:discussions}

In addition to a much smaller  memory footprint than full model personalization
and being less susceptible to catastrophic forgetting, partial model personalization has other advantages.
For example, it reduces the amount of communication between the server and the devices because only the shared parameters are transmitted.
While the communication savings may not be significant (especially when the personal parameters are only a small fraction of the full model), communicating only the shared parameters may have significant implications for privacy. 
Intuitively, it can be harder to infer private information from partial model information. This is especially the case if the more sensitive features of the data are processed through personal components of the model that are kept local at the devices. For example, we speculate that less noise needs to be added to the communicated parameters in order to satisfy differential privacy requirements~\citep{abadi2016deep}. 

\bibliography{bib/fl}
\bibliographystyle{abbrvnat}

\clearpage

\begingroup
\let\clearpage\relax 
\onecolumn 
\endgroup

\appendix
\addcontentsline{toc}{section}{Appendix} %
\part{Appendix} %
\parttoc %
\clearpage

\section{Convergence Analysis: Full Proofs} \label{sec:a:convergence}
We give the full convergence proofs here. 
The outline of this section is:
\begin{itemize}[itemsep=0cm,leftmargin=0.5cm,topsep=0cm]
    \item \S\ref{sec:a:proof:setup}: Review of setup and assumptions;
    \item \S\ref{sec:a:convrg:vfp}: Virtual Full Participation: Background and Details
    \item \S\ref{sec:a:proof:am}: Convergence analysis of \pflam and the full proof of Theorem~\ref{thm:pfl-am:main} (see Theorem~\ref{thm:pfl-am} and Corollary~\ref{cor:a:pfl-am});
    \item \S\ref{sec:a:proof:su}: Convergence analysis of \pflsu and the full proof of Theorem~\ref{thm:pfl-su:main} (see Theorem~\ref{thm:pfl-su} and Corollary~\ref{cor:a:pfl-su});
    \item \S\ref{sec:a:techn}: Technical lemmas used in the analysis.
\end{itemize} 
\subsection{Review of Setup and Assumptions} \label{sec:a:proof:setup}

We consider a federated learning system with $n$ devices. 
Let the loss function on device $i$ be $F_i(u,v_i)$, where 
$u \in \reals^{d_0}$ denotes the shared parameters across all devices and $v_i \in \reals^{d_i}$ denotes the personal parameters at device $i$. 
We aim to minimize the function
\begin{align} \label{eq:pfl:obj}
    F(u, V) := \frac{1}{n} \sum_{i=1}^n F_i(u, v_i) \,,
\end{align}
where $V = (v_1, \cdots, v_n)$ is a concatenation of all the personalized parameters. This is a special case of \eqref{eqn:pfl-partial} with the equal per-device weights, i.e., $\alpha_i = 1/n$.
Recall that we assume that $F$ is bounded from below by $F^\star$.

For convenience, we reiterate Assumptions~\ref{assmp:smoothness}, \ref{assmp:stoc-grad-var} and~\ref{assmp:grad-diversity} from the main paper as Assumptions~\ref{asmp:smoothness}, \ref{asmp:stoc-grad-var} and~\ref{asmp:grad-diversity} below respectively, 
with some additional comments and discussion.

\begin{assumptionP}{assmp:smoothness}[Smoothness]
\label{asmp:smoothness}
For each device $i=1,\ldots,n$, the objective $F_i$ is smooth, i.e., it is continuously differentiable and,
\begin{enumerate}[label=(\alph*),nolistsep,leftmargin=\widthof{(a) }]
    \item $u \mapsto \grad_u F_i(u, v_i)$ is $L_u$-Lipschitz for all $v_i$, 
    \item $v_i \mapsto \grad_v F_i(u, v_i)$ is $L_v$-Lipschitz for all $u$, 
    \item $v_i \mapsto \grad_u F_i(u, v_i)$ is $L_{uv}$-Lipschitz for all $u$, and, 
    \item $u \mapsto \grad_v F_i(u, v_i)$ is $L_{vu}$-Lipschitz for all $v_i$.
\end{enumerate}
Further, we assume for some $\chi > 0$ that 
\[
    \max\{L_{uv}, L_{vu}\} \le \chi\sqrt{L_u L_v} \,.
\]
\end{assumptionP}
The smoothness assumption is a standard one. 
We can assume without loss of generality that the cross-Lipschitz coefficients $L_{uv}, L_{vu}$ are equal. 
Indeed, if $F_i$ is twice continuously differentiable, 
we can show that $L_{uv}, L_{vu}$ are both 
equal to the operator norm $\norm{\grad^2_{uv} F_i(u, v_i)}_{\op}$ of the mixed second derivative matrix. 
Further, $\chi$ denotes the extent to which $u$ impacts the gradient of $v_i$ and vice-versa.

For concreteness, consider the full personalization setting of Eq.~\eqref{eqn:pfl-full-model}, where each $F_i$ is $L$-smooth; this is a special case of the formulation \eqref{eq:pfl:obj}, as we argue in \S\ref{sec:models}. In this case, a simple calculation shows that 
\[
    \chi^2 = \frac{\lambda}{\lambda+L} \le 1 \,.
\]

Our next assumption is about the variance of the stochastic gradients, and is standard in literature. 
Compared to the main paper, we adopt a more precise notation about stochastic gradients. 
\begin{assumptionP}{assmp:stoc-grad-var}[Bounded Variance]
\label{asmp:stoc-grad-var}
Let $\Dcal_i$ denote a probability distribution over the data space $\Zcal$ on device $i$. There exist functions $G_{i,u}$ and $G_{i,v}$ which are unbiased estimates of $\grad_u F_i$ and $\grad_v F_i$ respectively. That is, for all $u, v_i$:
\begin{align*}
    \expect_{z \sim \Dcal_i}\left[ G_{i,u}(u, v, z)\right]
    = \grad_u F_i(u, v_i),
    \quad\text{and}\quad
    \expect_{z \sim \Dcal_i}\left[ G_{i,v}(u, v, z)\right]
    = \grad_v F_i(u, v_i)\,.
\end{align*}
Furthermore, the variance of these estimators is at most $\sigma_u^2$ and $\sigma_v^2$ respectively. That is, 
\begin{align*}
    \expect_{z \sim \Dcal_i}\normsq*{G_{i,u}(u, v, z) - \grad_u F_i(u, v_i)} 
    &\le \sigma_u^2\,, %
    \\
    \expect_{z \sim \Dcal_i}\normsq*{G_{i,v}(u, v, z) - \grad_v F_i(u, v_i)} 
    &\le \sigma_v^2 \,.
\end{align*}
\end{assumptionP}
In practice, one usually has $G_{i, u}(u, v_i, z) = \grad_u f_i((u, v_i), z)$, which is the gradient of the loss on datapoint $z\sim\Dcal_i$ under the model $(u, v_i)$, and similarly for $G_{i, v}$.

Finally, we make a gradient diversity assumption.
\begin{assumptionP}{assmp:grad-diversity}[Partial Gradient Diversity]
\label{asmp:grad-diversity}
There exist $\delta \ge 0$ and $\rho \ge 0$ such that  for all $u$ and $V$, 
\begin{align} \label{eq:pfl:gradient-diversity}
    \frac{1}{n}\sum_{i=1}^n \normsq{\grad_u F_i(u, v_i) - \grad_u F(u, V)}
    \le \delta^2 + \rho^2 \normsq{\grad_u F(u, V)} \,.
\end{align}
\end{assumptionP}

This is a generalization of 
Assumption~\ref{asmp:grad-diversity} used in the main paper, which is a special case of Assumption~\ref{assmp:grad-diversity} with $\rho = 0$. We allow the partial gradient diversity to grow with the squared norm of the gradient with a factor of $\rho^2$.
This assumption is analogous to the bounded variance assumption (Assumption~\ref{asmp:stoc-grad-var}), but with the stochasticity coming from the sampling of devices. 
It characterizes how much local steps on one device help or hurt convergence globally. 

Similar gradient diversity assumptions are often used for analyzing non-personalized federated learning~\citep{koloskova2020unified,karimireddy2019scaffold}.
Finally, it suffices for the partial gradient diversity assumption to only hold at the iterates $(u\pow{t}, V\pow{t})$ generated by either FedSim or FedAlt.

\subsection{Virtual Full Participation: Background and Details}
\label{sec:a:convrg:vfp}

We recap the challenge of dependent random variables with FedAlt, and explain the technique of virtual full participation in some more detail. 
For this section, we assume full gradients on each device ($\sigma_u^2 = 0 = \sigma_v^2$) and a single local update per device ($\tau_u = 1 = \tau_v$). The only stochasticity in the algorithm comes from partial device participation, i.e., sampling $m$ devices in each round.

\myparagraph{Background: Stochastic Gradient Convergence Analysis}
Consider the minimization problem
\[
    \min_{w \in \reals^d} f(w)\,,
\]  
where the function $f: \reals^d \to \reals$ is $L$-smooth. 
Starting from some fixed $w\pow{0} \in \reals^d$, consider the stochastic gradient iterations $w\pow{t+1} = w\pow{t} - \gamma g\pow{t}$,
where 
$\gamma$ is a fixed learning rate, and
$g\pow{t}$ is an unbiased estimate of $\grad f(w\pow{t})$, i.e., 
$\expect[g\pow{t} | w\pow{t}] = \grad f(w\pow{t})$. 

Typical proofs of convergence proceed in the general nonconvex case with the smoothness bound 
\begin{align} \label{eq:convrg:smoothness}
    f(w\pow{t+1}) - f(w\pow{t})
    &\le \inp{\grad f(w\pow{t})}{w\pow{t+1} - w\pow{t}} 
    + \frac{L}{2} \normsq{w\pow{t+1} - w\pow{t}}
    \\ \nonumber
    &= -\gamma \inp{\grad f(w\pow{t})}{g\pow{t}} 
    + \frac{\gamma^2 L}{2} \normsq{g\pow{t}} \,.
\end{align}
Since the stochastic gradient $g\pow{t}$ is \emph{unbiased}, we get (under typical assumptions) an inequality
\begin{align} \label{eq:convrg:descent}
    \expect_t\left[ f(w\pow{t+1}) \right] - f(w\pow{t}) 
    &\le -c\gamma \, \normsq{\grad f(w\pow{t})} + O(\gamma^2) \,,
\end{align}
where $c > 0$ is some absolute constant and $\expect_t[\cdot] = \expect[ \, \cdot \, | w\pow{t}]$ takes an expectation only over the randomness in step $t$. 
The second term is a noise term that can be made small by choosing an appropriately small learning rate $\gamma$. Telescoping the inequality over $t$ and rearranging gives a convergence bound. 

The \textbf{key intuition} behind this proof   
is that the update is unbiased in linear term of the smoothness
upper bound~\eqref{eq:convrg:smoothness}.
The same intuition holds for most smooth nonconvex stochastic gradient convergence analyses~\citep{bottou2018optimization}.
In particular, this takes the following form in this case
\begin{align} \label{eq:convrg:unbiased}
    \expect_t \left[ \inp{\grad f(w\pow{t})}{w\pow{t+1} - w\pow{t}} \right]
    = \inp*{\grad f(w\pow{t})}{\expect_t[w\pow{t+1} - w\pow{t}]} \,.
\end{align}
This ensures that the contribution of the stochasticity occurs in a lower order $O(\gamma^2)$ term. 
As we shall see next, such an equality does not hold for FedAlt in the partial participation case due to dependent random variables.

\myparagraph{The Challenge in FedAlt with Partial Participation}
Consider the iterates $(u\pow{t}, V\pow{t})$ generated by FedAlt. 
The progress in one round is the combined progress of the $v$-step (call it $\Tcal_v$) and the $u$-step (call it $\Tcal_u$) so that
\begin{align*}
    F\left(u\pow{t+1}, V\pow{t+1}\right)
    - F\left(u\pow{t}, V\pow{t}\right)
    = 
    \underbrace{F\left(u\pow{t}, V\pow{t+1}\right)
    - F\left(u\pow{t}, V\pow{t}\right)}_{=:\Tcal_v}
    + \underbrace{F\left(u\pow{t+1}, V\pow{t+1}\right)
    - F\left(u\pow{t}, V\pow{t+1}\right)}_{=:\Tcal_u} \,.
\end{align*}
The analysis of the $v$-step is easy because the unbiasedness condition similar to \eqref{eq:convrg:unbiased} holds:
\[
    \expect_t 
        \inp*{\grad_V F\left(u\pow{t}, V\pow{t}\right)}{V\pow{t+1} - V\pow{t}}
    = \inp*{\grad_V F\left(u\pow{t}, V\pow{t}\right)}{\expect_t\left[
    V\pow{t+1} - V\pow{t} \right]} \,,
\] 
since $\expect_t[\cdot]$ takes an expectation w.r.t. the client sampling $S\pow{t}$. The recipe laid out earlier gives a descent condition similar to \eqref{eq:convrg:descent}.

For the $u$-step, an unbiasedness condition similar to \eqref{eq:convrg:unbiased} does not hold:
\[
    \expect_t 
        \inp*{\grad_u F\left(u\pow{t}, V\pow{t+1}\right)}{u\pow{t+1} - u\pow{t}}
    \neq \inp*{\expect_t\left[\grad_u F\left(u\pow{t}, V\pow{t+1}\right)\right]}{\expect_t\left[
    u\pow{t+1} - u\pow{t} \right]} \,.
\]
The expectation cannot pass into the inner product because $V\pow{t+1}$ and $u\pow{t+1}$ are dependent random variables.
Both are dependent on the device sampling $S\pow{t}$, as shown Figure~\ref{fig:a:fedalt:dependent_rv} (left).

\myparagraph{Virtual Full Participation}
We decouple these random variables by using virtual full participation.
Define a virtual iterate $\tV\pow{t+1}$ as the result of local $v$-updates 
as if \emph{every} device had participated.
Specifically, 
we introduce $\tV\pow{t+1}$ on the right hand side of the smoothness bound applied on $\Tcal_u$
to get 
\begin{align*}
    &F\bigl(u\pow{t+1}, V\pow{t+1}\bigr)
    - F\bigl(u\pow{t}, V\pow{t+1}\bigr)
    \le  E\pow{t} + 
    \bigl\langle\grad_u F(u\pow{t}, \tV\pow{t+1}),\, u\pow{t+1}\!\! -\! u\pow{t}\bigr\rangle
    + \frac{L_u}{2}\bigl\|u\pow{t+1} \!\! - \! u\pow{t}\bigr\|^2 \,,
\end{align*}
where $E\pow{t}$ is the error term from replacing $V\pow{t+1}$ with $\tV\pow{t+1}$
Since $\tV\pow{t+1}$ is independent of the client sampling $S\pow{t}$, we can now take an expectation $\expect_t[\cdot]$ over $u\pow{t+1}$ only, leading us to a situation similar to \eqref{eq:convrg:unbiased}; cf. Figure~\ref{fig:a:fedalt:dependent_rv} (right).

We bound the error term $E\pow{t}$ using Young's inequality and smoothness (Assumption~\ref{asmp:smoothness}) respectively as 
\begin{align*}
    E\pow{t}
    &=\inp{\grad_u F(u\pow{t}, V\pow{t+1}) - \grad_u F(u\pow{t}, \tV\pow{t+1})}{u\pow{t+1} - u\pow{t}} \\
    &\le    \frac{L_u}{2}\normsq{u\pow{t+1} - u\pow{t}} 
        + \frac{1}{2 L_u}\normsq{\grad_u F(u\pow{t}, V\pow{t+1}) - \grad_u F(u\pow{t}, \tV\pow{t+1})}
    \\
     &\le 
      \frac{L_u}{2}\normsq{u\pow{t+1} - u\pow{t}} + 
      \frac{\chi^2 L_v}{2n} \sum_{i=1}^n \normsq{\tv_i\pow{t+1}-v_i\pow{t+1}} \,.
\end{align*}
These two terms are similar to the quadratic terms we get from the smoothness upper bound. We can similarly show $\expect_t[E\pow{t}] = O(L_u\gamma_u^2 + \chi^2 L_v\gamma_v^2)$, so the error term from virtual full participation is also a lower order $O(\gamma^2)$ term.

\myparagraph{Virutal Iterates in Related Work}
Virtual or shadow iterates have long been used
in decentralized optimization~\citep{yuan2016convergence}, and 
have since been adopted in the analysis of federated
optimization algorithms in the non-personalized setting~\citep{li2020convergence_fedavg,koloskova2020unified,wang2021field}. 

In our notation, the shadow iterates used in \citep{koloskova2020unified,wang2021field} take the form 
\[
    \bar u_k\pow{t} = \frac{1}{n} \sum_{i=1}^n u_{i, k}\pow{t} \,,
\]
which is an average of the local versions of the shared parameters. This only makes sense for the case of full participation since $u_{i, k}\pow{t}$ is only defined for selected devices $i \in S\pow{t}$.
In partial participation case, \citet{li2020convergence_fedavg} define the virtual sequence $(\tu_{i,k}\pow{t})_{k=0}^{\tau_u}$
as the local SGD updates on all devices $i$ irrespective of whether they were selected. Then, they define the average 
\[
    \bar u_k\pow{t} = \frac{1}{n} \sum_{i=1}^n \tu_{i, k}\pow{t} \,.
\]
Their proof relies on the fact that $\expect_{S\pow{t}}[u\pow{t+1}] = \bar u_{\tau_u}\pow{t}$ due to the properties of the sampling. 

In contrast, we consider personalized federated learning --- the problem of dependent random variables only shows up in the analysis of FedAlt with partial participation, a setting not considered in prior works. 
We employ virtual \emph{personal} parameters $\tv_{i,k}\pow{t}$ to overcome this problem. We believe that this technique of decoupling dependent random variables can be of independent interest for (distributed) stochastic optimization, 
including personalized extensions 
of nonsmooth federated learning objectives~\cite{deng2020distributionally,pillutla2021federated} or more general multi-task learning formulations~\cite{misra2016cross}.

\subsection{Convergence Analysis of \pflam} \label{sec:a:proof:am}

We give the full form of \pflam in Algorithms~\ref{algo:pfl-am} for the general case of unequal $\alpha_i$'s but focus on $\alpha_i=1/n$ for the analysis.
Theorem~\ref{thm:pfl-am:main} of the main paper is a simplification of Corollary~\ref{cor:a:pfl-am} below, which in turn is proved based on Theorem~\ref{thm:pfl-am}.

Throughout this section, we use the constants 
\begin{align*}
    \sigma_{\mathrm{alt}, 1}^2 = \frac{\delta^2}{L_u}\left(1-\frac{m}{n}\right) + \frac{\sigma_u^2}{L_u} + \frac{\sigma_v^2 (m + \chi^2(n-m))}{L_v n}\,,
    \qquad
    \sigma_{\mathrm{alt}, 2}^2 = \frac{\sigma_u^2 + \delta^2}{L_u}(1-\tau_u^{-1}) + \frac{\sigma_v^2 m}{L_v n}(1-\tau_v^{-1}) + \frac{\chi^2 \sigma_v^2}{L_v} \,.
\end{align*}
We also recall the definitions
    \[
        \Delta_u\pow{t} = \normsq*{\grad_u F\left(u\pow{t}, V\pow{t+1} \right)}\,,
        \quad\text{and},\quad 
        \Delta_v\pow{t} = \frac{1}{n}\sum_{i=1}^n \normsq*{\grad_v F_i\left(u\pow{t}, v_i\pow{t}\right)} \,.
    \]

\begin{theorem}[\textbf{Convergence of \hlalt{FedAlt}}] 
\label{thm:pfl-am}
Suppose Assumptions~\ref{asmp:smoothness}, \ref{asmp:stoc-grad-var} and~\ref{asmp:grad-diversity} hold and the learning rates in FedAlt are chosen as     $\gamma_u = \lr / (L_u \tau_u)$ and 
    $\gamma_v = \lr / (L_v \tau_v)$, with
\[
        \lr \le \min\left\{\frac{1}{24(1+\rho^2)}, \frac{m}{128\chi^2(n-m)}, \sqrt{\frac{m}{\chi^2 n}} \right\}\,.
\]
    Then, ignoring absolute constants, we have
        \begin{align*}
        &\frac{1}{T}\sum_{t=0}^{T-1} \left(\frac{1}{ L_u} \expect\bigl[\Delta_u\pow{t}\bigr]
        + \frac{m}{nL_v} \expect\bigl[\Delta_v\pow{t}\bigr] \right)
        \le
        \frac{\Delta F_0}{\lr T} 
        + 
        \lr \, \sigma_{\mathrm{alt}, 1}^2
        + 
        \lr^2 \, \sigma_{\mathrm{alt}, 2}^2 \,.
    \end{align*}
\end{theorem}

Before proving the theorem, we have the corollary with  optimized learning rates. 
\begin{corollary}[\textbf{Final Rate of FedAlt}] \label{cor:a:pfl-am}
    Consider the setting of Theorem~\ref{thm:pfl-am} and let the number of rounds $T$ be known in advance.
    Suppose we set the learning rates $\gamma_u = \lr/(\tau L_u)$ and $\gamma_v = \lr/(\tau L_v)$, where (ignoring absolute constants), 
    \begin{align*}
        \lr = 
        \left(\frac{\Delta F_0}{T \sigma_{\mathrm{alt}, 1}^2 }\right)^{1/2}
        \bigwedge 
        \left(
            \frac{\Delta F_0^2}{T^2 \, \sigma_{\mathrm{alt}, 2}^2}
        \right)^{1/3}
        \bigwedge \frac{1}{1+\rho^2}
        \bigwedge \frac{m}{\chi^2(n-m)}
        \bigwedge \sqrt{\frac{m}{\chi^2 n}}
        \,.
    \end{align*}
    We have, ignoring absolute constants, 
    \begin{align*}
        &\frac{1}{T}\sum_{t=0}^{T-1}
        \left(
        \frac{1}{L_u} \expect\normsq*{\grad_u F\left(u\pow{t}, V\pow{t}\right)}
        + 
        \frac{m}{L_v n^2} \sum_{i=1}^n \expect\normsq*{\grad_v F_i\left(u\pow{t}, v_i\pow{t}\right)}
        \right)
        \le  \\
        &\frac{\left(\Delta F_0 \, \sigma_{\mathrm{alt}, 1}^2 \right)^{1/2}}{\sqrt{T}} 
        +  \frac{\left(\Delta F_0^2 \, \sigma_{\mathrm{alt}, 2}^2     \right)^{1/3}}{T^{2/3}}  
        + 
        \frac{\Delta F_0}{T}\left(
            1+\rho^2 + \chi^2\left(\frac{n}{m} - 1\right) + \sqrt{\chi^2 \frac{n}{m}}
        \right)
        \,.
    \end{align*}
\end{corollary}
\begin{proof}
    The proof follows from invoking Lemma~\ref{lem:sfl:best-params-2}
    on the bound of Theorem~\ref{thm:pfl-am}.
\end{proof}

\begin{remark}[\textbf{Asymptotic Rate}]
    The asymptotic $1/\sqrt{T}$ rate of Theorem~\ref{thm:pfl-am:main} is achieved when the $1/T$ term is dominated by the $1/\sqrt{T}$ term. This happens when (ignoring absolute constants)
    \[
        T \ge \frac{\Delta F_0}{\sigma_{\mathrm{alt}, 1}^2}
        \left(  1 + \rho^4 + \chi^4 \frac{n^2}{m^2} \right) \,.
    \]
\end{remark}

\begin{algorithm}[t!]
	\caption{\pflam: Alternating updates of shared and personalized parameters}
	\label{algo:pfl-am}
\begin{algorithmic}[1]
		\STATE \textbf{Input:} Initial iterates $u\pow{0}, V\pow{0}$,
		    Number of communication rounds $T$, 
		    Number of devices per round $m$, 
		    Number of local updates $\tau_u, \tau_v$,
		    Local step sizes $\gamma_u, \gamma_v$, 
	    \FOR{$t=0, 1, \cdots, T-1$}
	        \STATE Sample $m$ devices from $[n]$ without replacement in $S\pow{t}$\label{line:pfl-am:sample}
	         \FOR{each selected device $i \in S\pow{t}$ in parallel}
	        \label{line:pfl-am:local}
	            \STATE Initialize $v_{i, 0}\pow{t} = v_i\pow{t}$
	            \FOR{$k=0, \cdots, \tau_v-1$} 
	            \STATE \texttt{// Update personal parameters}
	                \STATE Sample data $z_{i,k}\pow{t} \sim \Dcal_i$
	            \label{line:pfl-am:local-v}
	    	      \STATE $v_{i, k+1}\pow{t} = v_{i,k}\pow{t} - \gamma_v G_{i,v}(u\pow{t}, v_{i,k}\pow{t}, z_{i,k}\pow{t})$
	    	    \ENDFOR
	    	    \STATE Update $v_i\pow{t+1} = v_{k, \tau_v}\pow{t}$
	            \STATE Initialize $u_{i, 0}\pow{t} = u\pow{t}$
	    	    \FOR{$k=0, \cdots, \tau_u-1$}
	    	    \STATE \texttt{// Update shared parameters}
	    	    \label{line:pfl-am:local-w}
	    	      \STATE $u_{i, k+1}\pow{t} = u_{i,k}\pow{t} - \gamma_u G_{i,u}(u_{i,k}\pow{t}, v_i\pow{t+1}, z_{i,k}\pow{t})$
	    	        
	    	    \ENDFOR
	    	    \STATE Update $u_i\pow{t+1} = u_{i, \tau_u}\pow{t}$
	    	\ENDFOR
	    	\STATE Update $u\pow{t+1} = {\sum_{i \in S\pow{t}} \alpha_i u_{i}\pow{t+1}} / {\sum_{i \in S\pow{t}} \alpha_i}$ at the server with secure aggregation
	    	\label{line:pfl-am:aggregation}
	    \ENDFOR
	    \STATE \textbf{return} $u\pow{T}, v_1\pow{T}, \cdots, v_n\pow{T}$
\end{algorithmic}
\end{algorithm}

We now prove Theorem~\ref{thm:pfl-am}. 
\begin{proof}[Proof of Theorem~\ref{thm:pfl-am}]
The proof mainly applies the smoothness upper bound to write out a descent condition with suitably small noise terms. We start with some notation.

We introduce the notation
$\widetilde\Delta_u\pow{t}$
as the analogue of $\Delta_u\pow{t}$
with the virtual variable $\tV\pow{t+1}$:
\[
    \widetilde\Delta_u\pow{t} = \normsq*{\grad_u F\left(u\pow{t}, \tV\pow{t+1} \right)}\,.
\]

\myparagraph{Notation}
    Let $\Fcal\pow{t}$ denote the $\sigma$-algebra generated by $\left(u\pow{t}, V\pow{t}\right)$ and
    denote $\expect_t[\,\cdot\,] = \expect[\,\cdot\,| \Fcal\pow{t}]$. 
    For all devices, including those not selected in each round, we define virtual sequences $\tu_{i,k}\pow{t}, \tv_{i,k}\pow{t}$ 
    as the SGD updates in Algorithm~\ref{algo:pfl-am} 
    for all devices regardless of whether they are selected.
    For the selected devices $i \in S\pow{t}$, we have 
    $v_{i,k}\pow{t} = \tv_{i,k}\pow{t}$ and
    $u_{i,k}\pow{t} = \tu_{i,k}\pow{t}$.
    Note now that the random variables $\tu_{i,k}\pow{t}, \tv_{i,k}\pow{t}$ are independent of the device selection $S\pow{t}$.
    Finally, we have that the updates for the selected devices $i \in S\pow{t}$ are given by 
    \begin{align*}
        v_i\pow{t+1} &= v_i\pow{t} - \gamma_v \sum_{k=0}^{\tau_v-1} G_{i,v}\left(u\pow{t}, \tv_{i,k}\pow{t}, z_{i,k}\pow{t} \right) \,,
    \end{align*}
    and the server update is given by
    \begin{align*}
        u\pow{t+1} &= u\pow{t} - \frac{\gamma_u}{m} \sum_{i \in S\pow{t}}\sum_{k=0}^{\tau_u-1} G_{i,u}\left(\tu_{i,k}\pow{t}, \tv_{i,\tau_v}\pow{t}, z_{i,k}\pow{t} \right)\,.
    \end{align*}

    \myparagraph{Proof Outline and the Challenge of Dependent Random Variables}
    We start with 
    \begin{align} \label{eq:pfl-am:pf:1}
        \begin{aligned}
        F\left(u\pow{t+1}, V\pow{t+1}\right)
        - F\left(u\pow{t}, V\pow{t}\right)
        =&\, F\left(u\pow{t}, V\pow{t+1}\right)
        - F\left(u\pow{t}, V\pow{t}\right) \\
        &+ F\left(u\pow{t+1}, V\pow{t+1}\right)
        - F\left(u\pow{t}, V\pow{t+1}\right) \,.
        \end{aligned}
    \end{align}
    The first line corresponds to the effect of the $v$-step and the second line to the $u$-step. The former is easy to handle with standard techniques that rely on the smoothness of $F\left(u\pow{t}, \cdot\right)$. 
    The latter is more challenging. 
    In particular, the smoothness bound for the $u$-step gives us
    \begin{align*}
        F&\left(u\pow{t+1}, V\pow{t+1}\right)
        - F\left(u\pow{t}, V\pow{t+1}\right)
        \le 
        \inp*{\grad_u F\left(u
        \pow{t}, V\pow{t+1}\right)}{u\pow{t+1} - u\pow{t}}
        + \frac{L_u}{2}\normsq*{u\pow{t+1} - u\pow{t}} \,.
    \end{align*}
    The standard proofs of convergence of stochastic gradient methods rely on the fact that we can take an expectation w.r.t. the sampling $S\pow{t}$ of devices for the first order term. However, both $V\pow{t+1}$ and $u\pow{t+1}$ depend on the sampling $S\pow{t}$ of devices. Therefore, we cannot directly take an expectation with respect to the sampling of devices in $S\pow{t}$. 
    
    \myparagraph{Virtual Full Participation to Circumvent Dependent Random Variables}
    The crux of the proof lies in replacing $V\pow{t+1}$ in the analysis of the $u$-step with the virtual iterate $\tV\pow{t+1}$ so as to move all the dependence of the $u$-step on $S\pow{t}$ to the $u\pow{t+1}$ term. This allows us to take an expectation; it remains to carefully bound the resulting error terms.
    
    Finally, we will arrive at a bound of the form
    \[
    \frac{1}{T}\sum_{t=0}^{T-1}\left( \frac{\gamma_u \tau_u}{8} \expect[\widetilde\Delta_u\pow{t}]
        + \frac{\gamma_v \tau_v m}{16n} \expect[\Delta_v\pow{t}]\right)
        \le \frac{\Delta F_0}{T} + O(\gamma_u^2 + \gamma_v^2) \,.
    \]
    Next, we translate this bound from gradient $\expect[\widetilde \Delta_u\pow{t}]$ of the virtual $\tV\pow{t+1}$ to $\expect[\Delta_u\pow{t}]$, which is the gradient computed at the actual iterate $V\pow{t}$. A careful analysis shows that we only incur a lower order term of $O(\gamma_u\gamma_v^2)$ in this translation. 
    Choosing $\gamma_u$ and $\gamma_v$ small enough will give us the final result.
    
    \myparagraph{Analysis of the $u$-Step with Virtual Full Participation}
    We introduce the virtual iterates $\tV\pow{t+1}$ into the analysis of the $u$-step as follows:
    \begin{align*}
        F&\left(u\pow{t+1}, V\pow{t+1}\right)
        - F\left(u\pow{t}, V\pow{t+1}\right) \\
        &\le 
        \inp*{\grad_u F\left(u
        \pow{t}, V\pow{t+1}\right)}{u\pow{t+1} - u\pow{t}}
        + \frac{L_u}{2}\normsq*{u\pow{t+1} - u\pow{t}} \\
        &
        \begin{aligned}
        &=\,\inp*{\grad_u F\left(u
        \pow{t}, \tV\pow{t+1}\right)}{u\pow{t+1} - u\pow{t}} + \frac{L_u}{2}\normsq*{u\pow{t+1} - u\pow{t}} \\ &\qquad\qquad+ 
         \inp*{\grad_u F\left(u
        \pow{t}, V\pow{t+1}\right) - \grad_u F\left(u
        \pow{t}, \tV\pow{t+1}\right)}{u\pow{t+1} - u\pow{t}} 
        \end{aligned} \\
        &
        \begin{aligned}
        &\le \inp*{\grad_u F\left(u
        \pow{t}, \tV\pow{t+1}\right)}{u\pow{t+1} - u\pow{t}} + L_u\normsq*{u\pow{t+1} - u\pow{t}} \\ &\qquad\qquad+ 
         \frac{1}{2L_u} \normsq*{\grad_u F\left(u
        \pow{t}, V\pow{t+1}\right) - \grad_u F\left(u
        \pow{t}, \tV\pow{t+1}\right)}
        \end{aligned} \\
        &\le
        \underbrace{
        \inp*{\grad_u F\left(u
        \pow{t}, \tV\pow{t+1}\right)}{u\pow{t+1} - u\pow{t}}}_{\Tcal_{1,u}}
        + \underbrace{
        L_u\normsq*{u\pow{t+1} - u\pow{t}}
        }_{\Tcal_{2, u}}
        + \underbrace{
         \frac{\chi^2 L_v}{2n} \sum_{i=1}^n \normsq*{\tv_i\pow{t+1}-v_i\pow{t+1}}
         }_{\Tcal_{3, u}} \,.
    \end{align*}
    The last two inequalities follow from Young's inequality and Lipschitzness of $V \mapsto \grad_u F(u, V)$ respectively. 
    
    We have now successfully eliminated the dependence of the first-order term $\Tcal_{1,u}$ on $V\pow{t+1}$. The virtual iterates $\tV\pow{t+1}$ are now independent of $S\pow{t}$. This allows us to take an expectation w.r.t. the sampling $S\pow{t}$ of the devices. 
    
    We bound each of these terms in Claims~\ref{claim:pflam:w-step-1}
    to~\ref{claim:pflam:w-step-3} below to get
    \begin{align*}
        \expect_t\Bigg[F&\left(u\pow{t+1}, V\pow{t+1}\right)
        - F\left(u\pow{t}, V\pow{t+1}\right)\Bigg] \\
        &\le -\frac{\gamma_u \tau_u}{4} 
        \expect_t[\widetilde\Delta_u\pow{t}]
        + \underbrace{\frac{2\gamma_u L_u^2}{n} \sum_{i=1}^n \sum_{k=0}^{\tau_u-1}\expect_t\normsq*{\tu_{i,k}\pow{t} - u\pow{t}}}_{=:\Tcal_{2,u}'}
        +
        4 \gamma_v^2 \tau_v^2 L_v \sigma_v^2 \chi^2(1-m/n) 
        \\ &\quad + \frac{L_u\gamma_u^2 \tau_u^2}{m}\left(\sigma_u^2 + 
        3\delta^2\left(1-\frac{m}{n}\right)\right)+ 8 \gamma_v^2 \tau_v^2 L_v \chi^2(1-m/n) \Delta_v\pow{t} \,.
    \end{align*}
    Note that we used the fact that $24L_u \gamma_u\tau_u(1 + \rho^2) \le 1$ to simply the coefficients of some of the terms above. 
    The second term has also been referred to as client drift in the literature;
    we bound it with Lemma~\ref{lem:techn:client-drift} and invoke the assumption on gradient diversity (Assumption~\ref{asmp:grad-diversity}) to get
    \begin{align*}
        \Tcal_{2,u}'
        &\le \frac{16\gamma_u^3 L_u^2 \tau_u(\tau_u-1)}{n} \sum_{i=1}^n \expect_t \normsq*{\grad_u F_i\left(u\pow{t}, \tv_i\pow{t+1}\right)} + 8\gamma_u^3 L_u^2 \tau_u^2(\tau_u-1) \sigma_u^2 \\
        &\le \frac{16\gamma_u^3 L_u^2 \tau_u(\tau_u-1)}{n} \left(\delta^2 + \rho^2 \expect_t \normsq*{\grad_u F\left(u\pow{t}, \tV\pow{t+1}\right)}\right) + 8\gamma_u^3 L_u^2 \tau_u^2(\tau_u-1) \sigma_u^2 \,.
    \end{align*}
    Plugging this back in, we get, 
    \begin{align*}
        \expect_t\Bigg[F&\left(u\pow{t+1}, V\pow{t+1}\right)
        - F\left(u\pow{t}, V\pow{t+1}\right)\Bigg] \\
        &\le -\frac{\gamma_u \tau_u}{8} \expect_t[\widetilde\Delta_u\pow{t}] + 
        \frac{L_u \gamma_u^2 \tau_u^2}{m}\left( \sigma_u^2 + 2\delta^2(1-m/n)\right)
        + 4 \gamma_v^2 \tau_v^2 L_v \sigma_v^2 \chi^2 (1-m/n) \\
        &\qquad + 8\gamma_v^2 \tau_v^2 L_v \chi^2 (1-m/n) \Delta_v\pow{t} 
        + 8\gamma_u^2 L_u^3 \tau_u^2 (\tau_u-1)(\sigma_u^2 + 2\delta_u^2) \,.
    \end{align*}
     Note that we used $128 \gamma_u^2 L_u^2 \tau_u(\tau_u-1)\rho^2 \le 1$, which is implied by $24L_u \gamma_u\tau_u(1 + \rho^2) \le 1$.
    
    \myparagraph{Bound with the Virual Iterates}
    We plug this analysis of the $u$-step and Claim~\ref{claim:pflam:v-step} for the $v$-step into \eqref{eq:pfl-am:pf:1} next. We also simplify some coefficients using $128\gamma_v\tau_vL_v\chi^2(n/m-1) \le 1$. 
    This gives us
    \begin{align*}
        &\expect_t\Bigg[F\left(u\pow{t+1}, V\pow{t+1}\right)
        - F\left(u\pow{t}, V\pow{t}\right)\Bigg] \\
        &
        \begin{aligned}
        \le&\, -\frac{\gamma_u\tau_u}{8}\expect_t[\widetilde\Delta_u\pow{t}]
        - \frac{\gamma_v\tau_v m}{16n}\expect_t[\Delta_v\pow{t}]
        + 4\gamma_v^2L_v\tau_v^2\sigma_v^2\left( \frac{m}{n} + \chi^2(1-m/n)\right)
        \\ &+
        \frac{\gamma_u^2 L_u \tau_u^2}{m}\left(\sigma_u^2 + 2\delta^2(1-m/n)\right)
        + 8\gamma_u^3 L_u^2 \tau_u^2(\tau_u-1)(\sigma_u^2 + 2\delta^2)
        + \frac{4 \gamma_v^3 L_v^2 \tau_v^2(\tau_v-1)\sigma_v^2 m}{n} \,.
        \end{aligned}
    \end{align*}
    Taking an unconditional expectation, summing it over $t=0$ to $T-1$ and rearranging this gives
    \begin{align} \label{eq:pfl-am:original}
        \frac{1}{T}\sum_{t=0}^{T-1}&\left( \frac{\gamma_u \tau_u}{8} \expect[\widetilde\Delta_u\pow{t}]
        + \frac{\gamma_v \tau_v m}{16n} \expect[\Delta_v\pow{t}]\right)
        \\  \nonumber
        &
        \begin{aligned}
        \le \frac{\Delta F_0}{T} &
            + 4 \gamma_v^2 L_v \tau_v^2 \sigma_v^2 \left(\frac{m}{n} + \chi^2(1-m/n)\right)
            + \frac{\gamma_u^2 L_u \tau_u^2}{m} \left(\sigma_u^2 + 2\delta^2(1-m/n)\right) \\
            &+ 8\gamma_u^3 L_u^2 \tau_u^2(\tau_u-1)(\sigma_u^2 + 2\delta^2) + 
            \frac{4 \gamma_v^3 L_v^2 \tau_v^2(\tau_v-1)\sigma_v^2 m}{n}\,.
        \end{aligned}
    \end{align}
    This is a bound in terms of the virtual iterates $\tV\pow{t+1}$. However, we wish to show a bound in terms of the actual iterate $V\pow{t}$.
    
    \myparagraph{Obtaining the Final Bound}
    It remains now to relate $\widetilde\Delta_u\pow{t}$ with $\Delta_u\pow{t}$. 
    Using the Cauchy-Schwartz inequality and smoothness, we have, 
\begin{align*}
    \expect_t\Big\|\grad_u F\left(u\pow{t}, V\pow{t}\right) - \grad_u F&\left(u\pow{t}, \tV\pow{t+1}\right)\Big\|^2
    \\
    &\le 
    \frac{1}{n}\sum_{i=1}^n \expect_t \normsq*{\grad_u F_i\left(u\pow{t}, v_i\pow{t}\right) - \grad_u F_i\left(u\pow{t}, \tv_i\pow{t+1}\right)} \\
    &\le \frac{\chi^2 L_u L_v}{n} \sum_{i=1}^n \expect_t\normsq*{\tv_i\pow{t+1} - v_i\pow{t}} \\
    &\le \frac{\chi^2 L_u L_v}{n} \sum_{i=1}^n\left( 16\gamma_v^2 \tau_v^2 \normsq*{\grad_v F_i\left(u\pow{t}, v_i\pow{t}\right)} + 8\gamma_v^2 \tau_v^2 \sigma_v^2 \right) \\
    &= 8\gamma_v^2 \tau_v^2 \sigma_v^2 \chi^2 L_u L_v + 16 \gamma_v^2 \tau_v^2 \chi^2 L_u L_v \Delta_v\pow{t} \,,
\end{align*}
where the last inequality followed from Lemma~\ref{lem:techn:client-drift-2}.
Using 
\[
 \normsq*{\grad_u F\left(u\pow{t}, V\pow{t}\right)}
 \le 
 2  \normsq*{\grad_u F\left(u\pow{t}, V\pow{t}\right) - \grad_u F\left(u\pow{t}, \tV\pow{t+1}\right)}
 + 2  \normsq*{\grad_u F\left(u\pow{t}, \tV\pow{t+1}\right)} \,, 
\]
we get, 
\[
    \expect[\Delta_u\pow{t}] \le 2\, \expect[\widetilde \Delta_u\pow{t}] + 16\gamma_v^2 \tau_v^2 \sigma_v^2 \chi^2 L_u L_v + 32 \gamma_v^2 \tau_v^2 \chi^2 L_u L_v \, \expect[\Delta_v\pow{t}] \,.
\]
Therefore, we get, 
\begin{align*}
    \frac{\gamma_u \tau_u}{16}&  \expect[\Delta_u\pow{t}] 
    + \frac{\gamma_v \tau_v m}{32 n} \expect[\Delta_v\pow{t}] 
      \\
    &\le \frac{\gamma_u \tau_u}{8} \expect[\widetilde \Delta_u\pow{t}] 
    + \frac{\gamma_v \tau_v m}{16 n}\left(\frac{1}{2} + \frac{32\lr^2\chi^2m}{n}\right) \expect[\Delta_v\pow{t}]
    + \gamma_u \tau_u\gamma_v^2 \tau_v^2 \sigma_v^2 \chi^2 L_u L_v  \\
    &\le \frac{\gamma_u \tau_u}{8} \expect[\widetilde \Delta_u\pow{t}] 
    + \frac{\gamma_v \tau_v m}{16 n} \expect[\Delta_v\pow{t}]
    + \gamma_u \tau_u\gamma_v^2 \tau_v^2 \sigma_v^2 \chi^2 L_u L_v   \,,
\end{align*}
where we used 
$\frac{32\lr^2\chi^2m}{n} \le 1/2$, which is one of the conditions we assume on $\lr$. 

Summing this up and plugging in \eqref{eq:pfl-am:original} gives 
\begin{align*}
     \frac{1}{T}&\sum_{t=0}^{T-1}\left(\frac{\gamma_u \tau_u}{16} \expect[\Delta_u\pow{t}] 
    + \frac{\gamma_v \tau_v m}{32 n} \expect[\Delta_v\pow{t}] \right)
    \\
    &\le 
    \frac{1}{T}\sum_{t=0}^{T-1}\left(
    \frac{\gamma_u \tau_u}{8} \expect[\widetilde \Delta_u\pow{t}] 
    + \frac{\gamma_v \tau_v m}{16 n} \expect[\Delta_v\pow{t}]
    \right)
    + \gamma_u \tau_u\gamma_v^2 \tau_v^2 \sigma_v^2 \chi^2 L_u L_v
    \\&
    \begin{aligned}
        \le \frac{\Delta F_0}{T} &
            + 4 \gamma_v^2 L_v \tau_v^2 \sigma_v^2 \left(\frac{m}{n} + \chi^2(1-m/n)\right)
            + \frac{\gamma_u^2 L_u \tau_u^2}{m} \left(\sigma_u^2 + 2\delta^2(1-m/n)\right) \\
            &+ 8\gamma_u^3 L_u^2 \tau_u^2(\tau_u-1)(\sigma_u^2 + 2\delta^2) + 
            \frac{4 \gamma_v^3 L_v^2 \tau_v^2(\tau_v-1)\sigma_v^2 m}{n} + \gamma_u \tau_u\gamma_v^2 \tau_v^2 \sigma_v^2 \chi^2 L_u L_v\,.
        \end{aligned}
\end{align*}
    Plugging in $\gamma_u = \lr/(L_u\tau_u)$ and $\gamma_v = \lr/(L_v\tau_v)$ completes the proof.
\end{proof}

The analysis of each of the terms in the $u$-step is given in the following claims.

\smallskip
\begin{claim}[Bounding $\Tcal_{1,u}$] \label{claim:pflam:w-step-1}
    We have,
    \[
    \expect_t\left[ \Tcal_{1,u} \right]
    \le -\frac{\gamma_u \tau_u}{2} \expect_t \normsq*{\grad_u F\left(u\pow{t}, \tV\pow{t+1} \right)} + \frac{\gamma_u L_u^2}{n} \sum_{i=1}^n \sum_{k=0}^{\tau_u -1} \expect_t\normsq*{\tu_{i,k}\pow{t} - u\pow{t}} \,.
    \]
\end{claim}
\begin{proof}
    For $i \in S\pow{t}$, we have that $\tu_{i,k}\pow{t} = u_{i,k}\pow{t}$. Therefore, we have, 
    \begin{align*}
        \expect_t[\Tcal_{1,u}] 
        &= -\gamma_u \expect_t \inp*{\grad_u F\left( u\pow{t}, \tV\pow{t+1} \right)}{\frac{1}{m}\sum_{i \in S\pow{t}}
        \sum_{k=0}^{\tau_u-1} \grad_u F_i\left(\tu_{i,k}\pow{t}, \tv_i\pow{t+1}\right)
        } \,.
    \end{align*}
    Using that $\tu_{i,k}\pow{t}$ is independent of $S\pow{t}$, we get,
    \begin{align*}
        \expect_t&[\Tcal_{1,u}] 
        = -\gamma_u \expect_t \inp*{\grad_u F\left( u\pow{t}, \tV\pow{t+1} \right)}{\frac{1}{n}\sum_{i=1}^n
        \sum_{k=0}^{\tau_u-1} \grad_u F_i\left(\tu_{i,k}\pow{t}, \tv_i\pow{t+1}\right)
        } \\ 
        &
        \begin{aligned}
        =&\, - \gamma_u \tau_u \expect_t \normsq*{\grad_u F\left(u\pow{t}, \tV\pow{t+1}\right)} \\&
        - \gamma_u \sum_{k=0}^{\tau_u-1} \expect_t\inp*{\grad_u F\left(u\pow{t}, \tV\pow{t+1}\right)}{\frac{1}{n}\sum_{i=1}^n \grad_u F_i\left(\tu_{i,k}\pow{t}, \tv\pow{t+1}\right) - \grad_u F_i\left(u\pow{t}, \tv\pow{t+1} \right)}
        \end{aligned} 
    \end{align*}
    Invoking $\inp{x}{y}\le \normsq{x}/2 + \normsq{y}/2$ for vectors $x, y$ followed by smoothness completes the proof.
\end{proof}

\smallskip
\begin{claim}[Bounding $\Tcal_{2,u}$] \label{claim:pflam:w-step-2}
    We have, 
    \begin{align*}
    \expect_t\left[ \Tcal_{2, u} \right]
    \le &\,  3L_u \gamma_u^2 \tau_u^2 \left(1 + \frac{2\rho^2}{m}(1-m/n) \right) \expect_t\normsq*{\grad_u F\left(u\pow{t}, \tV\pow{t+1} \right)} \\ &
    + \frac{3L_u^2 \gamma_u^2 \tau_u}{n} \sum_{i=1}^n \sum_{k=0}^{\tau_u-1} \expect_t\normsq*{\tu_{i,k}\pow{t} - u\pow{t}} + \frac{6L_u\gamma_u^2 \tau_u^2\delta^2}{m}(1-m/n) \,.
    \end{align*}
\end{claim}
\begin{proof}
    We use $\expect\normsq{z} = \normsq{\expect[z]} + \expect\normsq{z - \expect[z]}$ for a random vector $z$ to get 
    \begin{align*}
        \expect_t[\Tcal_{2, u}]
        \le \frac{L_u \gamma_u^2 \tau_u^2\sigma_u^2}{m} + L_u\gamma_u^2 \tau_u \sum_{k=0}^{\tau_u-1}
        \expect_t\underbrace{\normsq*{\frac{1}{m} \sum_{i\in S\pow{t}} \grad_u F_i\left(\tu_{i,k}\pow{t}, \tv_i\pow{t+1}\right)}}_{=: \Tcal'_k}  \,.
    \end{align*}
    We break the term $\Tcal'_k$ as
    \begin{align*}
        \Tcal'_k \le &\, 3\normsq*{\frac{1}{m}\sum_{i\in S\pow{t}}\left( \grad_u F_i\left(\tu_{i,k}\pow{t}, \tv_i\pow{t+1}\right) - \grad_u F_i\left(u\pow{t}, \tv_i\pow{t+1}  \right)\right)
        } \\
        &+ 3\normsq*{\frac{1}{m}\sum_{i\in S\pow{t}} \grad_u F_i\left(u\pow{t}, \tv_i\pow{t+1}\right) - \grad_u F\left(u\pow{t}, \tV\pow{t+1}\right) }
        + 3 \normsq*{\grad_u F\left(u\pow{t}, \tV\pow{t+1} \right)} \,.
    \end{align*}
    For the first term, we use Jensen's inequality to take the squared norm inside the sum, then use smoothness and take an expectation over the sampling of devices to get
    \[
        \expect_t \normsq*{\frac{1}{m}\sum_{i\in S\pow{t}}\left( \grad_u F_i\left(\tu_{i,k}\pow{t}, \tv_i\pow{t+1}\right) - \grad_u F_i\left(u\pow{t}, \tv_i\pow{t+1}  \right)\right)}
        \le \frac{L_u^2}{n} \sum_{i=1}^n \expect_t \normsq*{\tu_{i,k}\pow{t} - u\pow{t}} \,.
    \]
    For the second term, we use the fact that $S\pow{t}$ was sampled without replacement (cf. Lemma~\ref{lemma:techn:sampling-wo-replacement}) and invoke the gradient diversity assumption (Assumption~\ref{asmp:grad-diversity}) to get, 
    \begin{align*}
        \Bigg\|\frac{1}{m}\sum_{i\in S\pow{t}} \grad_u F_i&\left(u\pow{t}, \tv_i\pow{t+1}\right) - \grad_u F\left(u\pow{t}, \tV\pow{t+1}\right) \Bigg\|^2
        \\
        &\le \left(\frac{n-m}{n-1}\right) \frac{1}{mn} \sum_{i=1}^n \normsq*{\grad_u F_i\left(u\pow{t}, \tv_i\pow{t+1}\right) - \grad_u F\left(u, \tV\pow{t+1}\right)} \\
        &\le \frac{2}{m}\left(1 - \frac{m}{n}\right) \left(\delta^2 + \rho^2 \expect_t \normsq*{\grad_u F\left(u\pow{t}, \tV\pow{t+1}\right)}\right) \,.
    \end{align*}
    To complete the proof, we plug these terms back into the definition of $\Tcal_k'$ and $\expect_t[\Tcal_{2, u}]$ to complete the proof.
\end{proof}

\smallskip
\begin{claim}[Bounding $\Tcal_{3,u}$] \label{claim:pflam:w-step-3}
    We have, 
    \[
    \expect_t\left[ \Tcal_{3, u} \right]
    \le 8 \gamma_v^2 \tau_v^2 L_v \chi^2 \left(1- \frac{m}{n}\right) \Delta_v\pow{t} 
    + 
    4\chi^2 \gamma_v^2\tau_v^2 L_v \sigma_v^2 \left( 1- \frac{m}{n}\right)
    \,.
    \]
\end{claim}
\begin{proof}
    Since $v_i\pow{t+1} = \tv_i\pow{t+1}$ for $i \in S\pow{t}$, we have that 
    \[
        \Tcal_{3, u} = \frac{\chi^2 L_v}{2n} \sum_{i \notin S\pow{t}} \normsq*{\tv_i\pow{t+1} - v_i\pow{t}} \,.
    \]
    Since $\normsq*{\tv_i\pow{t+1} - v_i\pow{t}}$ is independent of $S\pow{t}$, we can take an expectation to get 
    \begin{align*}
        \expect_t[\Tcal_{3, u}] 
        &= \frac{\chi^2 L_v}{2n} \sum_{i=1}^n \prob(i \notin S\pow{t}) \, \expect_t \normsq*{\tv_i\pow{t+1} - v_i\pow{t}} \\
        &= \frac{\chi^2 L_v}{2n}\left(1- \frac{m}{n}\right) \sum_{i=1}^n   \expect_t \normsq*{\tv_i\pow{t+1} - v_i\pow{t}} \,.
    \end{align*}
    Plugging in Lemma~\ref{lem:techn:client-drift-2} completes the proof.
\end{proof}

The analysis of the $v$-step is given in the next result. 
\begin{claim} \label{claim:pflam:v-step}
    Consider the setting of Theorem~\ref{thm:pfl-am}
    and assume that $\gamma_v\tau_v L_v \le 1/8$. We have, 
    \[
    \expect_t\left[ F\left(u\pow{t}, V\pow{t+1}\right) - F\left(u\pow{t}, V\pow{t}\right)  \right]
    \le -\frac{\gamma_v \tau_v m \Delta_v\pow{t}}{8n} + \frac{\gamma_v^2\tau_v^2L_v \sigma_v^2 m}{2n}
    + \frac{4 \gamma_v^3 L_v^2 \tau_v^2(\tau_v - 1)\sigma_v^2 m}{n} \,.
    \]
\end{claim}
\begin{proof}
    From smoothness, we get, 
    \[
        F_i\left(u\pow{t}, \tv_i\pow{t+1}\right) - F_i\left(u\pow{t}, v_i\pow{t}\right)
        \le \underbrace{
        \inp*{\grad_v F_i\left(u\pow{t}, v_i\pow{t}\right)}{\tv_i\pow{t+1}-v_i\pow{t}}
        }_{\Tcal_{1, v}}
        + \underbrace{
        \frac{L_v}{2}\normsq*{\tv_i\pow{t+1} - v_i\pow{t}}
        }_{\Tcal_{2, v}} \,.
    \]
    We bound the first term as
    \begin{align*}
        \expect_t[\Tcal_{1, v}]
        &= -\gamma_v \expect_t \inp*{\grad_v F_i\left(u\pow{t}, v_i\pow{t}\right)}{\sum_{k=0}^{\tau_v-1} \grad_v F_i\left(u\pow{t}, \tv_{i,k}\pow{t}\right)} \\
        &\begin{aligned}
        =& -\gamma_v\tau_v \normsq*{\grad_v F_i\left(u\pow{t}, v_i\pow{t}\right)}
        \\ &
        -\gamma_v \sum_{k=0}^{\tau_v-1} \expect_t \inp*{\grad_v F_i\left(u\pow{t}, v_i\pow{t}\right)}{\grad_v F_i\left(u\pow{t}, \tv_{i,k}\pow{t}\right) - \grad_v F_i\left(u\pow{t}, v_i\pow{t}\right)}
        \end{aligned}\\
        &\le -\frac{\gamma_v\tau_v}{2}\normsq*{\grad_v F_i\left(u\pow{t}, v_i\pow{t}\right)}
        + \frac{\gamma_v}{2} \sum_{k=0}^{\tau_v-1} \expect_t\normsq*{\grad_v F_i\left(u\pow{t}, \tv_{i,k}\pow{t}\right) - \grad_v F_i\left(u\pow{t}, v_i\pow{t}\right)} \\
        &\le -\frac{\gamma_v\tau_v}{2}\normsq*{\grad_v F_i\left(u\pow{t}, v_i\pow{t}\right)}
        + \frac{\gamma_v L_v^2}{2} \sum_{k=0}^{\tau_v-1}\normsq*{\tv_{i,k}\pow{t} - v_i\pow{t}} \,.
    \end{align*}
    Next, we observe that 
    \[
        \expect_z\normsq{G_{i, v}(u, v_i, z)}
        = \normsq*{\grad_v F_i(u, v_i)} + \expect_z\normsq{G_{i, v}(u, v_i, z) - \grad_v F_i(u, v_i)}
        \le \normsq*{\grad_v F_i(u, v_i)} + \sigma_v^2 \,.
    \]
    We invoke this inequality to handle the second term as
    \begin{align*}
        \expect_t[\Tcal_{2, v}]
        &\le \frac{\gamma_v^2 L_v\tau_v}{2} \sum_{k=0}^{\tau_v-1} \expect_t\normsq*{G_{i, v}\left(u\pow{t}, \tv_{i,k}\pow{t}, z_{i,k}\pow{t}\right)} \\
        &\le \frac{\gamma_v^2 L_v\tau_v^2 \sigma_v^2}{2} + \frac{\gamma_v^2 L_v \tau_v}{2}\sum_{k=0}^{\tau_v-1} \expect_t\normsq*{\grad_v F_i\left(u\pow{t}, \tv_{i,k}\pow{t}\right)} \\
        &
        \begin{aligned}
        \le \frac{\gamma_v^2 L_v\tau_v^2 \sigma_v^2}{2}
        &+ \gamma_v^2 L_v \tau_v^2 \normsq*{\grad_v F_i\left(u\pow{t}, v_i\pow{t} \right)} \\
        &+ \gamma_v^2 L_v \tau_v\sum_{k=0}^{\tau_v-1} \expect_t\normsq*{\grad_v F_i\left(u\pow{t}, \tv_{i,k}\pow{t}\right) - \grad_v F_i\left(u\pow{t}, v_i\pow{t} \right)} 
        \end{aligned}
        \\
        &\le \frac{\gamma_v^2 L_v\tau_v^2 \sigma_v^2}{2}
        + \gamma_v^2 L_v \tau_v^2 \normsq*{\grad_v F_i\left(u\pow{t}, v_i\pow{t} \right)}
        + \gamma_v^2 L_v^3 \tau_v\sum_{k=0}^{\tau_v-1} \expect_t\normsq*{\tv_{i,k}\pow{t} -  v_i\pow{t}} \,.
    \end{align*}
    Plugging these bounds for $\Tcal_{1, v}$ and $\Tcal_{2, v}$ into the initial smoothness bound
    and using $\gamma_v L_v \tau_v \le 1/4$ gives
    \begin{align*}
        \expect_t\Big[ 
            F_i\left(u\pow{t}, \tv_i\pow{t+1}  \right)
            &- F_i\left(u\pow{t}, v_i\pow{t} \right)
        \Big]
        \le \\ &-\frac{\gamma_v\tau_v}{4}\normsq*{\grad_v F_i\left(u\pow{t}, v_i\pow{t}\right)}
        + \gamma_v L_v^2 \sum_{k=0}^{\tau_v-1}\normsq*{\tv_{i,k}\pow{t} - v_i\pow{t}} + \frac{\gamma_v^2 L_v\tau_v^2\sigma_v^2}{2} \,.
    \end{align*}
    We invoke Lemma~\ref{lem:techn:client-drift} to bound the $\sum_k \expect_t\normsq{\tv_{i,k}\pow{t} - v_i\pow{t}}$ term, which is also known as client drift. We simplify some coefficients using $8\gamma_v \tau_v L_v \le 1$ to get 
    \begin{align*}
        \expect_t\Big[ 
            F_i\left(u\pow{t}, \tv_i\pow{t+1}  \right)
            &- F_i\left(u\pow{t}, v_i\pow{t} \right)
        \Big]
        \le \\&-\frac{\gamma_v\tau_v}{8}\normsq*{\grad_v F_i\left(u\pow{t}, v_i\pow{t}\right)} + \frac{\gamma_v^2 L_v\tau_v^2\sigma_v^2}{2} + 4 \gamma_v^3 L_v \tau_v^2(\tau_v-1)\sigma_v^2 \,.
    \end{align*}
    It remains to invoke that $S\pow{t}$ is a uniformly random sample of $m$ devices from $\{1, \cdots, n\}$ and that $\tv_i\pow{t+1}$ is independent of $S\pow{t}$. 
    To this end, note that
    \begin{align*}
    \expect_t\left[F\left(u\pow{t}, V\pow{t+1}\right) - F\left(u\pow{t}, V\pow{t}\right)\right]
    &= \frac{m}{n} \, \expect_t \left[
    \frac{1}{m} \sum_{i \in S\pow{t}} F_i\left(u\pow{t}, \tv_i\pow{t+1}\right) - F_i\left(u\pow{t}, v_i\pow{t}\right) 
    \right] 
    \\&\le \frac{m}{n^2}\sum_{i=1}^n \expect_t \left[
     F_i\left(u\pow{t}, \tv_i\pow{t+1}\right) - F_i\left(u\pow{t}, v_i\pow{t}\right) 
    \right]  \,.
    \end{align*}
    Plugging in the previous bound completes the proof.
\end{proof}

\begin{remark} \label{remark:pfl-am:grad-diversity}
    We only invoked the partial gradient diversity assumption (Assumption~\ref{assmp:grad-diversity}) at (virtual) iterates $(u\pow{t}, \tV\pow{t+1})$; therefore, it suffices if the assumption only holds at iterates $(u\pow{t}, \tV\pow{t+1})$ generated by FedAlt, rather than at all $(u, V)$.
\end{remark}
 
\subsection{Convergence Analysis of \pflsu} \label{sec:a:proof:su}
\begin{algorithm}[t!]
	\caption{\pflsu: Simultaneous update of shared and personal parameters}
	\label{algo:pfl-su}
\begin{algorithmic}[1]
		\STATE \textbf{Input:} Initial iterates $u\pow{0}, V\pow{0}$,
		    Number of communication rounds $T$, 
		    Number of devices per round $m$, 
		    Number of local updates $\tau$,
		    Local step sizes $\gamma_u, \gamma_v$.
	    \FOR{$t=0, 1, \cdots, T-1$}
	        \STATE Sample $m$ devices from $[n]$ without replacement in $S\pow{t}$\label{line:pfl-su:sample}
	         \FOR{each selected device $i \in S\pow{t}$ in parallel}
	        \label{line:pfl-su:local}
	            \STATE Initialize $v_{i, 0}\pow{t} = v_i\pow{t}$
	            and $u_{i, 0}\pow{t} = u\pow{t}$
	            \FOR{$k=0, \cdots, \tau-1$}
	            \STATE \texttt{// Update all parameters jointly}
	                \STATE Sample data $z_{i,k}\pow{t} \sim \cD_i$	                \label{line:pfl-su:w-v}
	    	      \STATE $v_{i, k+1}\pow{t} = v_{i,k}\pow{t} - \gamma_v G_{i,v}(u_{i,k}\pow{t}, v_{i,k}\pow{t}, z_{i,k}\pow{t})$
	    	      \STATE $u_{i, k+1}\pow{t} = u_{i,k}\pow{t} - \gamma_u G_{i,u}(u_{i,k}\pow{t}, v_{i,k}\pow{t}, z_{i,k}\pow{t})$
	    	    \ENDFOR
	    	    \STATE Update $v_i\pow{t+1} = v_{i, \tau}\pow{t}$
	    	    and $u_i\pow{t+1} = u_{i, \tau}\pow{t}$
	    	\ENDFOR
	    	\STATE Update $u\pow{t+1} = {\sum_{i \in S\pow{t}} \alpha_i u_{i}\pow{t+1}}/{\sum_{i \in S\pow{t}} \alpha_i}$ at the server with secure aggregation
	    	\label{line:pfl-su:aggregation}
	    \ENDFOR
	    \STATE \textbf{return} $u\pow{T}, v_1\pow{T}, \cdots, v_n\pow{T}$
\end{algorithmic}
\end{algorithm}

We give the full form of \pflsu in Algorithm~\ref{algo:pfl-su} for the general case of unequal $\alpha_i$'s but focus on $\alpha_i=1/n$ for the analysis.
Theorem~\ref{thm:pfl-su:main} of the main paper is a simplification of Corollary~\ref{cor:a:pfl-su} below, which in turn is proved based on Theorem~\ref{thm:pfl-su}.

Throughout this section, we use constants
\begin{align*}
    \sigma_{\mathrm{sim}, 1}^2 = (1+\chi^2)\left(\frac{\delta^2}{L_u}\left(1-\frac{m}{n}\right) + \frac{\sigma_u^2}{L_u} + \frac{\sigma_v^2m}{L_v n}\right)\,,
    \quad\text{and}\,, \quad
    \sigma_{\mathrm{sim}, 2}^2 = (1+\chi^2)\left(\frac{\delta^2}{L_u} + \frac{\sigma_u^2}{L_u} + \frac{\sigma_v^2}{L_v}\right) (1-\tau^{-1}) \,.
\end{align*}

\begin{theorem}[\textbf{Convergence of \hlsim{FedSim}}]
\label{thm:pfl-su}
Suppose Assumptions~\ref{asmp:smoothness}, \ref{asmp:stoc-grad-var} and~\ref{asmp:grad-diversity} hold and the learning rates in FedSim are chosen as $\gamma_u = \lr/(L_u\tau)$ and $\gamma_v=\lr/(L_v\tau)$ with
    \[
        \lr \le \min\left\{
        \frac{1}{12(1+\chi^2)(1+\rho^2)},~
        \sqrt{\frac{m/n}{196 (1-\tau^{-1})(1+\chi^2)(1+\rho^2)}}\right\} \,.
    \]
    Then, ignoring absolute constants, we have
    \begin{align*}
        &\frac{1}{T}\sum_{t=0}^{T-1} \left(
        \frac{1}{L_u} \expect\bigl[\Delta_u\pow{t}\bigr]
        + \frac{m}{n L_v} \expect\bigl[\Delta_v\pow{t} \bigr] \right)
        \le 
        \frac{\Delta F_0}{\lr T}
        + \lr \, \sigma_{\mathrm{sim}, 1}^2
        + 
        \lr^2 \, \sigma_{\mathrm{sim}, 2}^2 \,.
    \end{align*}
\end{theorem}
Before proving the theorem, we give the following corollary with optimized learning rates. 

\begin{corollary}[\textbf{Final Rate of FedSim}] \label{cor:a:pfl-su}
    Consider the setting of Theorem~\ref{thm:pfl-su} and let the total number of rounds $T$ be known in advance.
    Suppose we set the learning rates $\gamma_u = \lr/(\tau L_u)$ and $\gamma_v = \lr/(\tau L_v)$, where (ignoring absolute constants), 
    \begin{align*}
        \lr = 
        \left(\frac{\Delta F_0}{T \, \sigma_{\mathrm{sim}, 1}^2 }\right)^{1/2}
        \bigwedge 
        \left(
            \frac{\Delta F_0^2}{T^2 \, \sigma_{\mathrm{sim}, 2}^2}
        \right)^{1/3}
        \bigwedge \frac{1}{(1+\chi^2) (1+\rho^2)}
        \bigwedge \sqrt{\frac{m/n}{(1-\tau^{-1}) (1+\chi^2) (1+\rho^2)}} \,.
    \end{align*}
    We have, ignoring absolute constants, 
    \begin{align*}
        &\frac{1}{T}\sum_{t=0}^{T-1}
        \left(
        \frac{1}{L_u} \expect\normsq*{\grad_u F\left(u\pow{t}, V\pow{t}\right)}
        + 
        \frac{m}{L_v n^2} \sum_{i=1}^n \expect\normsq*{\grad_v F_i\left(u\pow{t}, v_i\pow{t}\right)}
        \right)
        \le  \\
        &\frac{\left(\Delta F_0 \, \sigma_{\mathrm{sim}, 1}^2 \right)^{1/2}}{\sqrt{T}} 
        +  \frac{\left(\Delta F_0^2 \, \sigma_{\mathrm{sim}, 2}^2 \right)^{1/3}}{T^{2/3}}  
        + 
        \frac{\Delta F_0 (1+\chi^2) (1+\rho^2)}{T}
        + \frac{\Delta F_0 \sqrt{\frac{n}{m} (1-\tau^{-1}) (1+\chi^2) (1+\rho^2)}}{T}
        \,.
    \end{align*}
\end{corollary}
\begin{proof}
    The proof follows from invoking Lemma~\ref{lem:sfl:best-params-2}
    on the bound of Theorem~\ref{thm:pfl-su}.
\end{proof}

\begin{remark}[\textbf{Asymptotic Rate}]
    The asymptotic $1/\sqrt{T}$ rate of Theorem~\ref{thm:pfl-su:main} is achieved when the $1/T$ term is dominated by the $1/\sqrt{T}$ term. This happens when (ignoring absolute constants)
    \[
        T \ge 
        \frac{\Delta F_0 (1+\chi^2)(1+\rho^2)}{\sigma_{\mathrm{sim}, 1}^2} \,
        \max\left\{
        (1-\tau^{-1}) \frac{n}{m}, \,
        \, (1+\chi^2)(1+\rho^2) \right\} \,.
    \]
    Note that $T \ge \Omega(n/m)$ is necessary for each device to be seen at least once on average, or the personal parameters of some devices will never be updated.
\end{remark}

We now prove Theorem~\ref{thm:pfl-su}.
\begin{proof}[Proof of Theorem~\ref{thm:pfl-su}]
    The proof mainly applies the smoothness upper bound to write out a descent condition with suitably small noise terms. We start with some notation.

    \myparagraph{Notation}
    Let $\Fcal\pow{t}$ denote the $\sigma$-algebra generated by $\left(u\pow{t}, V\pow{t}\right)$ and denote $\expect_t[\cdot] = \expect[\cdot| \Fcal\pow{t}]$.
    For all devices, including those not selected in each round, we define virtual sequences $\tu_{i,k}\pow{t}, \tv_{i,k}\pow{t}$  as the SGD updates in Algorithm~\ref{algo:pfl-su}  for all devices regardless of whether they are selected.  For the selected devices $k\in S\pow{t}$, we have $\left(u_{i,k}\pow{t}, v_{i,k}\pow{t}\right) = \left(\tu_{i,k}\pow{t}, \tv_{i,k}\pow{t}\right)$. Note now that the random variables $\tu_{i,k}\pow{t}, \tv_{i,k}\pow{t}$ are independent of the device selection $S\pow{t}$.
The updates for the devices $i \in S\pow{t}$ are given by 
    \begin{align*}
        v_i\pow{t+1} &= v_i\pow{t} - \gamma_v \sum_{k=0}^{\tau-1} G_{i,v}\left(\tu_{i,k}\pow{t}, \tv_{i,k}\pow{t}, z_{i,k}\pow{t} \right) \,,
    \end{align*}
    and the server update is given by
    \begin{align}\label{eqn:pfl-su:proof:u-update}
        u\pow{t+1} &= u\pow{t} - \frac{\gamma_u}{m} \sum_{i \in S\pow{t}}\sum_{k=0}^{\tau-1} G_{i,u}\left(\tu_{i,k}\pow{t}, \tv_{i,k}\pow{t}, z_{i,k}\pow{t} \right)\,.
    \end{align}
    
    \myparagraph{Proof Outline}
    We use the smoothness of $F_i$, more precisely Lemma~\ref{lemma:techn:block-smoothness}, to obtain
    \begin{align} \label{eq:pfl-su:proof:1}
    \begin{aligned} 
         & F\bigl(u\pow{t+1}, V\pow{t+1}\bigr)
        - F\bigl(u\pow{t}, V\pow{t}\bigr) \\
    \le &~\underbrace{\inp*{\grad_u F(u\pow{t}, V\pow{t})}{u\pow{t+1}-u\pow{t}}}_{\Tcal_{1,u}}+ 
        \underbrace{\frac{1}{n} \sum_{i=1}^n \inp*{\grad_v F_i(u\pow{t}, v_i\pow{t})}{v_i\pow{t+1}-v_i\pow{t}}}_{\Tcal_{1, v}} \\
        & +  
        \underbrace{\frac{L_u(1+\chi^2)}{2}\normsq*{u\pow{t+1}-u\pow{t}}}_{\Tcal_{2, u}}  
        + \underbrace{\frac{1}{n}\sum_{i=1}^n\frac{L_v(1+\chi^2)}{2}\normsq*{v_i\pow{t+1}-v_i\pow{t}}}_{\Tcal_{2, v}} \, .
    \end{aligned}
    \end{align}
Our goal will be to bound each of these terms to get a descent condition from each step of the form
    \begin{align*}
        \expect_t&\left[F\bigl(u\pow{t+1}, V\pow{t+1}\bigr)
        - F\bigl(u\pow{t}, V\pow{t}\bigr)\right]
         \\
        &\leq  - \frac{\gamma_u\tau}{8} \normsq*{\grad_u F\bigl(u\pow{t}, V\pow{t}\bigr)}
        - \frac{\gamma_v \tau m}{8n^2} \sum_{i=1}^n \normsq*{\grad_v F_i\bigl(u\pow{t}, v_i\pow{t}\bigr)} + O(\gamma_u^2 + \gamma_v^2) \,,
    \end{align*}
    where the $O(\gamma_u^2 + \gamma_v^2)$ terms are controlled using the bounded variance and gradient diversity assumptions. Telescoping this descent condition gives the final bound.
    
    \myparagraph{Main Proof}
    Towards this end, we prove non-asymptotic bounds on each of the terms
    $\Tcal_{1,v}$, $\Tcal_{1,u}$, $\Tcal_{2,v}$ and $\Tcal_{2,u}$,
    in Claims \ref{claim:pfl-su:t1v} to \ref{claim:pfl-su:t2w} respectively.
    We then invoke them to get the bound
    \begin{align} \label{eq:pfl-su:proof:2}
    \begin{aligned}
        \expect_t&\left[F\bigl(u\pow{t+1}, V\pow{t+1}\bigr)
        - F\bigl(u\pow{t}, V\pow{t}\bigr)\right]
        \le 
        -\frac{\gamma_u\tau}{4}\Delta_u\pow{t}
        - \frac{\gamma_v\tau m}{4n}\Delta_v\pow{t} \\
        &+ \frac{L_u(1+\chi^2)\gamma_u^2\tau^2}{2}\left( \sigma_u^2 + \frac{12\delta^2}{m}(1-m/n)\right) 
        + \frac{L_v(1+\chi^2)\gamma_v^2 \tau^2 \sigma_v^2 m}{2n} 
        \\
        &+ 
        \frac{2}{n}\sum_{i=1}^n \sum_{k=0}^{\tau-1} \expect_t\normsq*{u_{i,k}\pow{t} - u\pow{t}}\left( L_u^2\gamma_u + \frac{m}{n} \chi^2 L_u L_v \gamma_v \right)
        \\
        &+ \frac{2}{n}\sum_{i=1}^n \sum_{k=0}^{\tau-1} \expect_t\normsq*{v_{i,k}\pow{t}-v\pow{t}} \left( \frac{m}{n} L_v^2 \gamma_v + \chi^2 L_u L_v \gamma_u \right) \,.
    \end{aligned}
    \end{align}
    Note that we simplified some constants appearing on the gradient norm terms using 
    \[
        \gamma_u \le \big(12 L_u(1+\chi^2)(1+\rho^2)\tau\big)^{-1}
        \quad\text{and}\quad
        \gamma_v \le \big(6 L_v(1+\chi^2)\tau\big)^{-1} .
    \]
    Our next step is to 
    bound the last two lines of \eqref{eq:pfl-su:proof:2} with Lemma~\ref{lemma:pfl-su:client-drift} and invoke the gradient diversity assumption (Assumption~\ref{asmp:grad-diversity}) as 
    \[
    \frac{1}{n} \sum_{i=1}^n \normsq*{\grad_u F_i\bigl(u\pow{t}, v_i\pow{t} \bigr)}
    \le \delta^2 + (1+\rho^2) \normsq*{\grad_u F\bigl(u\pow{t}, V\pow{t}\bigr)} \,.
    \]
    This gives, after plugging in the learning rates and further simplifying the constants,
    \begin{align*}
        &\expect_t\left[F\bigl(u\pow{t+1}, V\pow{t+1}\bigr)
        - F\bigl(u\pow{t}, V\pow{t}\bigr)\right] \\
        \le &  
        -\frac{c\Delta_u\pow{t}}{8 L_u}
        - \frac{cm \Delta_v\pow{t}}{8 L_v n} 
        + c^2(1+\chi^2) \left(\frac{\sigma_u^2}{2L_u}
        + \frac{m\sigma_v^2}{n L_v}
        + \frac{6\delta^2}{L_u m}\left(1-\frac{m}{n}\right)\right) \\
        &
        + c^3(1+\chi^2)(1-\tau^{-1})\left(\frac{24\delta^2}{L_u} + \frac{4\sigma_u^2}{L_u} + \frac{4 \sigma_v^2}{L_u} \right) \,.
    \end{align*}
    Taking full expectation, telescoping the series over $t=0,\cdots,T-1$ and rearranging the resulting terms give the desired bound in Theorem~\ref{thm:pfl-su}.
\end{proof}

\smallskip
\begin{claim}[Bounding $\Tcal_{1, v}$]\label{claim:pfl-su:t1v}
    Let $\Tcal_{1,v}$ be defined as in \eqref{eq:pfl-su:proof:1}. We have,
        \begin{align*}
        \expect_t[\Tcal_{1, v}]
        \le & \, -\frac{\gamma_v\tau m}{2n^2}\sum_{i=1}^n  \normsq*{\grad_v F_i\left(u\pow{t}, v_i\pow{t}\right)} \\
        &+ \frac{\gamma_v m}{n} \sum_{i=1}^n \sum_{k=0}^{\tau-1} \expect_t\left[ \chi^2 L_u L_v \normsq*{\tu_{i,k}\pow{t} - u\pow{t}}
        + L_v^2 \normsq*{\tv_{i,k}\pow{t} - v_i\pow{t}}
        \right] \,.
    \end{align*}
\end{claim}
\begin{proof}
    Define $\Tcal_{1, v, i}$ to be contribution of the $i$th term to $\Tcal_{1, v}$. For $i \notin S_t$, we have that $\Tcal_{1, v, i} = 0$, since $v_i\pow{t+1} = v_i\pow{t}$. On the other hand, for $i \in S\pow{t}$, we use the unbiasedness of the gradient estimator $G_{i,v}$ and the independence of $z_{i,k}\pow{t}$ from $u_{i,k}\pow{t}, v_{i,k}\pow{t}$ to get
    \begin{align}
    \nonumber
        \expect_t&\left[\Tcal_{1, v, i}\right] = 
        -\gamma_v \sum_{k=0}^{\tau-1} \expect_t\inp*{\grad_v F_i\left(u\pow{t}, v_i\pow{t}\right)}{ 
        \grad_v F_i\left(u_{i,k}\pow{t}, v_{i,k}\pow{t}\right)} \\ \nonumber
        &= 
        -\gamma_v \sum_{k=0}^{\tau-1} \expect_t\inp*{\grad_v F_i\left(u\pow{t}, v_i\pow{t}\right)}{ 
        \grad_v F_i\left(\tu_{i,k}\pow{t}, \tv_{i,k}\pow{t}\right)} \\ \nonumber
        &
        \begin{aligned}
        = &
        -\gamma_v\tau \normsq*{\grad_v F_i\left(u\pow{t}, v_i\pow{t}\right)}  \\
        &-\gamma_v \sum_{k=0}^{\tau-1} \expect_t\inp*{\grad_v F_i\left(u\pow{t}, v_i\pow{t}\right)}{ 
        \grad_v F_i\left(\tu_{i,k}\pow{t}, \tv_{i,k}\pow{t}\right) - \grad_v F_i\left(u\pow{t}, v_i\pow{t}\right)} 
        \end{aligned}\\
        &\le 
        -\frac{\gamma_v \tau}{2} \normsq*{\grad_v F_i\left(u\pow{t}, v_i\pow{t}\right)}
        + \frac{\gamma_v}{2} \sum_{k=0}^{\tau-1} \expect_t\normsq*{\grad_v F_i\left(\tu_{i,k}\pow{t}, \tv_{i,k}\pow{t}\right) - \grad_v F_i\left(u\pow{t}, v_i\pow{t}\right)} \,.
    \label{eq:pfl-su:t1v:1}
    \end{align}
    For the second term, we add and subtract $\grad_v F_i\left(u\pow{t}, \tv_{i,k}\pow{t}\right)$ and use smoothness to get
    \begin{align} \label{eq:pfl-su:t1v:2}
        \normsq*{\grad_v F_i\left(\tu_{i,k}\pow{t}, \tv_{i,k}\pow{t}\right) - \grad_v F_i\left(u\pow{t}, v_i\pow{t}\right)}
        &\le 
        2\chi^2 L_u L_v \normsq*{\tu_{i,k}\pow{t} - u\pow{t}}
        +
        2L_v^2 \normsq*{\tv_{i,k}\pow{t} - v_i\pow{t}} \,.
    \end{align}
    Since the right hand side of this bound is independent of $S_t$, we get, 
    \begin{align*}
        &\expect_t[\Tcal_{1, v}]
        = \frac{m}{n} \expect_t\left[ \frac{1}{m} \sum_{i \in S\pow{t}} \Tcal_{1, v, i} \right] 
        = \frac{m}{n^2} \sum_{i=1}^n \expect_t[\Tcal_{1, v, i}] \,,
    \end{align*}
    and plugging in \eqref{eq:pfl-su:t1v:1} and \eqref{eq:pfl-su:t1v:2} completes the proof.
\end{proof}

\smallskip
\begin{claim}[Bounding $\Tcal_{1,u}$] \label{claim:pfl-su:t1w}
    Consider $\Tcal_{1,u}$ defined in  \eqref{eq:pfl-su:proof:1}. We have the bound,
    \begin{align*}
        \expect_t[\Tcal_{1,u}]
        \le &\, -\frac{\gamma_u \tau}{2} \normsq*{\grad_u F\left(u\pow{t}, V\pow{t} \right)} \\
        &\,+ \frac{\gamma_u}{n} \sum_{i=1}^n \sum_{k=0}^{\tau-1} 
        \expect_t\left[
        L_u^2 \normsq*{\tu_{i,k}\pow{t} - u\pow{t}}
        + \chi^2 L_u L_v \normsq*{\tv_{i,k}\pow{t} - v_i\pow{t}} 
        \right] \,.
    \end{align*}
\end{claim}
\begin{proof}
Due to the independence of $S\pow{t}$ from $\tu_{i,k}\pow{t}, \tv_{i,k}\pow{t}$, we have, 
    \begin{align*}
        \expect_t\left[u\pow{t+1} - u\pow{t} \right]
        &= 
        -\gamma_u \expect_t\left[\frac{1}{m} \sum_{i \in S\pow{t}} \sum_{k=0}^{\tau-1} \grad_u F_i\left(u_{i,k}\pow{t}, v_{i,k}\pow{t}\right) \right] \\
        &= 
        -\gamma_u \expect_t\left[\frac{1}{m} \sum_{i \in S\pow{t}} \sum_{k=0}^{\tau-1} \grad_u F_i\left(\tu_{i,k}\pow{t}, \tv_{i,k}\pow{t}\right) \right] \\
        &=
        - \frac{\gamma_u}{n} \sum_{i=1}^n \sum_{k=0}^{\tau-1} \expect_t\left[\grad_u F_i\left(\tu_{i,k}\pow{t}, \tv_{i,k}\pow{t}\right)\right]  \,,
    \end{align*}
    where the last equality took an expectation over $S^{(t)}$, which is independent of $\tu_{i,k}\pow{t}, \tv_{i,k}\pow{t}$.
    Now, using the same sequence of arguments as Claim~\ref{claim:pfl-su:t1v}, 
    we have, 
    \begin{align*}
    &\expect_t\llangle\nabla_u F\bigl(\ut,\Vt\bigr),\utp-\ut\rrangle \\
    &= -\gamma_u \sum_{k=0}^{\tau-1} 
        \expect_t \inp*{\grad_u F\left(\ut, \Vt \right)}{\frac{1}{n}\sum_{i=1}^n \grad_u F_i\left(\tu_{i,k}\pow{t}, \tv_{i,k}\pow{t} \right)} \\
    &\le -\frac{\gamma_u \tau}{2} \normsq*{\grad_u F\left(\ut, \Vt \right)}
        + \frac{\gamma_u}{2} \sum_{k=0}^{\tau-1} \expect_t \normsq*{\frac{1}{n} \sum_{i=1}^n \grad_u F_i\left(\tu_{i, k}\pow{t}, \tv_{i, k}\pow{t}\right) - \grad_u F\left(\ut, \Vt \right)} \\
    &\stackrel{(*)}{\le} -\frac{\gamma_u \tau}{2} \normsq*{\grad_u F\left(\ut, \Vt \right)}
        + \frac{\gamma_u}{2n}\sum_{i=1}^n \sum_{k=0}^{\tau-1} \expect_t \normsq*{ \grad_u F_i\left(\tu_{i, k}\pow{t}, \tv_{i, k}\pow{t}\right) - \grad_u F_i\left(\ut, \vt_i \right)} \\
    &\le -\frac{\gamma_u \tau}{2} \normsq*{\grad_u F\left(\ut, \Vt \right)}
        + \frac{\gamma_u}{n}\sum_{i=1}^n \sum_{k=0}^{\tau-1} \expect_t\left[
        L_u^2 \normsq*{\tu_{i, k}\pow{t} - \ut} + L_{uv}^2 \normsq*{\tv_{i, k}\pow{t} - \vt_i}  \right]\,,
    \end{align*}
    where the inequality $(*)$ follows from Jensen's inequality as
    \[
    \normsq*{\frac{1}{n} \sum_{i=1}^n \grad_u F_i\left(\tu_{i,k}\pow{t}, \tv_{i,k}\pow{t}\right) - \grad_u F\left(u\pow{t}, V\pow{t}\right)}
    \le \frac{1}{n} \sum_{i=1}^n \normsq*{
    \grad_u F_i\left(\tu_{i,k}\pow{t}, \tv_{i,k}\pow{t}\right) - \grad_u F_i\left(u_{i,k}\pow{t}, v\pow{t}\right)}\,.
    \]
\end{proof}

\smallskip
\begin{claim}[Bounding $\Tcal_{2, v}$] \label{claim:pfl-su:t2v}
    Consider $\Tcal_{2,v}$ as defined in \eqref{eq:pfl-su:proof:1}. 
    We have the bound, 
    \begin{align*}
        \expect_t[\Tcal_{2, v}] 
           \le & \,\frac{3L_v(1+\chi^2) \gamma_v^2 \tau^2 m}{2n^2} \sum_{i=1}^n \normsq*{\grad_v F_i\left(u\pow{t}, v_i\pow{t}\right)}
            + \frac{L_v(1+\chi^2) \gamma_v^2 \tau^2 m \sigma_v^2}{2n}  \\
            & + \frac{3L_v(1+\chi^2)\gamma_v^2\tau m}{2n^2} \sum_{i=1}^n \sum_{k=0}^{\tau-1} \expect_t\left[L_v^2 \normsq*{\tv_{i,k}\pow{t} - v_i\pow{t}}
        + \chi^2 L_u L_v \normsq*{\tu_{i,k}\pow{t} - u\pow{t}} \right] \,.
    \end{align*}
\end{claim}
\begin{proof}
    We start with
    \begin{align*}
        \expect_t&\normsq*{\tv_{k, \tau}\pow{t}-v\pow{t}}
        = \gamma_v^2\expect_t\normsq*{\sum_{k=0}^{\tau-1} G_{i,v}\left(\tu_{i,k}\pow{t}, \tv_{i,k}\pow{t}, z_{i,k}\pow{t} \right)} \\
        &\le \gamma_v^2\tau \sum_{k=0}^{\tau-1}
            \expect_t\normsq*{G_{i,v}\left(\tu_{i,k}\pow{t}, \tv_{i,k}\pow{t}, z_{i,k}\pow{t} \right)} \\
        &\le \gamma_v^2\tau^2 \sigma_v^2 + \gamma_v^2\tau \sum_{k=0}^{\tau-1} \expect_t\normsq*{\grad_v F_i\left(\tu_{i,k}\pow{t}, \tv_{i,k}\pow{t} \right)} \\
        &
        \begin{aligned}\le \gamma_v^2\tau^2\sigma_v^2 &+ 3 \gamma_v^2 \tau^2 \normsq*{\grad_v F_i\left(u\pow{t}, v_i\pow{t}\right)} \\
        &+ 3 \gamma_v^2 \tau \sum_{k=0}^{\tau-1} 
        \expect_t\left[L_v^2 \normsq*{\tv_{i,k}\pow{t} - v_i\pow{t}}
        + \chi^2 L_u L_v \normsq*{\tu_{i,k}\pow{t} - u\pow{t}} \right] \,.
        \end{aligned}
    \end{align*}
    Using (a) $v_i\pow{t+1} = \tv_{i, \tau}\pow{t}$ 
     for $i \in S\pow{t}$, and, 
    (b) $S\pow{t}$ is independent from $\tu_{i,k}\pow{t}, \tv_{i,k}\pow{t}$, we get, 
    \begin{align*}
        \expect_t[\Tcal_{2, v}] 
        &= \frac{L_v(1+\chi^2)m}{2n} \, \expect_t\left[ \frac{1}{m} \sum_{i \in S\pow{t}} \normsq*{\tv_{i, \tau}\pow{t} - v_i\pow{t}} \right] \\
        &\le \frac{L_v(1+\chi^2)m}{2n^2}  \sum_{i=1}^n \expect_t\normsq*{\tv_{i, \tau}\pow{t} - v_i\pow{t}}  
    \end{align*}
    Plugging in the bound $\expect_t\normsq*{\tv_{i, \tau}\pow{t}-v\pow{t}}$ completes the proof.
\end{proof}

\smallskip
\begin{claim}[Bounding $\Tcal_{2, u}$] \label{claim:pfl-su:t2w}
    Consider $\Tcal_{2, u}$ as defined in \eqref{eq:pfl-su:proof:1}. We have, 
    \begin{align*}
    \expect_t[\Tcal_{2,u}]
    \le &\,  \frac{L_u(1+\chi^2) \gamma_u^2 \tau^2}{2m}\left( 
    \sigma_u^2 + 12\delta^2\left(1-\frac{m}{n} \right)\right) \\
    &+ 3 L_u(1+\chi^2) \gamma_u^2 \tau^2 (1+\rho^2) \normsq*{\grad_u F_i\left(u\pow{t}, V\pow{t} \right)}
    \\
    &+ \frac{3L_u(1+\chi^2)\gamma_u^2 \tau}{2n} \sum_{i=1}^n\sum_{k=0}^{\tau-1} \expect_t 
    \left[
    L_u^2 \normsq*{\tu_{i,k}\pow{t} - u\pow{t}}
    + \chi^2L_uL_v \normsq*{\tv_{i,k}\pow{t} - v_i\pow{t}}
    \right] \,.
    \end{align*}
\end{claim}
\begin{proof}
    We proceed with the first two inequalities as in the proof of
    Claim~\ref{claim:pfl-su:t2v}
    to get 
    \begin{align*}
        \expect_t\normsq*{u\pow{t+1} - u\pow{t}}
        &\le \frac{\gamma_u^2 \tau^2\sigma_u^2}{m}
            + \gamma_u^2 \tau \sum_{k=0}^{\tau-1}
            \underbrace{\expect_t\normsq*{\frac{1}{m}\sum_{i \in S\pow{t}} \grad_u F_i\left(\tu_{i,k}\pow{t}, \tv_{i,k}\pow{t} \right)}}_{=: \Tcal_{3, j}} \,.
    \end{align*}
    For $\Tcal_{3,j}$, (a) we add and subtract 
    $\grad_u F(u\pow{t}, V\pow{t})$ and $\grad_u F_i(u\pow{t}, \tv_{i,k}\pow{t})$, 
    (b) invoke the squared triangle inequality, and,
    (c) use smoothness to get
    \begin{align*}
       \Tcal_{3,j}  =  &\,
       6 \, \expect_t\normsq*{
       \frac{1}{m} \sum_{i \in S\pow{t}} \grad_u F_i\left(u\pow{t}, v_i\pow{t}\right) - 
       \grad_u F\left(u\pow{t}, V\pow{t}\right)
       } + 6 \normsq*{\grad_u F\left(u\pow{t}, V\pow{t} \right)} \\
       & \, + 3 \expect_t \left[ 
       \frac{1}{m}\sum_{i \in S\pow{t}}
       \left( L_u^2 \normsq*{\tu_{i,k}\pow{t} - u\pow{t}} + \chi^2 L_u L_v \normsq*{\tv_{i,k}\pow{t} - v_i\pow{t}}  \right)
       \right]
    \end{align*}
    For the first term, 
    we use the fact that $S\pow{t}$ is obtained by sampling without replacement to apply Lemma~\ref{lemma:techn:sampling-wo-replacement} together with the gradient diversity assumption to get
    \begin{align*}
        \expect_t&\normsq*{
       \frac{1}{m} \sum_{i \in S\pow{t}} \grad_u F_i\left(u\pow{t}, v_i\pow{t}\right) - 
       \grad_u F\left(u\pow{t}, V\pow{t}\right)
       }  \\ 
       &\le \frac{1}{m}\left(\frac{n-m}{n-1}\right) \frac{1}{n}\sum_{i=1}^n 
        \normsq*{\grad_u F_i\left(u\pow{t}, v_i\pow{t}\right) - 
       \grad_u F\left(u\pow{t}, V\pow{t}\right)} \\
       &\le \frac{1}{m}\left(\frac{n-m}{n-1}\right)\left( \delta^2 + \rho^2 \normsq*{\grad_u F\left(u\pow{t}, V\pow{t}\right)} \right) \,.
    \end{align*}
    Therefore, 
    \begin{align*}
    \Tcal_{3,j}  =  &\,
       \frac{12 \delta^2}{m} \left(1 - \frac{m}{n}\right) + 6(1+\rho^2) \normsq*{\grad_u F\left(u\pow{t}, V\pow{t} \right)} \\
       & \, + 
       \frac{3}{n}\sum_{i=1}^{n}
       \expect_t \left[ L_u^2 \normsq*{\tu_{i,k}\pow{t} - u\pow{t}} + \chi^2 L_u L_v \normsq*{\tv_{i,k}\pow{t} - v_i\pow{t}} 
       \right] \,,
    \end{align*}
    where we also used the independence between $S\pow{t}$ and $(\tu_{i,k}\pow{t}, \tv_{i,k}\pow{t})$. Plugging this into the expression for $\expect_t\normsq{u\pow{t+1} - u\pow{t}}$ completes the proof.
\end{proof}

\begin{lemma} \label{lemma:pfl-su:client-drift}
    Let $F_i$ satisfy Assumptions~\ref{asmp:smoothness}-\ref{asmp:grad-diversity}, and consider the iterates
    \[
        u_{k+1} = u_k - \gamma_u G_{i,u}(u_k, v_k, z_k)\,, \quad\text{and}, \quad
        v_{k+1} = v_k - \gamma_v G_{i,v}(u_k, v_k, z_k) \,,
    \]
    for $k = 0, \cdots, \tau-1$,
    where $z_k \sim \Dcal_i$. Suppose the learning rates satisfy $\gamma_u = c_u / (\tau L_u)$ and $\gamma_v = c_v / (\tau L_v)$ with 
    $c_u, c_v \le 1/\sqrt{6} \max\{1, \chi^{-2}\}$.
    Further, define, 
    \[
        A = \gamma_u L_u^2 + f\chi^2 \gamma_v  L_u L_v\,,
        \quad\text{and},\quad
        B = f\gamma_v L_v^2 + \chi^2 \gamma_u  L_u L_v \,,
    \]
    where $f \in (0, 1]$ is given.
    Then, we have the bound, 
    \begin{align*}
        \sum_{k=0}^{\tau-1}
        \expect \big[ A\normsq{u_k - u_0} + & B\normsq{v_k - u_0}  \big]
        \le\, 4 \tau^2(\tau-1)\left( \gamma_u^2 \sigma_u^2 A + \gamma_v^2 \sigma_v^2 B \right)  \\
        &+ 12\tau^2(\tau-1)\left(
            \gamma_u^2 A \normsq{\grad_u F_i(u_0, v_0)}
            + \gamma_v^2 B \normsq{\grad_v F_i(u_0, v_0)}
        \right) \,.
    \end{align*}
\end{lemma}
\begin{proof}
    If $\tau=1$, there is nothing to prove, so we assume $\tau > 1$. Let $\Delta_k := A \normsq{u_k - u_0} + B\normsq{v_k - v_0}$ and denote by $\Fcal_k$ the sigma-algebra generated by $(w_k, v_k)$. Further, 
    let $\expect_k[\cdot] = \expect[\cdot | \Fcal_k]$.
    We use the inequality $2\alpha\beta \le \alpha^2/\delta^2 + \delta^2\beta^2$ for reals $\alpha, \beta, \delta$ to get, 
    \begin{align*}
    \expect_k\normsq{u_{k+1}-u_0}
    &\le \left(1 + \frac{1}{\tau-1}\right) \normsq{u_k - u_0} + \tau\gamma_u^2 \expect_k\normsq*{G_{i, u}(u_k, v_k, z_k)} \\
    &\le\left(1 + \frac{1}{\tau-1}\right) \normsq{u_k - u_0} + \tau\gamma_u^2 \sigma_u^2 + \tau\gamma_u^2 \normsq*{\grad_u F_i(u_k, v_k)}  \\
    &
    \begin{aligned}
    \le \left(1 + \frac{1}{\tau-1}\right) & \normsq{u_k - u_0} + \tau\gamma_u^2 \sigma_u^2 + 3\tau\gamma_u^2 \normsq*{\grad_u F_i(u_0, v_0)}  
    \\& + 3\tau\gamma_u^2 L_u^2 \normsq{u_k - u_0} + 3\tau\gamma_u^2 L_{uv} \normsq{v_k-v_0} \,,
    \end{aligned}
    \end{align*}
    where the last inequality followed from the squared triangle inequality (from adding and subtracting $\grad_u F_i(u_0, v_k)$ and $\grad_u F_i(u_0, v_0)$) followed by smoothness. 
    Together with the analogous inequality for the $v$-update, we get, 
    \begin{align*}
        \expect_k[\Delta_{k+1}]
        \le \left(1 + \frac{1}{\tau-1}\right)\Delta_k
            + A' \normsq{u_k - u_0} + B'\normsq{v_k -v_0} + C\,,
    \end{align*}
    where we have 
    \begin{align*}
        A' = 3\tau (\gamma_u^2 L_u^2 A + \gamma_v^2 \chi^2 L_u L_v B), \quad \text{and},\quad
        B' = 3\tau(\gamma_v^2 L_v^2 B + \gamma_u^2 \chi^2 L_u L_v A) \quad\text{and}, \\
        C' = \tau \gamma_u^2 \sigma_u^2 A + \tau \gamma_v^2 \sigma_v^2 B + 3 \tau\gamma_u^2 A\normsq{\grad_u F_i(u_0, v_0)} + 3 \tau\gamma_v^2 B \normsq{\grad_v F_i(u_0, v_0)} \,.
    \end{align*}
    Next, we apply Lemma~\ref{lemma:techn:pflsu:consts} to get that $A' \le A/\tau$ and $B' \le B/\tau$ under the assumed conditions on the learning rates; this allows us to write the right hand side completely in terms of $\Delta_k$ and unroll the recurrence. 
    The intuition behind Lemma~\ref{lemma:techn:pflsu:consts} is as follows. 
    Ignoring the dependence on $\tau, L_u, L_v, \chi$ for a moment, if $\gamma_u$ and $\gamma_v$ are both $O(\eta)$, then $A', B'$ are both $O(\eta^3)$, while $A$ and $B$ are $O(\eta)$. Thus, making $\eta$ small enough should suffice to get $A' \le O(A)$ and $B' \le O(B)$.
    
    Concretely, Lemma~\ref{lemma:techn:pflsu:consts} gives 
    \begin{align*}
        \expect_k[\Delta_{k+1}]
        \le \left(1 + \frac{2}{\tau-1}\right)\expect[\Delta_k] + C\,,
    \end{align*}
    and unrolling this recurrence gives for $k \le \tau-1$
    \begin{align*}
        \expect[\Delta_k] &\le \sum_{j=0}^{k-1} \left(1 + \frac{2}{\tau-1}\right)^j C
        \le \frac{\tau-1}{2} \left(1 + \frac{2}{\tau-1}\right)^k C \\
        &\le \frac{\tau-1}{2} \left(1 + \frac{2}{\tau-1}\right)^{\tau-1} C
        \le \frac{e^2}{2}(\tau-1)C\,,
    \end{align*}
    where we used $(1+1/\alpha)^\alpha \le e$ for all $\alpha>0$. Summing over $k$ and using the numerical bound $e^2<8$ completes the proof.
\end{proof}

\begin{remark} \label{remark:pfl-su:grad-diversity}
    We only invoked the partial gradient diversity assumption (Assumption~\ref{assmp:grad-diversity}) at iterates $(u\pow{t}, V\pow{t})$; therefore, it suffices if the assumption only holds at iterates $(u\pow{t}, V\pow{t})$ generated by FedSim, rather than at all $(u, V)$.
\end{remark}

\subsection{Technical Lemmas} \label{sec:a:techn}

The first lemma involves smoothness of two blocks of variables; we use this in the proof of \pflsu. 
\begin{lemma}[Block Smoothness] \label{lemma:techn:block-smoothness}
    Suppose $F_i : \reals^d \times \reals^{d_i}$ satisfy 
    Assumption~\ref{asmp:smoothness}. Then, it holds that
    \begin{align*}
        F_i(w', v'_i) - F_i(w, v_i)
        \le&\, \inp{\grad_w F_i(w, v_i)}{w'-w} 
         + \inp{\grad_v F_i(w, v_i)}{v_i' - v_i} \\
         &+ \frac{L_w}{2}(1+\chi^2)\normsq{w'-w}
         + \frac{L_v}{2}(1+\chi^2)\normsq{v_i'-v_i} \,.
    \end{align*}
\end{lemma}
\begin{proof}
Using the $L_w$-smoothness of $F(\cdot,v'_i)$ 
and the $L_v$-smoothness of $F(w,\cdot)$, we have
\begin{align*}
    F_i(w',v'_i)-F_i(w,v'_i) & \leq \inp{\nabla_w F_i(w,v'_i)}{w'-w} + \frac{L_w}{2}\|w'-w\|^2, \\
    F_i(w,v'_i)-F_i(w,v_i) & \leq \inp{\nabla_w F_i(w,v_i)}{v'_i-v_i} + \frac{L_v}{2}\|v'_i-v_i\|^2.
\end{align*}
Summing the above two inequalities together gives
\begin{align}
        F_i(w', v'_i) - F_i(w, v_i)
        \le&\, \inp{\grad_w F_i(w, v_i')}{w'-w} 
         + \inp{\grad_v F_i(w, v_i)}{v_i' - v_i} \nonumber\\
         &+ \frac{L_w}{2}\normsq{w'-w}
         + \frac{L_v}{2}\normsq{v_i'-v_i} \,.
\label{eqn:sum-smoothness}
\end{align}
We can bound the first inner product term on the right-hand side of the above inequality as
\begin{align*}
\inp{\grad_w F_i(w, v_i')}{w'-w} 
&= \inp{\grad_w F_i(w, v_i)}{w'-w} +\inp{\grad_w F_i(w, v_i')-\grad_w F_i(w,v_i)}{w'-w}  \\
&\leq \inp{\grad_w F_i(w, v_i)}{w'-w} +\norm{\grad_w F_i(w, v_i')-\grad_w F_i(w,v_i)}\norm{w'-w}  \\
&\leq \inp{\grad_w F_i(w, v_i)}{w'-w} +L_{wv}\norm{v_i'-v_i}\norm{w'-w}  \\
&\leq \inp{\grad_w F_i(w, v_i)}{w'-w} +\chi\sqrt{L_w L_v}\norm{v_i'-v_i}\norm{w'-w}  \\
&\leq \inp{\grad_w F_i(w, v_i)}{w'-w} +\chi^2\frac{L_v}{2}\normsq{v_i'-v_i}+\chi^2\frac{L_w}{2}\normsq{w'-w} ,
\end{align*}
where the first inequality is due to Cauchy-Schwarz, the second inequality is due to $L_{wv}$-Lipschitz property of $\nabla_w F_i(w,\cdot)$, the third inequality is due to the definition of~$\chi$ in~\eqref{eqn:chi-def}, and the last inequality is due to Young's inequality.
Substituting the above inequality into~\eqref{eqn:sum-smoothness} yields the desired result.
\end{proof}

Next, we have the variance of sampling without replacement. Note the correction factor of $(n-m)/(n-1)$ over sampling with replacement. We include the elementary proof for completeness.
\begin{lemma}[Sampling Without Replacement]
\label{lemma:techn:sampling-wo-replacement}
    Let $a_1, \cdots, a_n \in \reals^d$ be given. 
    Let $S$ be a uniformly random sample of size $m$ 
    from this collection, where the sampling is without replacement. Denoting the mean $\bar a = \sum_{i=1}^n a_i / n$, we have, 
    \[
        \expect_S\normsq*{\frac{1}{m}\sum_{i\in S} a_i - \bar a}
        \le \left(\frac{n-m}{n-1}\right) \frac{1}{m}\left(\frac{1}{n}\sum_{i=1}^n \normsq{a_i - \bar a}\right) \,.
    \]
\end{lemma}
\begin{proof}
    The statement is trivially true if $m=1$ or $m=n$. Therefore, we assume now that $2 \le m \le n-1$. Further, without loss of generality, we assume that $\bar a = 0$. 
    Finally, let $\Scal$ denote the set of all subsets of $[n]$ of size $m$. Note that $|\Scal| = {n \choose m}$. 
	We now have, 
	\[
	    \expect_S\normsq*{\frac{1}{m} \sum_{i \in S} a_i}
	    = \frac{1}{m^2 {n\choose m}} \sum_{S \in \Scal} \left( \sum_{i \in S} \normsq{a_i} + \sum_{i,j \in S:\, i\neq j} \inp{a_i}{a_j} \right) \,.
	\]
	For the first term, we have, 
	\[
	    \sum_{S \in \Scal} \sum_{i\in S} \normsq{a_i}
	    = \sum_{i=1}^n \sum_{S \in \Scal \,:\, i \in S}\normsq{a_i}
	    = {n-1 \choose m-1} \sum_{i=1}^n \normsq{a_i} \,.
	\]
	Likewise, for the second term, we use $\sum_{j \ne i} a_j = -a_i$ to get,
	\[
	     \sum_{i,j \in S:\, i\neq j} \inp{a_i}{a_j}
	     = \sum_{i=1}^n \sum_{j \neq i} \sum_{S \in \Scal \,:\, i,j \in S} \inp{a_i}{a_j}
	     = {n-2 \choose m-2} \sum_{i=1}^n \sum_{j\neq i} \inp{a_i}{a_j}
	     = -{n-2 \choose m-2} \sum_{i=1}^n  \normsq{a_i} \,.
	\]
	Therefore, we get,
	\[
	    \expect_S\normsq*{\frac{1}{m} \sum_{i \in S} a_i}
	    = \frac{{n-1 \choose m-1} - {n-2 \choose m-2}}{m^2{n \choose m}} \sum_{i=1}^n \normsq{a_i}
	    = \frac{{n-2 \choose m-1}}{m^2{n \choose m}} \sum_{i=1}^n \normsq{a_i}
	    = \frac{n-m}{mn(n-1)} \sum_{i=1}^n \normsq{a_i}\,.
	\]
\end{proof}

The next two lemmas are about the effect of the local updates in the local SGD literature. The first lemma has also appeared in~\citep{karimireddy2019scaffold}; we give the proof for completeness.
\begin{lemma} \label{lem:techn:client-drift}
	Consider $f: \reals^d \to \reals$ which is $L$-smooth and fix a $w\pow{0} \in \reals^d$. 
	Define the sequence $(w\pow{t})$ of iterates produced by stochastic gradient descent with a fixed learning rate $\gamma$
	starting from $w\pow{0}$:
	\[
		w\pow{t+1} = w\pow{t} - \gamma g\pow{t} \,,
	\]
	where $g\pow{t}$ is an unbiased (and independent of $w$) estimator of $\grad f(w)$ with bounded variance $\sigma^2$.
	Fix a number $\tau$ of steps. 
	If $\gamma \le (\sqrt{2}\tau L)^{-1}$,  we have the bound
	\[
		\sum_{t=0}^{\tau - 1} \normsq{w\pow{t} - w\pow{0}} \le 8 \gamma^2 \tau^2 (\tau - 1) \normsq{\grad f(w\pow{0})} + 4 \gamma^2 \tau^2(\tau-1) \sigma^2 \,.
	\]
\end{lemma}
\begin{proof}
	If $\tau = 1$, we have nothing to prove. Assume now that $\tau \ge 2$. 
	Let $\Fcal\pow{t}$ be the sigma-algebra generated by $w\pow{t}$ and denote $\expect_t[\cdot] = \expect[\cdot\, | \Fcal\pow{t}]$. 
	We will use the inequality
	\begin{align} \label{eq:tech:cdrift:1}
	\expect_t\normsq*{g\pow{t}} = \expect_t\normsq*{g\pow{t} - \grad f(w\pow{t})} + \normsq*{\grad f(w\pow{t})}
	\le \sigma^2 + \normsq*{\grad f(w\pow{t})}\,.
	\end{align}
	
	We now successively deduce,
	\begin{align*}
		\expect_t&\normsq{w\pow{t+1} - w\pow{0}} 
		=	\normsq{w\pow{t} - w\pow{0} - \gamma g\pow{t}} \\
		&\stackrel{(a)}{\le} \left(1 + \frac{1}{\tau-1}   \right)\normsq{w\pow{t} - w\pow{0}}
				+ \gamma^2\tau \expect_t\normsq{g\pow{t}}   \\
		&\stackrel{(b)}{\le} \left(1 + \frac{1}{\tau-1}  \right)\normsq{w\pow{t} - w\pow{0}} + 
			2\gamma^2\tau \normsq{\grad f(w\pow{t}) - \grad f(w\pow{0})} 
			+ 2\gamma^2\tau \normsq{\grad f(w\pow{0})} + \gamma^2\tau \sigma^2 \\
		&\stackrel{(c)}{\le} \left(1 + \frac{1}{\tau-1} + 2\gamma^2\tau L^2  \right)\normsq{w\pow{t} - w\pow{0}} 
			+ 2\gamma^2\tau \normsq{\grad f(w\pow{0})} + \gamma^2\tau \sigma^2 \\
		&\stackrel{(d)}{\le} \left(1 + \frac{2}{\tau-1} \right)\normsq{w\pow{t} - w\pow{0}} 
			+ 2\gamma^2\tau \normsq{\grad f(w\pow{0})} + \gamma^2\tau \sigma^2 \,.
	\end{align*}
	Above, we used 
	(a) the inequality $2\alpha\beta \le \alpha^2/\delta^2 + \delta^2 \beta^2$ for reals $\alpha, \beta, \delta$, 
	(b) Eq.~\eqref{eq:tech:cdrift:1}, 
	(c) $L$-smoothness of $f$, and, 
	(d) the condition on the learning rate. 

    Let $C = 2\gamma^2\tau\normsq{\grad f(w\pow{0})} + \gamma^2 \tau \sigma^2$.
	Unrolling the inequality and summing up the series gives for all $t \le \tau-1$
	\begin{align*}
		\normsq{w\pow{t} - w\pow{0}} 
		&\le C\sum_{j=0}^{t-1} \left(1 + \frac{2}{\tau-1} \right)^j \le \frac{C}{2}(\tau - 1) \left(1 + \frac{2}{\tau-1} \right)^t  \\
		&\le  \frac{C}{2}(\tau - 1) \left(1 + \frac{2}{\tau-1} \right)^{\tau-1} \le  \frac{C}{2}(\tau - 1)e^2 \,,
	\end{align*}
	where we used the bound $(1 + 1/\alpha)^\alpha \le e$ for all $\alpha > 0$. Summing over $t$ and using the numerical bound $e^2 < 8$ completes the proof.
\end{proof}

\begin{lemma} \label{lem:techn:client-drift-2}
	Consider the setting of Lemma~\ref{lem:techn:client-drift}.
	If $\gamma \le (2\tau L)^{-1}$,  we have the bound
	\[
		\normsq{w\pow{\tau} - w\pow{0}} \le 16 \gamma^2 \tau^2  \normsq{\grad f(w\pow{0})} + 8 \gamma^2 \tau^2 \sigma^2 \,.
	\]
\end{lemma}
\begin{proof}
    Proceeding similar to the last proof (expect using $\delta=\tau$) gives us 
    \[
    \expect_t\normsq*{w\pow{t+1} - w\pow{0}}
    \le \left(1+\frac{2}{\tau}\right)\normsq*{w\pow{t} - w\pow{0}} + 4\gamma^2 \tau\normsq*{\grad f(w\pow{0})} + 2\gamma^2 \tau \sigma^2 \,.
    \]
    Unrolling and summing up the sequence completes the proof, similar to that of Lemma~\ref{lem:techn:client-drift}.
\end{proof}

The next lemma is about bounding constants. 
\begin{lemma} \label{lemma:techn:pflsu:consts}
    Let $\gamma_u, \gamma_v, L_w, L_v, \chi, f \in \reals_+$ and a natural number $\tau$ be given.
    Denote 
    \[
        A := \gamma_u L_u^2 + f \gamma_v \chi^2 L_u L_v\,,
        \quad\text{and},\quad
        B := f \gamma_vL_v^2 + \gamma_u \chi^2 L_{u} L_v \,.
    \]
    Suppose $\gamma_u = c_u / (\tau L_u)$ and 
    $\gamma_v = c_v / (\tau L_v)$
    with $c_u, c_v > 0$ satisfying 
    \[
        c_u, c_v \le \frac{1}{\sqrt{6}} \max\{1, \chi^{-2}\} \,.
    \]
    Then, we have that 
    \[
        \gamma_v^2 \chi^2 L_uL_v B + \gamma_u^2 L_u^2 A \le A/(3\tau^2)\,,
        \quad\text{and}, \quad
        \gamma_u^2 \chi^2 L_u L_v A + \gamma_v^2 L_v^2 B \le B / (3\tau^2) \,.
    \]
\end{lemma}
\begin{proof}
    Note that it suffices to show 
    \[
    3\tau^2 \chi^2 \gamma_v^2 L_u L_v  B \le A / 2 \,,
    \quad\text{and},\quad
    3\tau^2 \chi^2 \gamma_u^2 L_u L_v A \le B/2 \,.
    \]
    Plugging in $\gamma_u, \gamma_v$, these are equivalent to
    \[
        6\chi^2 f c_v^3 + 6 \chi^4 c_v^2 c_u \le \chi^2 f c_v + c_u
        \quad\text{and},\quad
        6\chi^2 c_u^3 + 6 \chi^4 f c_v c_u^2 \le f c_v + \chi^2 c_u \,.
    \]
    The assumption on $c_v$ implies that 
    $6\chi^2 f c_v^3 \le \chi^2 f c_v$ and 
    $6 \chi^4 c_v^2 c_u \le c_u$. Therefore, the first condition holds. Similarly, the second condition holds too. 
\end{proof}

The final lemma is about tuning the learning rate: the proof is elementary and is omitted. 
\begin{lemma} \label{lem:sfl:best-params-2}
    Consider the map $\varphi: (0, \Gamma] \to \reals_+$ given by
    \[
        \varphi(\gamma) = 
        \frac{A}{\gamma T}  + B \gamma + C \gamma^2 \,,
    \]
    where $\Gamma,  A, B, C > 0$ are given. 
    Then, we have, 
    \[
        \varphi(\gamma^\star) \le 
        \frac{A}{\Gamma T} + 2\left(\frac{AB}{T}\right)^{1/2}
             + 2 C^{1/3} \left( \frac{A}{T} \right)^{2/3} \,,
    \]
    where $\gamma^\star$ is given by
    \[
        \gamma^\star = \min\left\{ 
        \Gamma, \sqrt{\frac{A}{BT}}, \left(\frac{A}{CT}\right)^{1/3} \right\} \,.
    \]
\end{lemma} 
\section{Experiments: Detailed Setup and Hyperparameters} \label{sec:a:expt-setup}

We conduct our experiments on {four} datasets from three modalities, namely images, text, and speech.
The  datasets contain a natural, non-i.i.d. split of data which is reflective of data heterogeneity encountered in federated learning.
We describe in detail the experimental setup and hyperparameters. The code to reproduce the experimental results will be publicly released. 

The outline of this section is:
\begin{itemize}[itemsep=0cm,leftmargin=0.5cm,topsep=0cm]
    \item \S\ref{sec:a:expt:datasets} describes the tasks and their associated datasets and metrics. 
    \item \S\ref{sec:a:expt:pipeline} describes the experimental pipeline as well as the baselines we compare to.
    \item \S\ref{sec:a:expt:hyperparameters} presents the hyperparameters of all the algorithms. 
\end{itemize}

As discussed in \S\ref{sec:intro}, we take the weight $\alpha_k$ to be proportional to the number of datapoints available on the device.

\begin{figure*}[t]
\includegraphics[width=0.98\textwidth]{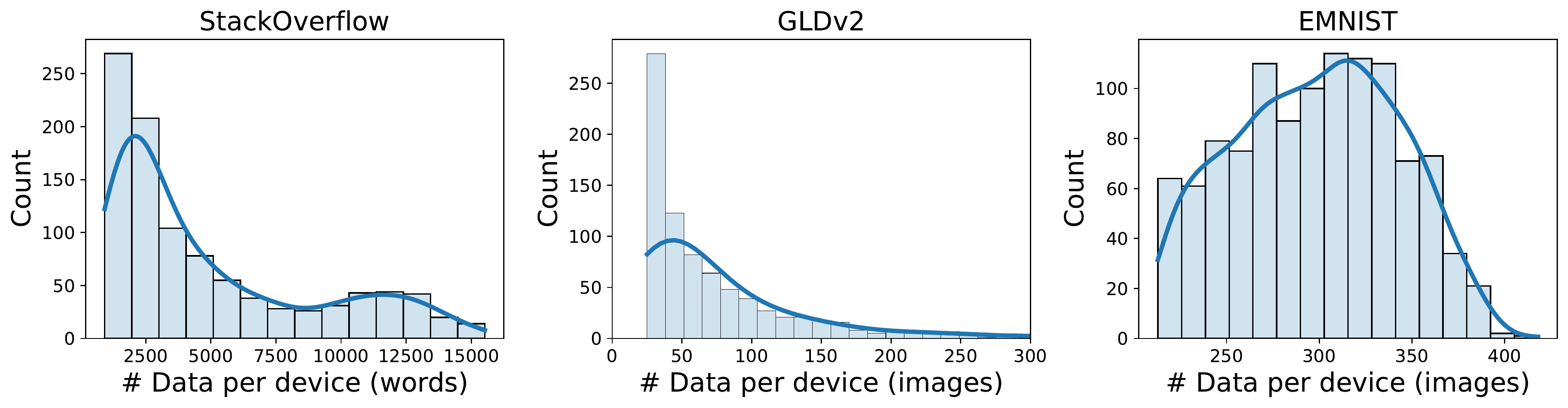}
\caption{\small{
    Distribution of number of training samples per device for each of the tasks considered in the experiments. For GLDv2, we do not show the long right tail, where the maximum number of data points per device is $1000$ (cf. Table~\ref{table:expt:dataset-summary}).
}}
\label{fig:expt:ds:hist}
\end{figure*}

\subsection{Datasets, Tasks and Models}\label{sec:a:expt:datasets}
We consider four tasks motivated by 
real-world applications of federated learning.
The tasks are summarized in Table~\ref{table:expt:dataset-summary} of the main paper and the distribution of data across the clients is visualized in Figure~\ref{fig:expt:ds:hist}. 

For each model, we consider three partial personalization architectures: 
\begin{enumerate}[itemsep=0cm,leftmargin=0.5cm,topsep=0cm,label=(\alph*)]
\item \textbf{Input layer personalization}: 
    Motivated by \citet{liang2020think}, this architecture places the first layer on-device to learn a personalized representation per-client, 
    while the rest of the model is shared.
    For the next-word prediction transformer model, we use the first transformer layer in place of the word embedding layer owing to its large size. 
\item \textbf{Output layer personalization}:
    Motivated by \citet{collins2021expoiting}, this architecture learns a shared global representation 
    but personalizes the prediction layer. 
    For the next-word transformer model, we use the last transformer layer in place of the last prediction layer owing to its large size. For the same reason, we use the second fully connected layer within the final transformer block for the speech-to-text transformer. 
\item \textbf{Adapter personalization}: 
    We also consider a novel partial personalization architecture, where the full model is shared among all clients, while each client adds personalized adapter modules, which are lightweight modules added between layers of the shared model. We use the transformer adapters proposed by \citet{houlsby2019parameter} 
    and residual adapters proposed by \citet{rebuffi2017learning}.
\end{enumerate}

\subsubsection{StackOverflow for Next Word Prediction}
\myparagraph{Dataset}
The StackOverflow dataset comprises of questions and answers from the programming question-answer website \href{Stack  Overflow}{stackoverflow.com}. 
The goal of the next word prediction task is to predict the next word given a partial sequence of words in a question or answer. This task is a good open-source benchmark for next word predictions in mobile keyboards. We use the StackOverflow dataset provided by~\citet{tff}.

\myparagraph{Client Distributions}
Each client corresponds to one user on Stack Overflow;
the data on the client corresponds to the questions and answers posted by this user.
We only consider clients with at least $100$ training sequences and $10$ testing sequences, where a sequence refers to either a question or an answer. 
We use a fixed subsample of $1000$ of them.
Following \citet{reddi2021adaptive}, 
we restrict the vocabulary to the top $10000$ most frequently occurring words in the dataset. 
We pad and truncate each sequence of each client to length $20$ and consider at most $1000$ training sequences on each client.

\myparagraph{Model}
We use a transformer model~\citep{vaswani2017attention} 
commensurate in size with BERT Mini~\citep{turc2019well}. 
It has with $4$ transformer blocks and $4$ attention heads in each self-attention layer with a transformer hidden dimension of $256$ and a fully-connected hidden dimension of $1024$. 
The output layer is a causal language modeling head, i.e., a fully connected layer which assigns a score for each possible vocabulary item, including the special tokens.
The model has $6$ million parameters, which require around $23$ megabytes of memory. 

\myparagraph{Partial Personalization Architecture}
The partial personalization architectures used are summarized in Table~\ref{table:a:expt:so:split}.

\myparagraph{Loss Function and Evaluation Metric}
We train the model with the causal language modeling objective. That is, for each partial sequence, we treat the prediction of the next word as a multiclass classification problem to minimize the multinomial logistic loss, also known as cross entropy loss.  
For evaluation, we use the top-$1$ accuracy of predicting words in the proper $10000$-word vocabulary (i.e., ignoring special tokens such as padding, out-of-vocabulary, and beginning/end of sequence).

\begin{table}[t]
    \caption{\small{Summary of partial personalization architectures for the transformer model for next word prediction.}}
\label{table:a:expt:so:split}
\begin{center}
\begin{adjustbox}{max width=\linewidth}
\small
\begin{tabular}{lccc}
\toprule
Personalization Type  &  Layer on-device & 
\begin{tabular}{c} \# Personalized \\ Params. \end{tabular}
&
\begin{tabular}{c} \# Shared \\ Params. \end{tabular} \\
\midrule
Input Layer & 1st transformer block & 
$0.8M$ & $4.9M$ \\
Output Layer & Last transformer block & 
$0.8M$ & $4.9M$ \\
Adapter & Adapter modules & $0.07M$ & $5.7M$ \\
\bottomrule
\end{tabular}
\end{adjustbox}
\end{center}
\end{table}

\subsubsection{GLDv2 for Visual Landmark Recognition}

\myparagraph{Dataset}
GLDv2 stands for Google Landmarks Dataset v2~\citep{weyand2020google}, which is a large-scale image dataset. It contains images of popular landmarks from around the world taken and uploaded by Wikipedia contributors. While the images vary in size, the most common image size is $800 \times 600$ pixels.

The goal of the visual landmark recognition task is to identify the landmark from its image. This task resembles a scenario where
smartphone users take photos of natural and architectural landmarks while traveling.
We use the federated version of the GLDv2 dataset introduced by \citet{hsu2020federated} with $2028$ landmarks and provided by \citet{tff}.

\myparagraph{Client Distributions}
Each client corresponds to one Wikipedia user and contains all the images contributed by that user. 
We only all $823$ clients with at least $50$
datapoints. We do not use original test set from GLDv2 from evaluation as it comes from different clients. Instead, we take $50\%$ of the data on each client as a testing set.

\myparagraph{Model}
We use a ResNet-18~\citep{he2016deep} model pretrained on ImageNet~\citep{deng2009imagenet}, 
with group normalization instead of batch normalization~\citep{hsieh2020noniid}.
We resize all images to $224\times 224$.
We use two data augmentations for training:
a random crop from $256\times 256$ and a random horizontal flip.
The model has $12$ million parameters, which require around $49$ megabytes of storage. 

\myparagraph{Partial Personalization Architecture}
The partial personalization architectures used are summarized in Table~\ref{table:a:expt:gldv2:split}.

\myparagraph{Loss Function and Evaluation Metric}
We use the multinomial logistic loss, also known as cross entropy loss. 
We evaluate the performance of the model using its classification accuracy.

\begin{table}[t]
    \caption{\small{Summary of partial personalization architectures for the ResNet-18 model for visual landmark recognition.}}
\label{table:a:expt:gldv2:split}
\begin{center}
\begin{adjustbox}{max width=\linewidth}
\small
\begin{tabular}{lccc}
\toprule
Personalization Type  &  Layer on-device & 
\begin{tabular}{c} \# Personalized \\ Params. \end{tabular}
&
\begin{tabular}{c} \# Shared \\ Params. \end{tabular} \\
\midrule
Input Layer & 1st conv. layer & 
$0.01M$ & $12.2M$ \\
Output Layer & Last fully connected layer & 
$1M$ & $11.2M$ \\
Adapter & Residual adapter modules & $1.4M$ & $12.2M$ \\
\bottomrule
\end{tabular}
\end{adjustbox}
\end{center}
\end{table}

\subsubsection{EMNIST for Character Recognition}
\myparagraph{Dataset}
EMNIST~\citep{cohen2017emnist} is a character recognition dataset. The goal is to identify images of handwritten digits or letters; there are 62 possible options (a-z,A-Z, 0-9).
The images are grey-scaled pictures of $28 \times 28 = 784$ pixels. We use the EMNIST dataset provided by~\citet{tff}.

\myparagraph{Client Distributions}
Each client corresponds to one ``writer'', i.e.,
the human subject who hand-wrote the digit/letter during the data collection process. 
We only use those clients with at least $100$ training points and $25$ testing points: there are $1114$ of such clients.

\myparagraph{Model}
We use a ResNet-18~\citep{he2016deep} model
with group normalization instead of batch normalization~\citep{hsieh2020noniid}.
We make two modifications to handle the smaller image size ($28\times28\times 1$ as opposed to the $224\times224\times3$ which the original ResNet was designed to accept): 
(a) we use a convolutional kernel of size $3\times 3$ rather than the original $7 \times 7$ in the first convolution layer, and, (b) we drop the first pooling layer.
The model has $11$ million parameters, which require around $45$ megabytes. 
Note that the number of parameters in this ResNet is smaller than the one for GLDv2 due to the architectural modifications we make for smaller images as well as the smaller number of classes.

\myparagraph{Partial Personalization Architecture}
The partial personalization architectures used are summarized in Table~\ref{table:a:expt:emnist:split}.

\myparagraph{Loss Function and Evaluation Metric}
We use the multinomial logistic loss, also known as cross entropy loss. 
We evaluate the performance of the model using its classification accuracy.

\begin{table}[t]
    \caption{\small{Summary of partial personalization architectures for the ResNet-18 model for character recognition.}}
\label{table:a:expt:emnist:split}
\begin{center}
\begin{adjustbox}{max width=\linewidth}
\small
\begin{tabular}{lccc}
\toprule
Personalization Type  &  Layer on-device & 
\begin{tabular}{c} \# Personalized \\ Params. \end{tabular}
&
\begin{tabular}{c} \# Shared \\ Params. \end{tabular} \\
\midrule
Input Layer & 1st conv. layer & 
$0.7K$ & $11.2M$ \\
Output Layer & Last fully connected layer & 
$0.03M$ & $11.2M$ \\
Adapter & Residual adapter modules & $1.4M$ & $11.2M$ \\
\bottomrule
\end{tabular}
\end{adjustbox}
\end{center}
\end{table}

\subsubsection{LibriSpeech for Automatic Speech Recognition}
\myparagraph{Dataset}
Librispeech is a speech-to-text dataset containing snippets of speech and the associated text from open domain audiobooks~\citep{panayotov2015librispeech}. 
Given an utterance containing read English speech, 
the goal is output a text transcription. 
Each device corresponds to the narrator of the utterance, leading to a natural non-identical split of the data with differences in accent, tone, and voice across devices. 
This task is reflective of voice commands and speech recognition on mobile phones. 

We create a federated version of LibriSpeech. We use the ``clean`` subsets of LibriSpeech (a total of $460$h of speech) to pretrain a model in a non-federated manner. We use the ``train-other-500`` subset (a total of $500$h of audio), which typically contains noiser audio, to construct a federated dataset. Real-world federated tasks often contain proxy data used to pretrain a model prior to federated training, such as ImageNet-pretrained vision models. We emulate this setup by first pretraining  all our models on the non-federated clean subset of LibriSpeech. 

\myparagraph{Client Distributions}
We construct the federated dataset from the train-other-500 subset of LibriSpeech and do not use the corresponding dev and test sets. Of the $1166$ narrators, we discard those with only one chapter of data.\footnote{
LibriSpeech organizes the data for each narrator into 
chapters of the source book.
}
For each narrator, we assign one chapter as the test data and the remaining as the training data. This is done to ensure that each device has between $10-50$\% of the device's total data in terms of length of audio\footnote{
When multiple candidate chapters are available for use as a test set, we use the one closest in size to 20\% of the data.
} --- 
this leads to approximately $30$\% of the available audio being used for testing and the remaining $70$\% for training.
Overall, we get a federated dataset with $902$ narrators, each of whom corresponds to a device in the federated setting.

\myparagraph{Model}
We use a transformer model~\citep{vaswani2017attention}
with convolutional subsamplers, as proposed by \citet{synnaeve2019end}. 
The input audio is represented as a sequence of $40$ log-mel filterbank coefficients.
The model has two 1D convolutional layers with a stride of $2$, followed by $6$ transformer blocks and $6$ attention heads in each self-attention layer with a transformer hidden dimension of $384$ and a fully-connected hidden dimension of $1536$. 
The final output layer produces log probabilities on an output vocabulary of $5000$ byte pair encodings of subwords. 
The model has $15$ million parameters, requiring around $60$ megabytes of memory. 

\myparagraph{Partial Personalization Architecture}
The partial personalization architectures used are summarized in Table~\ref{table:a:expt:ls:split}.

\myparagraph{Loss Function and Evaluation Metric}
We train the model with the Connectionist Temporal Classification (CTC) loss~\citep{graves2006connectionist}.
This is a structured prediction loss that uses dynamic programming to marginalize over all possible alignments between the per-frame subwords and the text transcription.
For evaluation, we use the word error rate (WER) obtained from a greedy decoding of the model prediction for a given utterance (or equivalently, beam search with a beam size of $1$ with no external language models). 

\begin{table}[t]
    \caption{\small{Summary of partial personalization architectures for the transformer model for speech recognition.}}
\label{table:a:expt:ls:split}
\begin{center}
\begin{adjustbox}{max width=\linewidth}
\small
\begin{tabular}{lccc}
\toprule
Personalization Type  &  Layer on-device & 
\begin{tabular}{c} \# Personalized \\ Params. \end{tabular}
&
\begin{tabular}{c} \# Shared \\ Params. \end{tabular} \\
\midrule
Input Layer & Convolutional subsamplers & 
$0.8M$ & $12.6M$ \\
Output Layer & 2nd f.c. in last transformer block & 
$0.6M$ & $12.8M$ \\
Adapter & Adapter modules & $0.15M$ & $13.4M$ \\
\bottomrule
\end{tabular}
\end{adjustbox}
\end{center}
\end{table}

\subsection{Experimental Pipeline and Baselines} \label{sec:a:expt:pipeline}
There are three components in the training pipeline for all experiments:
\begin{enumerate}[itemsep=0cm,leftmargin=0.5cm,topsep=0cm,label=(\alph*)]
    \item \label{pipeline:expt:non-pers}
    {Non-personalized federated training}: The first step involves training a global model $w_g$ using the one-model-fits-all approach of \eqref{eqn:one-fits-all} with FedAvg variants.
    \item \label{pipeline:expt:pfl}
    {Personalized federated training}: This optional second step involves training the shared parameters $w$ together with the personalized parameters $v_k$ using a personalized federated learning approach. We warm-start $w, v_k$ from the non-personalized model $w_g$ from the previous step.
    \item \label{pipeline:expt:pfl-ft} 
    {Final finetuning}: The last step involves only finetuning the personalized parameters $v_k$ while the shared parameters $w$ remain unchanged. 
\end{enumerate}

For step~\ref{pipeline:expt:pfl}, we initialize $v_k$ for each $k$ to be the appropriate part of $w_g$ for input/output layer personalization. On the other hand, for adapters, we initialize $v_k$ to be equal to the \emph{same} set of randomly initialized weights for each device $k$.

We consider the following baselines:
\begin{itemize}[itemsep=0cm,leftmargin=0.5cm,topsep=0cm]
    \item \textbf{Non-personalized}: This denotes the performance of step \ref{pipeline:expt:non-pers} of the pipeline above, i.e., non-personalized federated training with FedAvg variants.
    \item \textbf{Full model personalization}: We consider three baselines of personalization of the full model: 
        \begin{enumerate}[itemsep=0cm,leftmargin=1cm,topsep=0cm,label=(\roman*)]
            \item \textbf{Finetune}: The non-personalized model from step (a) of the pipeline above is finetuned locally on each client (step \ref{pipeline:expt:pfl-ft} of the pipeline). Step~\ref{pipeline:expt:pfl} is skipped for this baseline. 
            \item \textbf{Ditto}~\citep{Li2021ditto}: The non-personalized model from step (a) of the pipeline above is finetuned locally on each client (step \ref{pipeline:expt:pfl-ft} of the pipeline) with $\ell_2$ regularization $\norm{v - w_g}^2$. Step~\ref{pipeline:expt:pfl} is skipped for this baseline. 
            \item \textbf{pFedMe}~\citep{dinh2020moreau}: The non-personalized baseline model from step \ref{pipeline:expt:non-pers} is trained further in step~\ref{pipeline:expt:pfl} to optimize \eqref{eqn:pfl-full-model} using the pFedMe algorithm of~\citet{dinh2020moreau}. Finally the resulting model $w$ is finetuned locally in step~\ref{pipeline:expt:pfl-ft}.
        \end{enumerate}
    \item \textbf{Partial Model Personalization}: We consider partial model personalization with three different architectures, as defined in \S\ref{sec:a:expt:datasets}. For each personalization approach, we start with the non-personalized model in step~\ref{pipeline:expt:non-pers}, continue personalization in step~\ref{pipeline:expt:pfl} using either \pflam or \pflsu as the algorithm, and finally run step~\ref{pipeline:expt:pfl-ft} for the local finetuning.
\end{itemize}

\begin{table*}[t]
 \caption{\small{
    Hyperparameters for each dataset/task.
    }}
\label{table:expt:hyperparam}
\begin{center}
\begin{adjustbox}{max width=\linewidth}
\small
\small
\begin{tabular}{cccccc}
\toprule
& Hyperparameter & StackOverflow & GLDv2 & EMNIST & LibriSpeech \\
\midrule

\multirow{9}{*}{Common} 
& Batch size & 64 & 64 & 32 & 32\\
& Devices per round & 50 & 50 & 10 & 50 \\
& Local epochs & 1 & 1 & 1 & 1\\
& Server Optimizer & FedAdam & FedAdam & FedAvg & FedAdam \\
& Client Optimizer & SGD & SGD & SGD & SGD  \\
& Global Scheduler & {Linear} & Linear & Exponential & Linear \\
& Warm up & $10\%$ of rounds & $10\%$ of rounds & N/A & $10\%$ of rounds\\
& LR decay rounds & N/A & N/A & $500$ & N/A\\
& Max. grad. norm. & $0.1$ & N/A & N/A & $0.25$ \\
\midrule
\multirow{3}{*}{\begin{tabular}{l} Non-personalized training \\ (step (a) of the pipeline)\end{tabular}} 
& \# Rounds & 1000 & 2500 & 2000 & 500  \\
& Server learning rate & $5\times 10^{-4}$ & $2 \times 10^{-4}$ & 1.0 & $10^{-3}$ \\
& Client learning rate & $1$ & $10^{-2}$ & $0.5$ & $10^{-2}$
\\
\midrule
\multirow{3}{*}{\begin{tabular}{l} Personalized training \\ (step (b) of the pipeline)\end{tabular}} 
& \# Rounds & 500 & 600 & 500 & 500 \\
& Server learning rate & $5\times 10^{-5}$ & $2 \times 10^{-5}$ & 1.0 & $10^{-3}$ \\
& Client learning rate & $10^{-1}$ & $10^{-3}$ & $10^{-2}$ & $10^{-2}$ \\
\midrule
\multirow{3}{*}{\begin{tabular}{l} Local finetuning \\ (step (c) of the pipeline)\end{tabular}} 
& \#Epochs & 5 & 5 & 5 & 5 \\
& Optimizer & SGD & SGD & SGD & SGD \\
& Client learning rate & $10^{-1}$ & $10^{-3}$ & $10^{-2}$ & $10^{-4}$ \\
\bottomrule
\end{tabular} \end{adjustbox}
\end{center}
\end{table*}

\subsection{Hyperparameters and Evaluation Details}\label{sec:a:expt:hyperparameters}

All the tuning of hyperparameters was performed on validation data, formed by holding out $20\%$ of the training data on each device. Once the tuning was complete, we reran the experiments on the full training data, including those held out for validation.

\myparagraph{Evaluation Metric}
Our primary evaluation metric for next-word prediction and image classification is the weighted average of the test accuracy on each client, weighted by the number of test examples (the details of how the accuracy is computed on each dataset is given in \S\ref{sec:a:expt:datasets} in the paragraph on ``Loss Function and Evaluation Metric'').
This corresponds to the unweighted accuracy obtained by pooling all the data locally, similar to the loss as discussed in \S\ref{sec:intro}.
The same metric is used for hyperparameter tuning and is reported in all the tables and plots, unless explicitly noted otherwise. For speech recognition, we similarly use a weighted average of the word error rate (WER).

The final hyperparameters we use are given in Table~\ref{table:expt:hyperparam}.

\myparagraph{Rounds}    
We start with the number of communication rounds (i.e., the number of calls to secure aggregation routine for the shared parameters), which is used to measure the progress of each algorithm. For the non-personalized training, we use $1000$ rounds for StackOverflow, $2500$ rounds for GLDv2 and $2000$ rounds for EMNIST. 
For the personalized training, we warm-start the model from the non-personalized one, and run the training for $500$ rounds for StackOverflow and EMNIST and $600$ rounds for GLDv2.
 
\myparagraph{Devices per Round}   
All devices are assumed to be available and selections are made uniformly at random. 
Following \citep{reddi2021adaptive,weyand2020google}, we select $50$ devices per round for StackOverflow/GLDv2 and $10$ per round for EMNIST,
for both the non-personalized as well as the personalized training.

\myparagraph{Local Updates and Minibatch Size}
Each selected device locally runs $1$ epoch of mini-batch stochastic gradient descent locally for non-personalized as well as personalized federated training. The final finetuning at the end of personalized training is performed for $5$ epochs. 
We use a minibatch size of $64$ for StackOverflow/GLDv2 and $32$ for EMNIST for all settings.
    
\myparagraph{Server and Client Optimizer Details} 
We use FedAvg for EMNIST and FedAdam~\citep{reddi2021adaptive} for StackOverflow and GLDv2. We also use a global scheduler, which applies a schedule on the client learning rates across rounds, while the client learning rate within each round is held constant. We use either a linear scheduler or an exponential scheduler (also called ``stepLR'' in PyTorch). A linear scheduler applies a linear warmup, if applicable, until the maximum learning rate followed by a linear decay to $0$. An exponential scheduler halves the client learning rate once every fixed number of rounds. 
Both the client and server learning rates are tuned using the validation set. 

\myparagraph{Regularization Coefficient for pFedMe and Ditto}
We tune the regularization coefficient $\lambda_k = \lambda$ for pFedMe and Ditto using the validation data from the set $\{10^{-4}, 10^{-3}, \cdots, 10^{0}\}$ of possible values. The tuned values are:
\begin{itemize}[itemsep=0cm,leftmargin=0.5cm,topsep=0cm]
    \item StackOverflow: $10^{-3}$ for Ditto and  $10^{-4}$ for pFedMe,
    \item GLDv2: $10^{-1}$ for both Ditto and pFedMe,
    \item EMNIST: $10^{-1}$ for both Ditto and pFedMe.
\end{itemize}

\myparagraph{Random Seed}
We report numbers averaged over $5$ random seeds for all experiments, with the exception of the speech recognition task.

\begin{table*}
\caption{\small 
Memory requirements (in megabytes) for training partial model personalization and full model personalization for the experimental settings considered here.}  \label{tab:pfl:expt:memory}
\begin{center}
\small
\begin{tabular}{lrrr}
\toprule
\textbf{Mode}       &     \textbf{StackOverflow} & \textbf{GLDv2} & \textbf{EMNIST} \\
\midrule
No personalization     & $71$  & $186$   & $142$    \\
Input layer personalization   & $67$  & $186$   & $142$    \\
Output layer  personalization  & $67$  & $174$   & $142$    \\
Adapter      personalization & $72$  & $222$   & $159$    \\
Full personalization   & $116$ & $263$   & $232$    \\
\midrule
Memory savings with partial personalization &
$\mathbf{42\%}$  & $\mathbf{34\%}$ &  $\mathbf{39\%}$ \\
\bottomrule
\end{tabular}
\end{center}
\end{table*}
    
\subsection{Estimated Memory Requirement} \label{sec:a:expt-setup:memory}

We estimate the memory footprint for partial versus full personalization during training below. During deployment, the memory footprint of partial and full model personalization is the same since one full model is deployed. 

\myparagraph{Estimation Procedure}
We assume that the following are needed to be stored on device $i \in S\pow{t}$ during round $t$ of training: 
\begin{itemize}
    \item $u\pow{t}$, the previous broadcast global model, which is needed to calculate the model delta to be sent back to the server,
    \item current iterate of the shared parameter $u_{i,k}\pow{t}$,
    \item current iterate of the personal parameter $v_{i, k}\pow{t}$,
    \item their respective gradients $\grad_u$ and $\grad_v$, and, 
    \item the internal buffers required for backpropagation. 
\end{itemize}

The total memory consumption is therefore, 
\[
    \text{Memory} = 3 \times \text{size}(u) + 2 \times \text{size}(v) + \text{size(backprop)} \,.
\]

We estimate the size of the backpropagation buffers for a batch size of $1$. 

\myparagraph{Training Memory Requirement}
For full model personalization $\text{size}(v) = \text{size}(u)$, whereas $\text{size}(v) \ll \text{size}(u)$ for the partial personalization architectures we have considered. Therefore, the total memory requirement of training partial model personalization will be smaller than full model model personalization. 

From Table~\ref{tab:pfl:expt:memory}, we see that partial personalization can result in a $34\%$ to $42\%$ reduction in the memory consumption across the models and datasets considered in the experiments.

\section{Experiments: Additional Results} \label{sec:a:expt-results}

We now present the detailed experimental results.

\subsection{Speech Recognition: FedAlt vs. FedSim} 
\label{sec:a:expt-results:speech}

We compare FedAlt and FedSim for speech recognition in Table~\ref{table:a:speech-all}. We find that input layer personalization with FedAlt has the smallest word error rate of all the models considered. 

\begin{table*}[t]
    \caption{\small{
    A comparison of FedAlt and FedSim on the \textbf{speech recognition} task in terms of the word error rate (WER) \%. Smaller values indicate better predictive performance.
    }}
\label{table:a:speech-all}
\begin{center}
\begin{adjustbox}{max width=\linewidth}
\small
\begin{tabular}{lrr}
\toprule
{Personalization} &    FedAlt &    FedSim \\
\midrule
Finetune         &  $15.55$ &  $15.55$ \\
Input Layer        &  \tabemph{} $\mathbf{15.13}$ &  $15.47$ \\
Output Layer        &  $15.53$ &  $15.51$ \\
Adapter       &  $15.50$ &  $15.54$ \\
\bottomrule
\end{tabular}

 \end{adjustbox}
\end{center}
\end{table*}

\subsection{Ablation: Final Finetuning for \pflam and \pflsu}
\label{sec:a:expt-results:final-finetuning}
We now study the effect of the final finetuning (step~\ref{pipeline:expt:pfl-ft} of the experimental pipeline; cf. \S\ref{sec:a:expt:pipeline}) for \pflam and \pflsu. 

\myparagraph{The final finetuning has a minimal impact on partial personalization}
We see from Table~\ref{table:expt:ft:stateful} that the effect of the final finetuning is much smaller than the improvements from personalization.
For instance, the improvements from finetuning are close to $0$ for \pflam on the StackOverflow dataset. For GLDv2, the finetuning accounts for $<0.5$pp of improvement, whereas personalization overall accounts for $5$ to $15$pp. 

\myparagraph{The final finetuning is more important to \pflsu than \pflam}
Table~\ref{table:expt:ft:stateful} also shows that the final finetuning helps \pflsu more than \pflam. However, \pflam still outperforms \pflsu, as we saw in  Table~\ref{table:expt:opt-algo:stateful}.
Overall, this shows that \pflam is a better algorithm than \pflsu. The final finetuning helps \pflsu make up some percentage points in accuracy, but not enough to make up its gap with \pflam.

\begin{table*}[t]
    \caption{\small{
    The change in accuracy (percentage points) from the final finetuning for \pflam and \pflsu with stateful devices. 
    The subscript denotes the standard deviation over 5 random seeds.
    }}
\label{table:expt:ft:stateful}
\begin{center}
\begin{adjustbox}{max width=\linewidth}
\small
\renewcommand{\arraystretch}{1.2}
\begin{tabular}{lrrrrrr}
\toprule
{} & \multicolumn{2}{c}{StackOverflow} & \multicolumn{2}{c}{GLDv2} & \multicolumn{2}{c}{EMNIST} \\
\cmidrule(lr){2-3}
\cmidrule(lr){4-5}
\cmidrule(lr){6-7}
{} &          FedAlt &         FedSim &         FedAlt &         FedSim &         FedAlt &         FedSim \\
\midrule
Input Layer  &  $-0.06_{0.01}$ &  $0.04_{0.02}$ &  $0.12_{0.02}$ &  $0.17_{0.03}$ &  $0.12_{0.01}$ &  $0.12_{0.03}$ \\
Output Layer &   $0.00_{0.01}$ &  $0.25_{0.02}$ &  $0.49_{0.02}$ &  $0.57_{0.03}$ &  $0.09_{0.01}$ &  $0.09_{0.03}$ \\
Adapter      &   $0.01_{0.01}$ &  $0.40_{0.08}$ &  $0.14_{0.02}$ &  $0.17_{0.01}$ &  $0.27_{0.02}$ &  $0.33_{0.03}$ \\
\bottomrule
\end{tabular}
 \end{adjustbox}
\end{center}
\end{table*}

\begin{figure*}[t]
\includegraphics[width=0.98\textwidth]{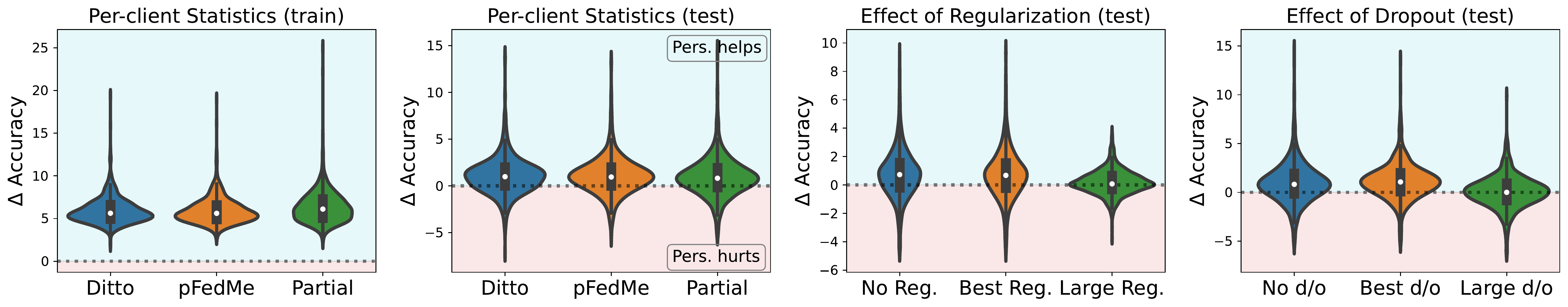}
\caption{\small{
    \textbf{Left two}: Distribution of change in the per-device train (left most) and test (center left) accuracy due to personalization on the StackOverflow dataset. 
    \textbf{Right two}: Distribution of change in the per-device test accuracy of partial personalization under regularization on the StackOverflow dataset: (a) center right: adapter personalization under $\ell_2$ regularization, and,
    (b) rightmost: output layer personalization under dropout.
    Note that the ``No Reg.'' and ``No d/o'' plots on the right two are different because they personalize different model parts. 
    \textbf{Interpretation}: The white dot in inside the violin denotes the median, while the black box enclosing this white dot marks the interquartile range (i.e., $25\textsuperscript{th}$ and $75\textsuperscript{th}$ percentiles). 
    The body of the violin is a kernel density estimate of the distribution of accuracies. The lines extend out to the minimum and maximum accuracy in each case.
}}
\label{fig:expt:violin}
\end{figure*}

\subsection{Effect of Personalization on Per-Device Generalization}
\label{sec:a:expt-results:generalization}

\myparagraph{Summary of all scatter plots}
All the scatter plots shown in the main paper are summarized in the violin plot of Figure~\ref{fig:expt:violin}. We see from the leftmost figure that the training accuracies on all devices improve with personalization. From the second figure, we see that the test accuracy of some of the devices reduces with personalization; this is true for both partial and full personalization.

From the third plot of Figure~\ref{fig:expt:violin}, we see that regularization does not mitigate this overfitting. In fact, the regularization tuned for best average accuracy leads to a nearly identical distribution of test accuracies. A larger regularization reduces the spread of accuracies, but does so at the expense of a smaller median (white dot). 
The fourth plot of Figure~\ref{fig:expt:violin} shows that the effect of dropout is similar. The best dropout improves the median accuracy, but it does not mitigate the issue of some devices being hurt by personalization.

\myparagraph{Train Accuracy plots for devices}
From Figure~\ref{fig:expt:scatter:train-test-so}, we see that personalization leads to \emph{a reduction in test accuracy} on some of the devices beyond the initial non-personalized model. The corresponding train accuracy plot is given in Figure~\ref{fig:expt:scatter:train-test-so}. We observe that the personalization always leads to an improvement in the training accuracy but not in the test accuracy.
The analogous plots for GLDv2 are in Figure~\ref{fig:expt:scatter:train-test-gldv2}, where the trends are similar.

\myparagraph{Whether personalization helps a device or not depends on the random seed}
We see in Figure~\ref{fig:expt:scatter:per-seed} that the shaded region for some of the devices intersects the dotted line at $0$. In other words, personalization sometimes helps this device and sometimes hurts it, depending on the random seed. This indicates that the best fix in practice is to use A/B testing on the \emph{deployed model} to choose whether to use the personalized model or the non-personalized one. 

\begin{figure*}[t]
\includegraphics[width=0.98\textwidth]{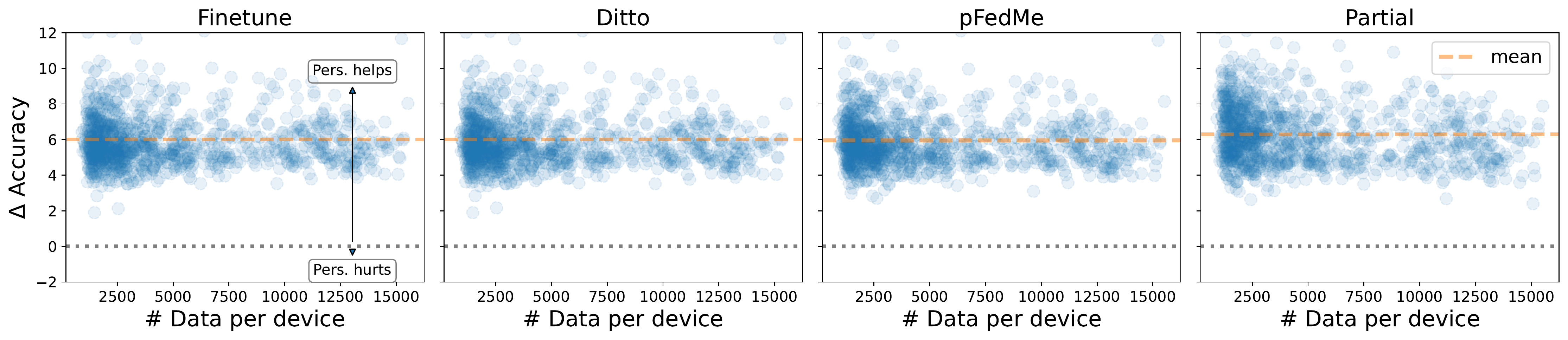}
\includegraphics[width=0.98\textwidth]{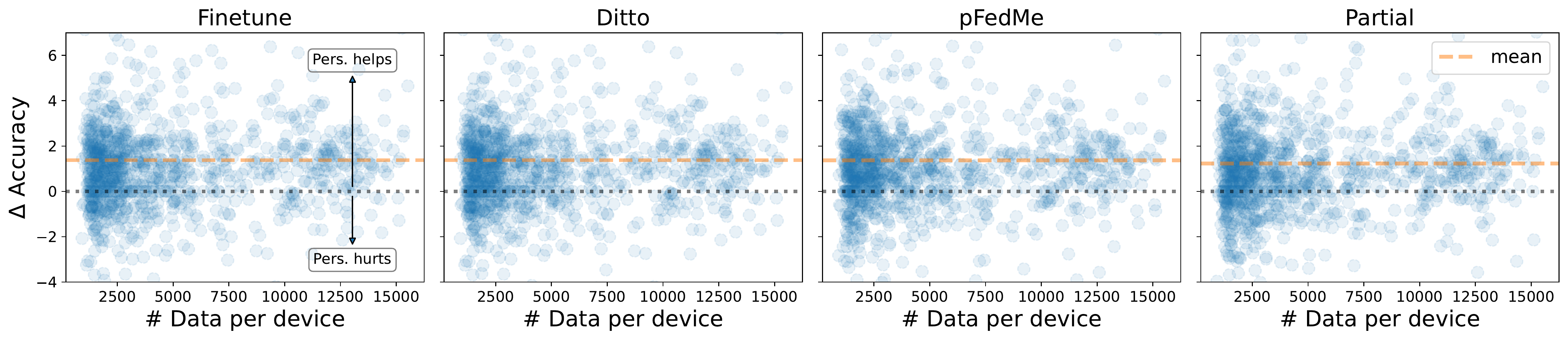}
\caption{\small{
    Scatter plot of change in accuracy (pp) per-device versus the number of training samples on the device for StackOverflow. \textbf{Top}: Training accuracy. \textbf{Bottom}: Test accuracy.
    This is the full version of Figure~\ref{fig:expt:scatter:main-so} from the main paper.
}}
\label{fig:expt:scatter:train-test-so}
\end{figure*}

\begin{figure*}[t]
\includegraphics[width=0.98\textwidth]{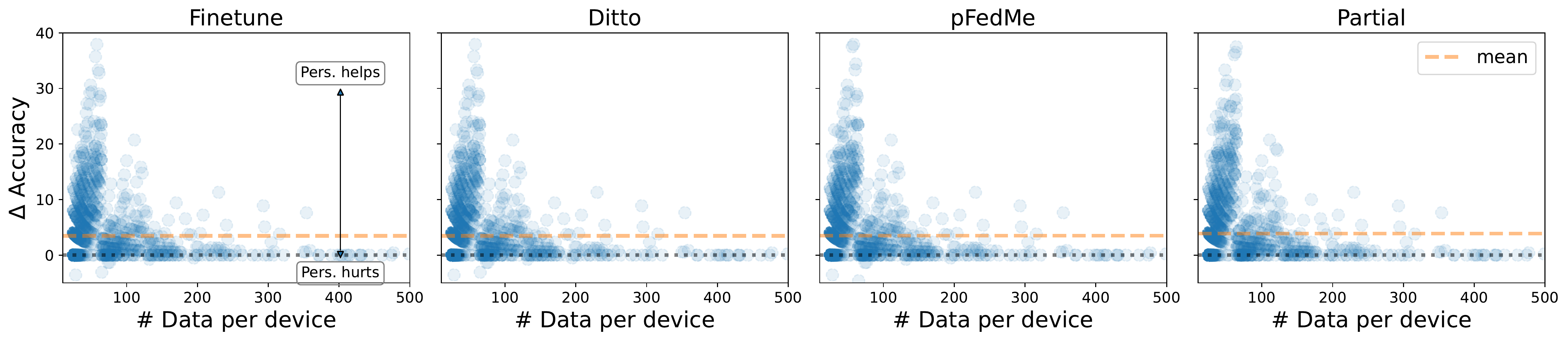}
\includegraphics[width=0.98\textwidth]{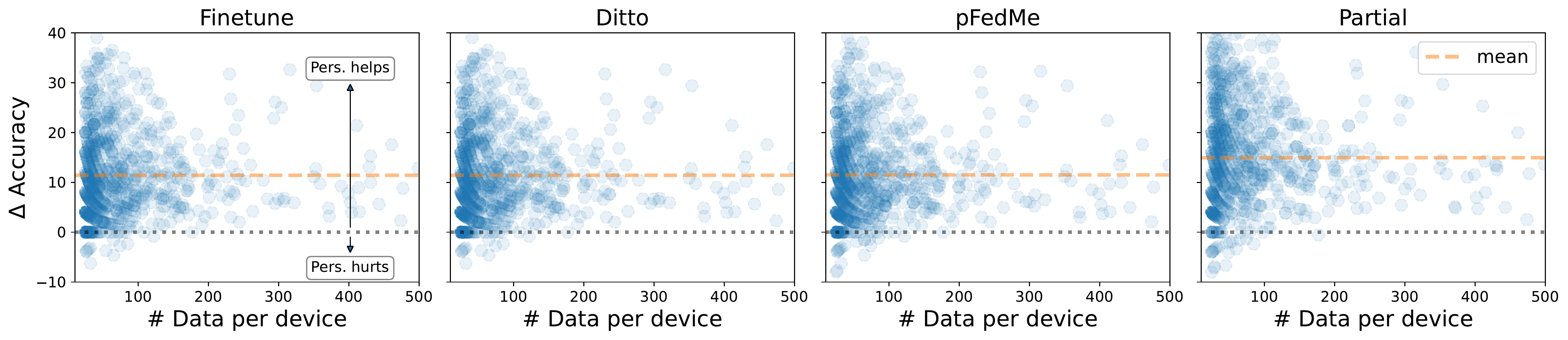}
\caption{\small{
    Scatter plot of change in accuracy (pp) per-device versus the number of training samples on the device for GLDv2. \textbf{Top}: Training accuracy. \textbf{Bottom}: Test accuracy.
}}
\label{fig:expt:scatter:train-test-gldv2}
\end{figure*}

\begin{figure*}[t]
\centering
\includegraphics[width=0.85\textwidth]{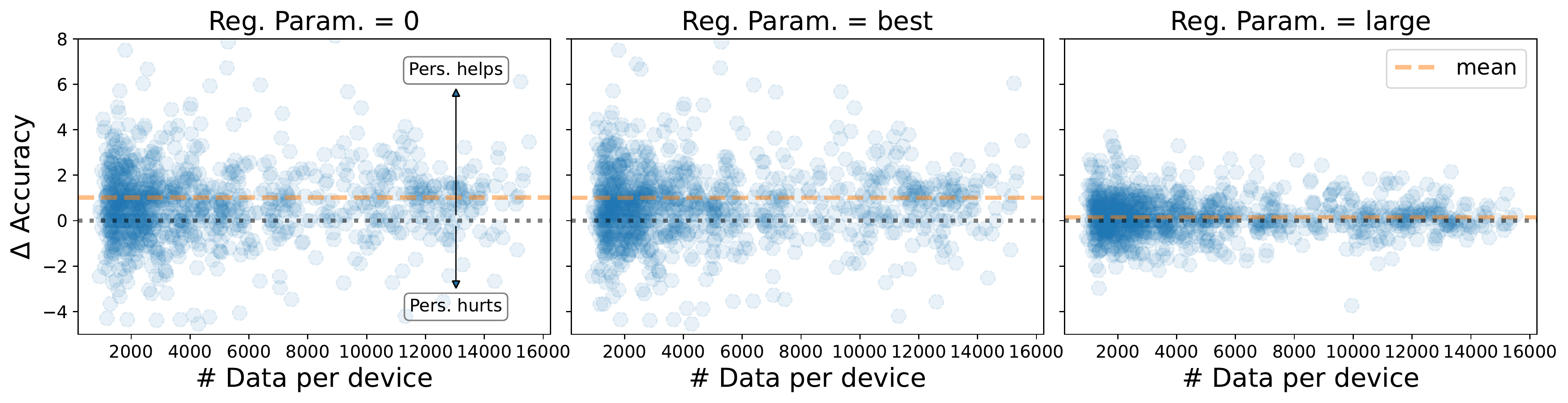}
\includegraphics[width=0.85\textwidth]{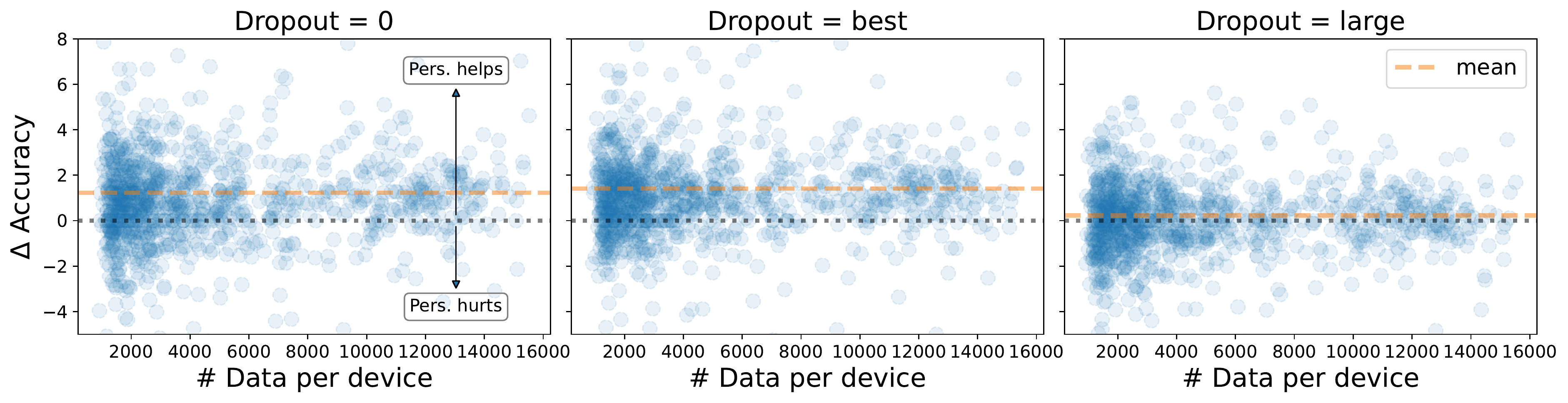}
\caption{\small{
    Scatter plot of change in accuracy (pp) per-device versus the number of training samples on the device with the effect of regularization. \textbf{Top}: $\ell_2$ regularization a.k.a. weight decay. \textbf{Bottom}: dropout.
    The ``best'' values of the $\ell_2$ regularization parameter and dropout are chosen to maximize the average test accuracy across all devices.
}}
\label{fig:expt:scatter:reg-so}
\end{figure*}

\myparagraph{Regularization and dropout do not mitigate this issue} From the first row of Figure~\ref{fig:expt:scatter:reg-so}, we see that the weight decay with best mean accuracy exactly matches the unreguarlized case in terms of per-device statistics. Increasing the regularization weight can reduce the spread of per-device accuracy. However, this only leads to a worse mean accuracy and does not mitigate the issue of personalization hurting individual devices. 

From the second row of Figure~\ref{fig:expt:scatter:reg-so}, we see that the best dropout ($0.3$ in this case) leads to slight increase in average accuracy ($0.18$ pp). It also reduces the number of devices hurt by personalization from $256$ out of $1000$  to $193$, but it does not fix this issue. Increasing dropout further only leads to a degradation of per-device statistics.

\begin{figure*}[t]
\includegraphics[width=0.98\textwidth]{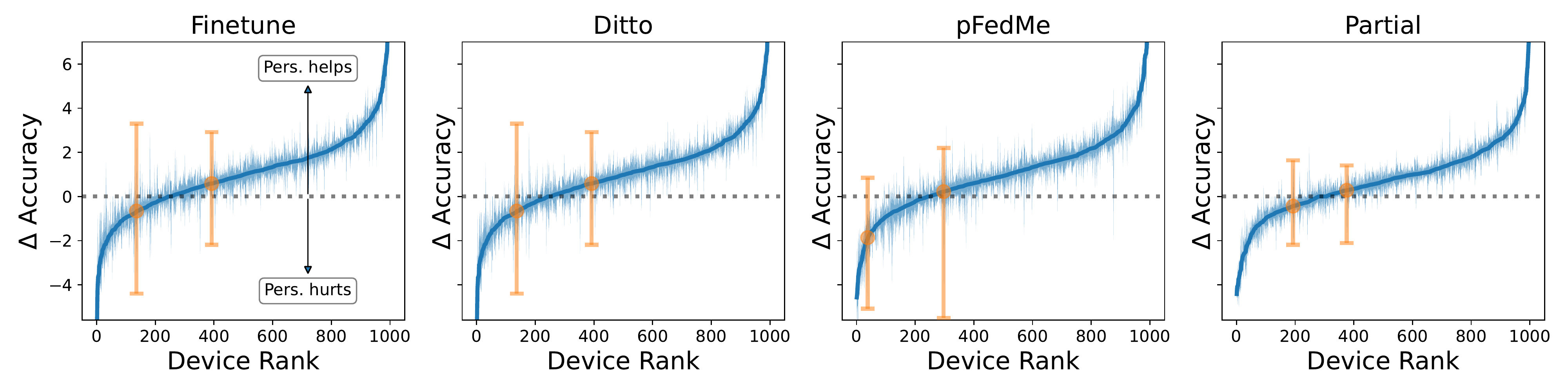}
\caption{\small{
    Change in per-device accuracy (pp) due to personalization. The solid line is the mean over $5$ random runs and the shaded area denotes the max/min across these runs. The devices are sorted in ascending order of accuracy change. 
    The points in orange depict two example devices who might either be helped or harmed by personalization depending on the random seed. 
}}
\label{fig:expt:scatter:per-seed}
\end{figure*}

\subsection{Partial Personalization for Stateless Devices}
\label{sec:a:expt-results:stateless}
The algorithms we considered in this paper, namely \pflam and \pflsu, require the devices to maintain the personalized parameters $v_i$'s as state across rounds.
In cross-device federated learning settings, it is also interesting to consider \emph{stateless} devices, which are not allowed to maintain state between training rounds.

We give preliminary experiments in this setting. We modify the \pflam and \pflsu algorithms from the main paper so that the personalized parameters $v_i$ are reinitialized each time device $i$ is chosen for participation. We warm-start $v_i$ from the appropriate part of the non-personalized model trained in step \ref{pipeline:expt:non-pers} of the pipeline. For adapters, we fix a random initialization once, and reuse it.

\myparagraph{\pflam is better than \pflsu for stateless devices, although the improvement is smaller}
We see from Table~\ref{table:expt:opt-algo:stateless} that all algorithms perform similarly for the stateless setting. Nevertheless, we see that \pflam obtains mild improvements over both \pflsu and finetuning for GLDv2, e.g., $0.24$pp with adapters.

\myparagraph{The final finetuning is crucial for stateless devices}
We see from Table~\ref{table:expt:ft:stateless} that the final finetuning accounts for most of improvements in the stateless case. For instance, for GLDv2, the final finetuning accounts for $11.68$ and $10.42$pp out of a total of $12.67$ and $11.76$pp for \pflam and \pflsu respectively. However, the personalized federated training (step~\ref{pipeline:expt:pfl} of the pipeline; cf. \S\ref{sec:a:expt:pipeline}) still leads to an increase in accuracy of $1$ to $1.34$pp.

\begin{table*}[t]
 \caption{\small{
    This is the counterpart of Table~\ref{table:expt:opt-algo:stateful} to stateless devices. 
    We compare \pflam and \pflsu for partial model personalization with stateless devices.
    ``FT (part.)'' corresponds to finetuning the personal parameters $v_i$ locally while fixing the shared parameters $u$ from a non-personalized training. The numbers are averaged over 5 random seeds; the boldfaced numbers denote the highest accuracy in each row.
    }}
\label{table:expt:opt-algo:stateless}
\begin{center}
\begin{adjustbox}{max width=\linewidth}
\small
\renewcommand{\arraystretch}{1.2}
\begin{tabular}{llllllllll}
\toprule
{} & \multicolumn{3}{c}{StackOverflow} & \multicolumn{3}{c}{GLDv2} & \multicolumn{3}{c}{EMNIST} \\
\cmidrule(lr){2-4}
\cmidrule(lr){5-7}
\cmidrule(lr){8-10}
{} &                 FT (part.) &                   FedAlt &          FedSim &        FT (part.) &                   FedAlt &          FedSim &        FT (part.) &                   FedAlt &          FedSim \\
\midrule
Input Layer  &  \tabemph{}$\mathbf{24.96}_{0.01}$ &           $24.84_{0.01}$ &  $24.89_{0.01}$ &  $51.97_{0.02}$ &  \tabemph{}$\mathbf{52.76}_{0.06}$ &  $52.74_{0.02}$ &  $93.29_{0.00}$ &  \tabemph{}$\mathbf{93.51}_{0.03}$ &  $93.48_{0.04}$ \\
Output Layer &           $24.93_{0.01}$ &  \tabemph{}$\mathbf{24.94}_{0.01}$ &  $24.94_{0.01}$ &  $53.21_{0.01}$ &  \tabemph{}$\mathbf{53.30}_{0.06}$ &  $53.30_{0.08}$ &  $93.37_{0.01}$ &  \tabemph{}$\mathbf{93.53}_{0.03}$ &  $93.51_{0.04}$ \\
Adapter      &  \tabemph{}$\mathbf{24.71}_{0.00}$ &           $24.69_{0.01}$ &  $24.71_{0.01}$ &  $63.86_{0.06}$ &  \tabemph{}$\mathbf{64.10}_{0.14}$ &  $63.19_{0.04}$ &  $93.66_{0.00}$ &  \tabemph{}$\mathbf{93.97}_{0.04}$ &  $93.89_{0.02}$ \\
\bottomrule
\end{tabular}
 \end{adjustbox}
\end{center}
\end{table*}

\begin{table*}[t]
    \caption{\small{
    The change in accuracy (percentage points) from the final finetuning for \pflam and \pflsu with stateless devices. 
    The subscript denotes the standard deviation over 5 random seeds.
    }}
\label{table:expt:ft:stateless}
\begin{center}
\begin{adjustbox}{max width=\linewidth}
\small
\renewcommand{\arraystretch}{1.2}
\begin{tabular}{lrrrrrr}
\toprule
{} & \multicolumn{2}{c}{StackOverflow} & \multicolumn{2}{c}{GLDv2} & \multicolumn{2}{c}{EMNIST} \\
\cmidrule(lr){2-3}
\cmidrule(lr){4-5}
\cmidrule(lr){6-7}
{} &         FedAlt &         FedSim &          FedAlt &          FedSim &         FedAlt &         FedSim \\
\midrule
Input Layer  &  $0.86_{0.03}$ &  $1.00_{0.02}$ &   $0.44_{0.03}$ &   $0.42_{0.03}$ &  $0.11_{0.02}$ &  $0.10_{0.04}$ \\
Output Layer &  $1.08_{0.03}$ &  $1.10_{0.02}$ &   $1.47_{0.04}$ &   $1.46_{0.05}$ &  $0.15_{0.02}$ &  $0.11_{0.02}$ \\
Adapter      &  $0.84_{0.04}$ &  $0.88_{0.02}$ &  $11.68_{0.20}$ &  $10.42_{0.09}$ &  $0.46_{0.02}$ &  $0.42_{0.04}$ \\
\bottomrule
\end{tabular}
 \end{adjustbox}
\end{center}
\end{table*}

\end{document}